%% file: main.tex
\documentclass[jair,twoside,11pt,theapa,letter]{article}
\let\chapter\section
\usepackage[ruled,vlined]{algorithm2e}
\usepackage{jair}
\usepackage{theapa}

\ShortHeadings{Query Answering for Greedy Sets of Existential Rules}%
{Rudolph, Thomazo, Baget, \& Mugnier} 

\usepackage{relsize}
\usepackage{xspace}
\usepackage{graphicx,color}
\usepackage{amsthm,amsmath,amssymb}
\usepackage{stmaryrd}
\usepackage{boxedminipage}
\usepackage{xy}\xyoption{all}
\usepackage{pgf,tikz}
\usepackage{pgfplots}
\usetikzlibrary{arrows}
\usetikzlibrary{shapes}

\input{macros}

\newtheorem{theorem}{Theorem}
\newtheorem{lemma}[theorem]{Lemma}

\newtheorem{property}[theorem]{Property}
\newtheorem{corollary}[theorem]{Corollary}
\newtheorem{definition}{Definition}
\newtheorem{example}{Example}

\newcommand{\citep}[1]{\cite{#1}}

\usepackage[dvips]{psfrag}

\title{Worst-case Optimal Query Answering for Greedy Sets of Existential Rules and Their Subclasses}

\author{\name Sebastian Rudolph\thanks{This work has been partially realized while S. Rudolph was working at AIFB, KIT, Germany and M. Thomazo was a Ph.D. student at University Montpellier 2} \email sebastian.rudolph@tu-dresden.de \\
       \addr TU Dresden, Germany
        \AND
       \name Micha\"{e}l Thomazo$^*$ \email michael.thomazo@tu-dresden.de \\
       \addr TU Dresden, Germany
       \AND
       \name Jean-Fran\c{c}ois Baget \email baget@lirmm.fr \\
       \addr Inria, France
       \AND
       \name Marie-Laure Mugnier \email mugnier@lirmm.fr \\
       \addr University Montpellier 2, France}

\begin{document}

\maketitle

\begin{abstract}

The need for an ontological layer on top of data, associated with advanced reasoning mechanisms able to exploit the semantics encoded in ontologies, has been acknowledged both in the database and knowledge representation communities. We focus in this paper on the ontological query answering problem, which consists of querying data while taking ontological knowledge into account. More specifically, we establish complexities of the conjunctive query entailment problem for classes of \emph{existential rules} (also called tuple-generating dependencies, \Datalogpm rules, or $\forall\exists$-rules). Our contribution is twofold. First, we introduce the class of \emph{greedy bounded-treewidth sets} {(\RgRbRtRsR)} of rules, which covers guarded rules, and their most
well-known generalizations. 
We provide a generic algorithm for query entailment under \RgRbRtRsR, which is worst-case optimal for combined complexity with or without bounded predicate arity, as well as for data complexity and query complexity. Secondly, we classify several \RgRbRtRsR classes, whose complexity was unknown,
with respect to combined complexity (with both unbounded and bounded predicate arity) and data complexity to obtain a comprehensive picture of the complexity of existential rule fragments that are based on diverse guardedness notions. Upper bounds are provided by showing that the proposed algorithm is optimal for all of them.


\end{abstract}


\newpage


\newpage

\section{Introduction}
\input{sec-introduction} %

\section{Preliminaries}\label{sec:prelim}
\input{sec-preliminaries} %

\section{Greedy Bounded-Treewidth Sets of Rules} \label{sec:gbts}%
\input{sec-gbts}

\section{An Algorithm for \RgRbRtRsR: PatSat} \label{sec:gbts-algo} %

\input{outline-kr}

\subsection{Patterned Forward Chaining}
\input{fc2pfc-kr}

\subsection{Bag Equivalence}
\input{pfc2bpfc-kr}

\subsection{Abstract Patterns and Abstract Pattern Saturation}
\input{blocked-kr-revised}


\subsection{Querying the Full Blocked Tree}
\input{querying2}

\subsection{Complexity Analysis and Worst-Case Optimal Adaptation to Subclasses}
\label{sec-subclasses}
\input{subclasses-kr}



\section{Matching Lower Complexity Bounds for \RgRbRtRsR Subclasses} \label{sec:fg} %
\input{sec-fg}

\section{Body-Acyclic \RfRgR and \RfRrRoR Rules} \label{sec:ba} %
\input{sec-bodyacyclic}

\section{Related Work} 
\label{sec-related}

In this section, discuss the relationship of the existential rule fragments considered here with another major paradigm in logic-based knowledge representation: description logics. Also, we will point out similarities of the techniques applied in the presented algorithm with reasoning approaches established for other logics.


\subsection{Relationships to Horn Description Logics and their Extensions}

The relationship of description logics and existential rules has often been recognized. In particular Horn-DLs \cite{HMS:hornshiq,KRH:hornAAAI,KRH:HornDLs2013} share many properties with existential rules such as the existence of a (homomorphically unique) canonical model. Crucial differences between the two approaches are that (1) as opposed to DLs, existential rules allow for predicates of arity greater than two as well as for the description of non-tree shaped terminological information and (2) as opposed to existential rules, expressive DLs allow for a tighter integration of cardinality constraints to a degree (at least currently) unachieved by existential rules.

In the following, we will point out which Horn-DLs are subsumed by which existential rules fragments. We will refrain from providing full translations and restrict ourselves to examples that provide the underlying intuition.

The description logic $\mathcal{EL}$ essentially allows for encoding implications of tree-shaped substructures in a model. For instance the statement ``Everybody who has a caring mother and a caring father has a nice home'' can be expressed by the $\mathcal{EL}$ axiom $\exists \mathit{hasMother}.\mathit{Caring} \sqcap \exists \mathit{hasFather}.\mathit{Caring} \sqsubseteq \exists \mathit{hasHome}.\mathit{Nice}$ which is equivalent to the existential rule $$\mathit{hasMother}(x,y) \wedge \mathit{Caring}(y) \wedge \mathit{hasFather}(x,z) \wedge \mathit{Caring}(z) \to \mathit{hasHome}(x,w) \wedge \mathit{Nice}(w).$$ Horn-$\mathcal{ALC}$ is more expressive than $\mathcal{EL}$ in that it allows to express some sort of universal quantification such as in ``Whenever some caring person has children, all of them are happy'' denoted by $\mathit{Caring} \sqsubseteq \forall\mathit{hasChild}$ which corresponds to the existential rule $\mathit{Caring}(x) \wedge \mathit{hasChild}(x,y) \to \mathit{Happy}(y)$.

It is not hard to see that the Horn-DLs $\mathcal{EL}$ and Horn-$\mathcal{ALC}$ are captured by \RfRrRoR rules; they can even be linearly rewritten into \RgRfRrRoR rules when auxiliary predicates are allowed (as it is often done when normalizing DL knowledge bases). This still holds when these DL languages are extended by inverses (indicated by adding an $\mathcal{I}$ to the name of the DL: $\mathcal{ELI}$, Horn-$\mathcal{ALCI}$) and/or nominals (indicated by adding an $\mathcal{O}$). In the latter case, constants must occur in the rules). For instance, the $\mathcal{ELIO}$ proposition ``Everybody who is born in Germany likes some soccer team which has a member who is also a member of the the German national team'', written in DL notation $$\exists \mathit{bornin}.\{germany\}\sqsubseteq \exists\mathit{likes}.(\mathit{SoccerTeam}\sqcap \exists hasMember.\exists hasMember^-.\{gnt\})$$, in existential rules form $$\mathit{bornin}(x,germany) \to \mathit{likes}(x,y) \wedge \mathit{SoccerTeam}(y) \wedge hasMember(y,z) \wedge hasMember(gnt,z)$$

\emph{Role hierarchies} ($\mathcal{H}$), allow to express generalization/specialization relationships on binary predicates (such as $\mathit{fatherOf}$ implying $\mathit{parentOf}$). Adding this feature to any of the abovementioned description logics requires to use rules with frontier size 2 and thus leads outside the frontier-one fragment. Still a linear translation into $\RgR$ rules remains possible.\footnote{It might, however, be noteworthy that it is possible to come up with a polynomial translation into \RgRfRrRoR rules by materializing the subsumption hierarchy of the binary predicates upfront and, whenever a binary atom is created by a rule ``creating'' all the subsumed atoms at the same time.}. Going further to DLs that feature functionality or transitivity of binary predicates leads to existential rule fragments which are not longer guarded in any way considered here. The definition of existential rule fragments capturing these expressive description logics is subject of ongoing research.

Diverse proposals have been made to overcome the structural restrictions of DLs, i.e. to allow to express non-tree-shaped relationships in the terminological knowledge. \emph{Description graphs} \cite{DBLP:journals/ai/MotikGHS09} constitute one of these endeavors, where the existentially quantified structure in the head of a DL axiom is allowed to be arbitrarily graph-shaped and, additionally, there are datalog rules operating locally on these graph structures. It is straight-forward that the extension of any DL up to Horn-$\mathcal{ACLHIO}$ by description graphs can be coded into $\RgR$ existential rules.

Another suggestion made to allow for non-tree shaped structures in both body and head of a DL axiom is to introduce \emph{DL-safe variables}, that is, variables that are only allowed to be bound to ``named individuals'' (i.e., domain elements denoted by constants). In a setting where each statement can carry either exclusively safe variables or exclusively non-safe ones, this can be captured by the notion of \emph{DL-safe rules} \cite{DLsafe}. A more liberal approach is that of the so-called \emph{nominal schemas} \cite{DBLP:conf/www/KrotzschMKH11,DBLP:conf/kr/KrotzschR14}, where the two types of variables can occur jointly in the same statement. In both cases, the \RwRgR fragment captures these extensions when applied to DLs up to Horn-$\mathcal{ALCHIO}$. Note that there is a direct correspondence between non-affected positions in existential rules and DL-safe variables: both can only carry domain elements corresponding to elements present in the initial data.

%
%
%
%
%
%
%

\subsection{Pattern- or Type-Based Reasoning}

For many logics, reasoning algorithms as well as related complexity arguments are based on the notion of types.
On an intuitive level, types represent ``configurations'' which may occur in a model. In the easiest case (as in some description logics), such configurations might be sets of unary predicates $\{p_1,\ldots,p_n\}$, where $\{p_1,\ldots,p_n\}$ occurring in a model just means that there is an individual $a$ that is in the extension of every $p_i$. More complex notions of types may refer to more than just one individual, leading to notions like 2-types, also known as dominoes \cite{DBLP:journals/corr/abs-1202-0914}, star-types \cite{PH:C2complex}, mosaics \cite{mlhandbook}, or types based on trees \cite{RuGl10a}. Often, reasoning in a logic can be carried out by only considering the (multi-)set of types that are realized in a model. Typical reasoning strategies may then compute the set of these types bottom-up (as in tableaux with anywhere-blocking ), top-down (as in type-elimination-based procedures), or describe their multiplicity by means of equational systems. The applicability of such strategies guarantees decidability whenever the overall set of possible types is finite.
It is not hard to see that, in our case, abstract patterns can be seen as ``graph types'' where the bound on the tree-width and the finiteness of the vocabulary guarantee the finiteness of the set of types, and therefore the effectiveness of the applied blocking strategy.

\subsection{Consequence-Driven Reasoning}

%
%
%

Our idea of the saturation of pattern rules in Section ... has many similarities with the approach of consequence-driven reasoning in description logics. In both cases, logical sentences that are consequences of a given theory are materialized. To see this, one should be aware that every evolution rule $\mathbb{P}_1 \rightsquigarrow \mathbb{P}_2$ corresponds to an existential rule $$\bigwedge_{(G,\pi)\in \mathbb{P}_1} \pi^\mathrm{safe}(G) \to \bigwedge_{(G,\pi)\in \mathbb{P}_2} \pi^\mathrm{safe}(G)$$ and every creation rule $\mathbb{P}_1 \rightsquigarrow \lambda.\mathbb{P}_2$ corresponds to an existential rule $$\bigwedge_{(G,\pi)\in \mathbb{P}_1} \pi^\mathrm{safe}(G) \to \bigwedge_{(G,\pi)\in \mathbb{P}_2} \pi^\mathrm{safe}(\lambda(G)).$$ It can then be readily checked that the deduction calculus presented in Fig.~\ref{fig-deductioncalculus} is indeed sound. As such, the presented algorithm has indeed similarities with type-based consequence driven reasoning approaches as, e.g., layed out by \cite{DBLP:conf/ijcai/Kazakov09} and \cite{DBLP:conf/kr/OrtizRS10,DBLP:conf/ijcai/OrtizRS11}. The crucial difference here is that the mentioned works use only 1-types, whereas the patterns defined characterize larger ``clusters'' of elements.

\subsection{Tableaux and Blocking}

It is well-known that the chase known from databases has many commonalities with
the semantic tableau method in FOL \cite{TableauxBeth,smullyan1995first}, which has also been used in many other logics, most notably DLs \cite{TableauxInDLs,DBLP:journals/jar/HorrocksS07}. Note that the generic semantic tableaux for first order logic only gives rise to a semidecision procedure. In order to obtain a decision procedure for a restricted logic, termination needs to be guaranteed, typically through establishing a tree(-like) model property and the detection of repetitions in the course of the tableaux construction, leading to the idea of \emph{blocking} as soon as repeating types occur (depending on the expressivity of the logic, 1-types, 2-types or even larger types have to be considered). Clearly, the blocking technique used by us in the construction of the full blocked tree can be seen as a pattern-based anywhere blocking.

\subsection{Relationships to other work on guarded existential rules}

As already mentioned, guarded and weakly-guarded rules were introduced in \cite{cali-gottlob-kifer:08,Cali2009,cgk:13}.
A fundamental notion used to bound the depth of the breadth-first saturation (with a bound depending on $\mathcal R$ and $Q$)
is that of the type of a guard atom in the saturation (a guard atom in $\alpha_\infty(F, \mathcal R)$ is the image of a rule guard by a rule application, and the type of an atom $a$ is the set all atoms in $\alpha_\infty(F, \mathcal R)$ with arguments included in $\terms{a}$). This notion has some similarities with our bag patterns, without being exactly the restriction of bag patterns to the guarded case.

The notion of affected position/variable was
refined into that of \emph{jointly affected} position/variable in
\citep{kr:11}. This yields the new classes of \emph{jointly guarded}
rules and \emph{jointly frontier-guarded} rules, which respectively
generalize \RwRgR rules and \RwRfRgR rules. Since these new classes
are \RgRbRtRsR, our results apply to them. In particular, the data
and combined complexities  of jointly frontier-guarded rules
directly follow from those of \RgRbRtRsR and \RwRfRgR:
\ExpTime-complete data complexity, and \ExpExpTime-complete combined
complexity (in both bounded and unbounded predicate arity cases).
Since \RwRgR rules have the same complexities as \RwRfRgR rules for
data complexity and combined complexity with bounded arity, these
complexities also apply to jointly guarded rules; for combined
complexity with unbounded arity, the bounds are not tight
(\ExpTime-hardness from the result on \RwRgR and
\ExpExpTime-membership from the result on \RwRfRgR). Note that
\citep{kr:11} also provides a further generalization, namely
\emph{glut frontier-guarded}, which is \RbRtRsR, but not \RgRbRtRsR
nor \RfReRsR.

Combinations of the \RgRbRtRsR family with other families of rules have been proposed, by restricting possible interactions between rules of the different kinds.  In \cite{blms:11}, conditions expressed on the strongly connected components of a graph of rule dependencies allow to combine \RgRbRtRsR, \RfReRsR and \RfRuRsR sets of rules. In \cite{tplp-13-gmp}, a notion called tameness allows to combine guarded rules with sticky rules (an expressive \RfReRsR concrete class of rules) by restricting the interactions between the sticky rules and the guard atoms in the guarded rules.

Finally, let us cite some very recent work related to the guarded family of existential rules. The expressiveness of the \RwRgR and \RwRfRgR fragments was studied in \cite{pods-14-grs}. Other work has analyzed the complexity of entailment with guarded rules extended with disjunction \cite{bmp:13} or with stable negation \cite{ghkl:14}.

\subsection{Combined Approach}

The combined approach \cite{lutz:09,kr-10-kltwz}, designed for $\mathcal{EL}$ and the DL-Lite familly, share some similarities with our approach. The combined approach is a two-step process. First, some materialization is performed. In order to ensure finiteness of this step, an over-specialization of the canonical model is thus computed. This over-specialization requires thus a rewriting of the query, in order to recover soundness. This rewriting may require the ontology. By comparison, our approach computes a materialization that is less specific than the saturation. It is thus the completeness that has to be recovered, through a change of the querying operation, which is not a simple homomorphism anymore, but is based on the notion of APT-homomorphism.



\section{Conclusion} 
\input{sec-conclusion}

\section*{Acknowledgments}

Micha\"el Thomazo aknowledges support from the Alexander von Humboldt foundation.

\bibliographystyle{theapa}
\bibliography{biblio}

\end{document}

%% file: macros.tex
\newcommand{\Datalogpm}{Datalog$^\pm$\xspace}

\newcommand{\url}[1]{\texttt{#1}}

\DeclareMathAlphabet{\mathpzc}{OT1}{pzc}{b}{rm}
\newcommand{\rulelang}[1]{\textscale{1.1}{{\ensuremath{\mathpzc{#1}}}}}
\newcommand{\RgR}{\rulelang{g}\xspace}
\newcommand{\RfRgR}{\rulelang{fg}\xspace}
\newcommand{\RfRrRoR}{\rulelang{fr1}\xspace}
\newcommand{\RgRfRrRoR}{\rulelang{gfr1}\xspace}
\newcommand{\RwRgR}{\rulelang{wg}\xspace}
\newcommand{\RwRfRgR}{\rulelang{wfg}\xspace}
\newcommand{\RwRfRrRoR}{\rulelang{wfr1}\xspace}
\newcommand{\RwRgRfRrRoR}{\rulelang{wgfr1}\xspace}
\newcommand{\RfReRsR}{\rulelang{fes}\xspace}
\newcommand{\RfRuRsR}{\rulelang{fus}\xspace}
\newcommand{\RbRtRsR}{\rulelang{bts}\xspace}
\newcommand{\RgRbRtRsR}{\rulelang{gbts}\xspace}

\newcommand{\fun}[1]{\ensuremath{\textsl{#1}}}

\newcommand{\terms}[1]{\ensuremath{\fun{terms}(#1)}}
\newcommand{\vars}[1]{\ensuremath{\fun{vars}(#1)}}

\newcommand{\body}[1]{\ensuremath{\fun{body}(#1)}}
\newcommand{\bod}[1]{\ensuremath{\fun{body}(#1)}}
\newcommand{\head}[1]{\ensuremath{\fun{head}(#1)}}
\newcommand{\atoms}[1]{\ensuremath{\fun{atoms}(#1)}} 

\newcommand{\fr}[1]{\ensuremath{\fun{fr}(#1)}}

\newcommand{\NP}{{\sc{NP}}}
\newcommand{\PTime}{{\sc{PTime}}}
\newcommand{\ExpTime}{{\sc{ExpTime}}}
\newcommand{\ExpExpTime}{{\sc{2ExpTime}}}



%% file: sec-introduction.tex

%
Intelligent methods for searching and managing large amounts of
data require rich and elaborate schematic ``ontological'' knowledge.
The need for an ontological layer on top of that data, associated
with advanced querying mechanisms able to exploit the semantics
encoded in ontologies\footnote{In this paper, we reserve the
term \emph{ontology} to general domain knowledge--sometimes also
called \emph{terminological} knowledge--in order to clearly distinguish it
from the factual data--or \emph{assertional} knowledge.},
has been widely
acknowledged both in the knowledge representation (KR) and database
communities.

Deductive databases and KR typically adopt two different perspectives
on how to add this ontological layer to the picture of plain query answering \cite<cf.>{DBLP:conf/amw/Rudolph14}.
In deductive databases, this knowledge is considered part of the query,
forming a so-called ontology-mediated query to be executed on the database.
According to the KR perspective, knowledge is encoded in an ontology,
and queries are asked to a knowledge base composed of the data and the ontology.
In this paper, we will focus on the KR perspective.

Indeed, ontologies, which typically encode general knowledge about the domain of interest,
can be used to infer data
 that are not explicitely stored, hence palliating incompleteness in databases \cite{pods-03-clr}.
 They can also be used to enrich the vocabulary of data sources,
 which allows a user to abstract from the specific way data are stored.
 Finally, when several data sources use different vocabularies,
 ontologies can be used to align these vocabularies.

Given a knowledge base (KB) composed of an ontology
and of factual knowledge, and a query, the ontology-based query answering problem
consists in computing the set of answers to the query on the KB,
while taking implicit knowledge represented in the ontology into
account. We make here the simplifying assumption that the ontology
and the database use the same vocabulary. Otherwise, mappings have to be defined
between both vocabularies,  as in the ontology-based data access framework \cite{jods-08-plcglr}.
As most work in this area, we focus on conjunctive queries (CQs),
the basic and most frequent querying formalism in databases.


In the Semantic Web area, one of the most prominent fields where KR
technology is practically applied, ontological knowledge is often
represented by means of formalisms based on description logics (DLs,
\cite{dlhandbook,DBLP:conf/rweb/Rudolph11}). However, DLs are
restricted in terms of expressivity in that they usually  support only unary
and binary predicates
and that terminological expressiveness is
essentially restricted to tree-like dependencies between the atoms
of a formula. Moreover, DL research has traditionally been focusing on so-called
standard reasoning tasks about the knowledge base, which are reducible to
knowledge base satisfiability, for instance classifying concepts;
querying tasks were essentially
restricted to ground atom entailment. Answering full conjunctive
queries (CQs) over DL knowledge bases has become a subject of
research only recently\footnote{CQ answering in the context of DLs
was first mentioned by \citeauthor{LeRo96a} \citeyear{LeRo96a}, with
the first publication focusing on that subject by \citeauthor{CaDL98a} \citeyear{CaDL98a}.}, turning out to be extremely complex (e.g.,
for the classical DL $\mathcal{ALCI}$, it is already
\ExpExpTime-complete, and still \NP-complete in the size of the
data). Consequently, conjunctive query answering has been
particularly studied on less expressive DLs, such as DL-Lite
\cite{dl-lite} and $\mathcal {EL}$ \cite{Baader03:lcs,lutz:09}. These
DLs are the basis of so-called tractable profiles OWL~2 QL and OWL~2 EL of the
Semantic Web language OWL~2 \cite{owl2-overview,owl2-profiles}.\footnote{Beside the profiles based on DL-Lite and $\mathcal{EL}$,
there is a third OWL~2 tractable profile, OWL~2 RL, which can be seen as a restriction of Datalog.}

On the other hand, querying large amounts of data is the fundamental
task in databases. Therefore, the  challenge in this domain is now
to access data while taking ontological knowledge into account. The
deductive database language Datalog allows to express some
ontological knowledge.
However, in Datalog rules, variables are
range-restricted, i.e., all variables in the rule are universally
quantified, which does not allow to infer the existence of initially
unknown domain individuals (a capability called \emph{value
invention} in databases \citep{Alice}). Yet, this feature has been
recognized as crucial in an open-world perspective, where it cannot
be assumed that all individuals are known in advance.

To accommodate the requirements sketched above -- value invention
and complex rela\-tion\-ships -- we consider here an extension of
first-order function-free Horn rules that allows for existentially
quantified variables in the rule heads and thus features value
invention. More precisely, these extended rules are of the form
$\mathit{Body} \rightarrow \mathit{Head}$, where $\mathit{Body}$ and
$\mathit{Head}$ are conjunctions of atoms, and variables occurring
only in the $\mathit{Head}$ are existentially quantified, hence
their name ``existential rules'' in \citep{bmrt:11,kr:11}.

\begin{example} \label{ex-human}
Consider the existential rule  $$R = \forall x
\Big(\mathit{human}(x) \rightarrow \exists
y\big(\mathit{hasParent}(x,y) \wedge \mathit{human}(y)\big)\Big)$$
and a fact $F = \mathit{human(a)}$, where $a$ is a constant. The
application of $R$ to $F$ produces new factual knowledge, namely
$$\exists y_0 \big(\mathit{hasParent}(a,y_0) \wedge
\mathit{human}(y_0)\big),$$ where $y_0$ is a variable denoting an
unknown individual. Note that $R$ could be applied again to
$\mathit{human}(y_0)$, which would lead to create another
existentially quantified variable, and so on.
\end{example}

Such rules are well-known in databases as  Tuple-Generating
Dependencies (TGDs) \citep{beeri-vardi:84} and have been extensively
used, e.g., for data exchange \citep{fagin-kolaitis-al:05}. Recently,
the corresponding logical fragment has gained new interest in the
context of ontology-based query answering. It has been
introduced as the \Datalogpm framework in
\citep{cali-gottlob-kifer:08,Cali2009,cali-glmp:10}, and
independently, stemming from graph-based knowledge representation
formalisms \citep{chein-mugnier:09}, as $\forall\exists$ rules
\citep{blms:09,blm:10}.

 This rule-based framework generalizes the core of the lightweight description logics
mentioned above, namely DL-Lite and $\mathcal {EL}$.\footnote{The DL
constructor called existential restriction ($\exists R.C$) is
fundamental in these DLs. The logical encoding  of an axiom
that contains $\exists R.C$ in its right-hand side requires an existentially
quantified variable in the corresponding rule head. For instance,
the rule from Example~\ref{ex-human} can be seen as the logical
translation of the DL axiom $Human \sqsubseteq
\exists hasParent.Human$.} Moreover, in the case of the DL-Lite family
\citep{dl-lite}, it has been shown that this covering by a
\Datalogpm fragment is done without increasing complexity
\citep{Cali2009}.

Several fundamental decision problems can be associated
with conjunctive query answering under existential rules.
In this paper, we consider the entailment problem of a Boolean conjunctive query under existential rules,
 which we are now able to define formally. A Boolean conjunctive query is an existentially closed
 conjunction of (function-free) atoms. A set of facts has the same form.
 A knowledge base is composed of a set of facts and a set of existential rules. The entailment problem
 takes as input a knowledge base and a Boolean conjunctive query and asks if this query is entailed
 by the knowledge base.

The presence of existentially quantified variables in rule heads, associated with
arbitrarily complex conjunctions of atoms, makes the entailment problem undecidable
\citep{beeri-vardi:81,chandra-lewis-makowsky:81}.
Since the birth of
TGDs, and recently within the \Datalogpm and $\forall \exists$ rule
frameworks, various conditions of decidability have been exhibited.
Three ``abstract'' classes have been introduced in \citep{blm:10} to
describe known decidable behaviors: an obvious condition of
decidability is the finiteness of the forward chaining (known as the
\emph{chase} in the TGD framework
\citep{mms:79,johnson-klug:84,beeri-vardi:84});
sets of rules ensuring this condition are called \emph{finite
expansion sets} (\RfReRsR); a more general condition introduced in
\citep{cali-gottlob-kifer:08} accepts infinite forward chaining
provided that the facts generated have a bounded treewidth (when
seen as graphs); such sets of rules are called \emph{bounded-treewidth
sets} (\RbRtRsR); then decidability follows from the decidability of
first-order logic (FOL) classes with the bounded-treewidth model
property \citep{courcelle:90}. The third condition, giving rise to
\emph{finite unification sets} (\RfRuRsR), relies on the finiteness
of (a kind of) backward chaining mechanism, this condition is also
known as \emph{first-order rewritability}. None of these three
abstract classes is recognizable, i.e., the problem of deciding
whether a given set of rules is \RfReRsR, \RbRtRsR, or \RfRuRsR is
undecidable \citep{blm:10}.

In this paper, we focus on the \RbRtRsR paradigm and its main
``concrete'' classes. (Pure) \emph{Datalog} rules (i.e.,
without existential variables) are \RfReRsR (thus \RbRtRsR).
\emph{Guarded} (\RgR) rules \citep{cali-gottlob-kifer:08} are
inspired by the guarded fragment of FOL
\citep{andreka-benthem-nemeti:96,andreka-nemeti-benthem-98}. Their
body has an atom, called a guard, that contains all variables from the
body. Guarded rules are \RbRtRsR (and not \RfReRsR). They are
generalized by \emph{weakly guarded rules} (\RwRgR), in which the
guarding condition is relaxed: only so-called ``affected'' variables
need to be guarded; intuitively, affected variables are variables
that are possibly mapped, during the forward chaining process, to
newly created variables  \citep{cali-gottlob-kifer:08}. \RwRgR rules
include Datalog rules (in which there are no affected variables).
Other decidable classes rely on the notion of the \emph{frontier} of
a rule (the set of variables shared between the body and the head of
a rule). In a \emph{frontier-one} rule {(\RfRrRoR)}, the frontier is
restricted to a single variable \citep{blms:09}. In a
frontier-guarded rule {(\RfRgR)}, an atom in the body  guards the
frontier \citep{blm:10}. Hence, \RfRgR rules generalize both
\emph{guarded} rules and \RfRrRoR rules. When requiring only
affected variables from the frontier to be guarded, we obtain the
still decidable class of \emph{weakly frontier-guarded rules
}{(\RwRfRgR)}, which generalizes both \RfRgR and \RwRgR classes
\citep{blm:10}. Not considered until now were rule sets obtained by
straightforward combinations of the above properties: \emph{guarded
frontier-one rules} (\RgRfRrRoR) as well as \emph{weak frontier-one
rules} (\RwRfRrRoR) and \emph{weakly guarded frontier-one rules}
(\RwRgRfRrRoR). Table~\ref{table:complexity} summarizes the
considered existential rule fragments with respect to their
constraints on frontier variables and guardedness.


\begin{table}[th] \label{table:complexity}
\begin{center}
\begin{tabular}{|l|l|l|l|l|}
        \hline
         \hfill\textbf{syntactic\ \ \ }     & non-affected  & non-frontier & at most one\\
         \hfill\textbf{properties\ \ \ }     & variables     & variables    & frontier\\
                                & must be       & must be      & variable that\\
         \textbf{class (abbrv)} & guarded       & guarded      & needs guarding \\
         \hline\hline
         guarded frontier-one rules (\RgRfRrRoR) & yes & yes & yes \\
         \hline
         guarded rules (\RgR) & yes & yes & no \\
         \hline
         frontier-one rules (\RfRrRoR) & yes & no & yes \\
         \hline
         frontier-guarded rules (\RfRgR) & yes & no & no \\
         \hline
         weakly guarded frontier-one rules (\RwRgRfRrRoR) & no & yes & yes \\
         \hline
         weakly guarded rules (\RwRgR) & no & yes & no \\
         \hline
         weakly frontier-one rules (\RwRfRrRoR) & no & no & yes \\
         \hline
         weakly frontier-guarded rules (\RwRfRgR) & no & no & no \\
         \hline
\end{tabular}
\end{center}
\caption{Considered classes with syntactic properties.}
\end{table}


\begin{example}  We consider the following relations, where the subscripts indicate the arity of the relation:  \emph{project$_{/3}$},
 \emph{projectField$_{/2}$},  \emph{projectDpt$_{/2}$}, \emph{hasManager$_{/2}$}, \emph{memberOf$_{/2}$}, \emph{isSensitiveField$_{/1}$} and \emph{isCriticalManager$_{/1}$}. Intuitively, \emph{project($x,d,z$)} means that $x$ is a project in department $d$ and is about field $z$; relations \emph{projectField} and \emph{projectDpt} are projections of the \emph{project } relation; \emph{hasManager($x,y$)} and \emph{member($y,d$)} respectively mean that $x$ is managed by $y$ and $y$ is member of $d$; relations \emph{isSensitiveField} and \emph{isCriticalManager} respectively apply to fields and managers.
 Let $\mathcal R$ be the following set of existential rules built on this vocabulary:
\begin{itemize}
\item \emph{Decomposition of the relation project into two binary relations}\\
$R_0 = \forall x \forall d \forall z (project(x,d,z) \rightarrow projectDpt(x,d) \wedge projectField(x,z)$
\item \emph{``Every project has a manager''}\\
$ R_1 =\forall x \forall z (projectField(x,z) \rightarrow \exists y ~hasManager(x,y))$
\item \emph{``Every managed project has some field''}\\
$R_2 =\forall x  \forall y (hasManager(x,y) \rightarrow \exists z ~projectField(x,z))$
\item \emph{``The manager of a project is a member of the department that owns the project''}\\
$R_3 =\forall x \forall y \forall d ( hasManager(x,y) \wedge projectDpt(x,d) \rightarrow memberOf(y,d))$
\item \emph{``If a manager manages a project in a sensitive field, then (s)he is a critical manager''}\\
$R_4 =\forall x \forall y \forall z(hasManager(x,y)~\wedge~projectField(x,z)~\wedge~isSensitiveField(z) \rightarrow isCriticalManager(y))$
\item \emph{``Every critical manager manages a project in a sensible field''}\\
$R_5 =\forall  y (isCriticalManager(y) \rightarrow \exists x \exists z (hasManager(x,y) \wedge projectField(x,z) \wedge isSensitiveField(z)))$
\end{itemize}
Note that rules $R_0$, $R_3$ and $R_4$ do not introduce any existential variable.
Rules $R_0$, $R_1$, $R_2$ and $R_5$ have an atomic body, hence they are trivially guarded.
Rule $R_4$ is not guarded, but it is frontier-one.
Hence, if we exclude $R_3$, all rules are frontier-guarded.
$R_3$ is not frontier-guarded, since no atom from the body contains both frontier variables $y$ and $d$;
however, we remark that variable $d$ is not affected, i.e., it can never be mapped to a newly created variable
(indeed, the only rule able to produce an atom with predicate \emph{projectDpt }is $R_0$;
the variables in this atom come from the body atom with predicate \emph{project}, which never appears in a rule head,
hence can only be mapped to initially present data).
Affected frontier-variables are guarded in all rules, hence $\mathcal R$ is weakly frontier-guarded.
$\mathcal R$ is even weakly frontier-one, since for each rule the frontier contains at most one affected variable.
\end{example}

Contrarily to \RfReRsR and \RfRuRsR, the definition of \RbRtRsR is
not based on a constructive entailment procedure. The complexity of the subclasses \RgR and \RwRgR is known and an algorithm for the corresponding entailment problem has been provided \citep{cali-gottlob-kifer:08,Cali2009}. However, this is not
the case for the classes \RgRfRrRoR, \RfRrRoR, \RfRgR, \RwRgRfRrRoR,
\RwRfRrRoR, \RwRfRgR, and \RgRbRtRsR. The aim of this paper is to
solve these algorithmic and complexity issues.
%

Our contribution is threefold. First, by imposing a restriction on the
allowed forward-chaining derivation sequences, we define a subclass
of \RbRtRsR, namely \emph{greedy bounded-treewidth sets} of rules
{(\RgRbRtRsR)}, which have the nice property of covering the
\RwRfRgR class. \RgRbRtRsR are defined by a rather simple condition:
when such a set is processed in forward chaining and a rule $R$ is
applied, all frontier variables of $R$ which are not mapped to terms
from the initial data set must be uniformly mapped to terms
introduced by one \emph{single} previous rule application. The
fundamental property satisfied thanks to this condition is that any
derivation can be naturally associated with a bounded-width
tree decomposition of the derived facts, which can be built in a ``greedy
manner'', that is, on the fly during the forward chaining process.
We also prove that \RwRfRgR and \RgRbRtRsR have essentially the same expressivity.


Secondly, we provide a generic algorithm for the  \RgRbRtRsR class,
which is worst-case
optimal for data complexity, for combined complexity (with or
without bound on the arity of involved predicates), and for query
complexity.
We furthermore show that this algorithm can be slightly adapted to
be worse-case optimal for subclasses with smaller complexities.

Thirdly, we classify \RgRfRrRoR, \RfRrRoR, \RfRgR,
\RwRgRfRrRoR, \RwRfRrRoR, \RwRfRgR, and \RgRbRtRsR with respect to
both combined (with and without predicate arity bound) and data
complexities. We also consider the case of rules with an acyclic
(more precisely, hypergraph-acyclic) body  and point out that
body-acyclic \RfRgR rules coincide with guarded rules from an
expressivity and complexity perspective.

\begin{figure}[th]
\begin{center}
\includegraphics[width=10cm]{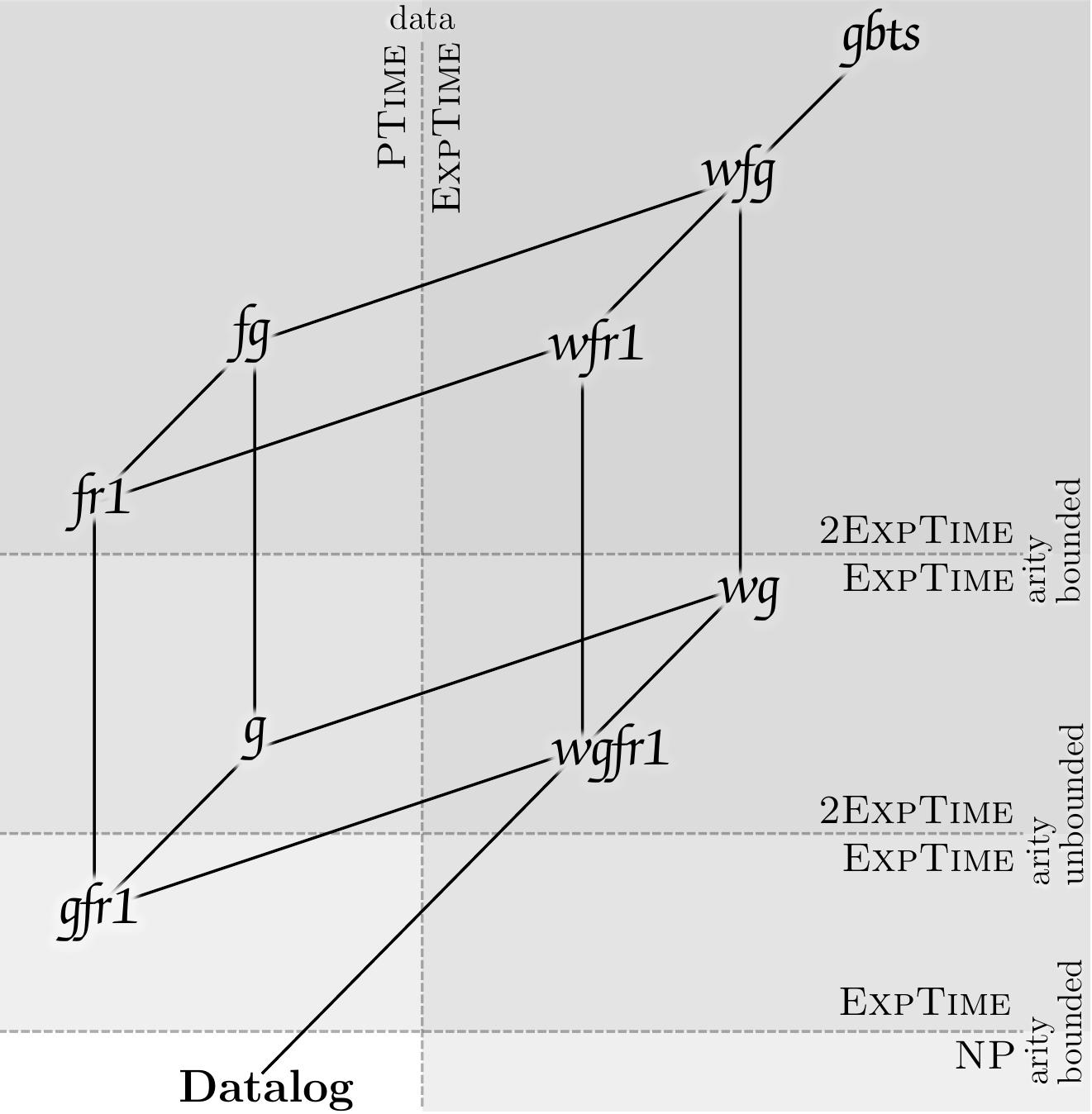}

\caption{Existential rule fragments, their relative expressiveness,
and complexities.} \label{complexities-min}
\end{center}
\end{figure}

Fig.~\ref{complexities-min} shows the complexity lines for these
classes of rules with three complexity measures, namely combined
complexity without or with bound on the predicate arity, and data
complexity. Notice in particular that the two direct extensions of  guarded rules,
i.e., weakly guarded and frontier-guarded rules,  do not behave
in the same way with respect to the different complexity measures:
for data complexity, \RfRgR rules remain in \PTime{},
while \RwRgR rules  are in \ExpTime{}; on the other hand,
for bounded-arity combined complexity,  \RfRgR rules are in  \ExpExpTime{},
while \RwRgR rules remain in \ExpTime{}.
%
Precise complexity results obtained are given in
Tab.~\ref{table:complexity}. All complexities are complete for their class.
New results are indicated by a star.



\begin{table}[th] \label{table:complexity}
\begin{center}
\begin{tabular}{|l|l|l|l|}
        \hline
         \textbf{Class}   &\textbf{arity}& \textbf{arity} & \textbf{Data }\\
         \textbf{ }   &\textbf{ unbounded }& \textbf{ bounded} & \textbf{Complexity}\\
       \hline\hline
       {\bf Datalog}   & \ExpTime{}  & \NP{} & \PTime{}\\
       \hline
       {\bf \RgRfRrRoR}    &  \ExpTime{}$^{\star}$  & \ExpTime{}$^{\star}$ & \PTime{}$^{\star}$\\
       \hline
       {\bf \RgR} &   \ExpExpTime{}   & \ExpTime{} & \PTime{} \\
       \hline
       {\bf \RfRrRoR}    &  \ExpExpTime{}$^{\star}$  & \ExpExpTime{}$^{\star}$ & \PTime{}$^{\star}$\\
       \hline
       {\bf \RfRgR}   & \ExpExpTime{}$^{\star}$ & \ExpExpTime{}$^{\star}$ & \PTime{}$^{\star}$\\
       \hline
       {\bf \RwRgRfRrRoR}   & \ExpExpTime{}$^{\star}$  & \ExpTime{}$^{\star}$ & \ExpTime{}$^{\star}$\\
       \hline
       {\bf \RwRgR}   & \ExpExpTime{}  & \ExpTime{} & \ExpTime{}\\
       \hline
       {\bf \RwRfRrRoR}   & \ExpExpTime{}$^{\star}$  & \ExpExpTime{}$^{\star}$ & \ExpTime{}$^{\star}$\\
       \hline
       {\bf  \RwRfRgR}   &  \ExpExpTime{}$^{\star}$ & \ExpExpTime{}$^{\star}$ & \ExpTime{}$^{\star}$\\
       \hline
       {\bf  \RgRbRtRsR}  & \ExpExpTime{}$^{\star}$ & \ExpExpTime{}$^{\star}$ & \ExpTime{}$^{\star}$\\
       \hline
\end{tabular}
\end{center}
\caption{Combined and Data Complexities}
\end{table}

\paragraph{Paper Organization}
%
In Section \ref{sec:prelim}, basic definitions and results about existential
rules are recalled. Section \ref{sec:gbts} introduces the \RgRbRtRsR class and 
specifies its relationships with the \RwRfRgR class.
Section \ref{sec:gbts-algo} is devoted to an algorithm for
entailment with \RgRbRtRsR and to the associated complexity results.
The next sections consider increasingly simpler classes, namely
\RfRgR and \RfRrRoR (Section \ref{sec:fg}) and body-acyclic
\RfRgR/\RfRrRoR rules (Section \ref{sec:ba}); several reductions are
provided, which provide tight lower bounds and allow to completely
classify these classes for data and combined complexities. Related work is reviewed in Section \ref{sec-related}. 

This article is an extended version of two papers published at IJCAI 2011 \cite{bmrt:11}  and KR 2012 \cite{tbmr:12}, respectively.
It provides detailed proofs of the results presented in these conference papers and benefits from further clarifications concerning the \RgRbRtRsR algorithm, stemming from Micha\"{e}l Thomazo's PhD thesis \cite{thomazo:13}. Furthermore, it contains complexity results for new classes of rules which complement the picture, namely \RgRfRrRoR, \RwRgRfRrRoR and \RwRfRrRoR.

%% file: sec-preliminaries.tex

We assume the reader to be familiar with syntax and semantics of first-order logic.
We do not consider functional symbols except constants, hence a term is simply a 
variable or a constant. An \emph{atom} is thus of the form $p(t_1, \ldots, t_k)$ where $p$ is a predicate with arity $k$, and the $t_i$ are terms. If not otherwise specified, a conjunction is a finite conjunction of atoms. We denote it by $C[\bold{x}]$, where $\bold{x}$ is the set of variables occurring in $C$.


A \emph{fact} is the existential closure of a conjunction.\footnote{Note that hereby we generalize the classical notion of a (ground) fact in order to take existential variables into account. This is in line with the notion of a \emph{database instance} in database theory, where the existentially quantified variables are referred to as \emph{labeled nulls}.} A Boolean \emph{conjunctive query} (CQ) has the same form as a fact. We may also represent conjunctions of atoms, facts, and CQs as sets of atoms. Given an atom or a set of atoms $A$, we denote by $\vars{A}$ the set of variables, and by $\terms{A}$ the set of terms, occurring in $A$. Given conjunctions $F$ and $Q$, a \emph{homomorphism}  $\pi$ from $Q$ to $F$ is a substitution of $\vars{Q}$ by $\terms{F}$ such that $\pi(Q) \subseteq F$ (we say that $Q$ \emph{maps} to $F$ by $\pi$). 
For convenience, we often assume the domain of $\pi$ extended to $\terms{Q}$, by mapping constants to themselves.
Given an atom $ a = p(t_1, \ldots, t_n)$, we let $\pi(a) = p(\pi(t_1), \ldots, \pi(t_n))$ and similarly for a set of atoms.
 First-order semantic entailment is denoted by  $\models$.
  It is well-known that, given two facts or CQs $F$ and $Q$, $F \models Q$ iff there is a homomorphism from $Q$ to $F$.

\begin{definition}[Existential Rule] An \emph{existential rule} (or simply \emph{rule} when not ambiguous) is a first-order formula:

 $$R = \forall \bold{x} \forall \bold{y}\big(B[\bold{x},\bold{y}] \rightarrow \exists \bold{z} H[\bold{y},\bold{z}]\big),$$ 
 
 where $B$ is a conjunction, called the \emph{body} of $R$ (also denoted by $\bod{R}$), and $H$ is a conjunction called the \emph{head} of $R$ (denoted by $\head{R}$). The \emph{frontier} of $R$, denoted by $\fr{R}$, is the set of variables $\bold{y} = \vars{B} \cap \vars{H}$ occurring both in the rule's body and head.
\end{definition}

Note that an existential rule could be equivalently defined as the formula $\forall \bold{y} (\exists \bold{x} B[\bold{x},\bold{y}] \rightarrow \exists \bold{z} H[\bold{y},\bold{z}])$.  In the following, we will omit quantifiers since there is no ambiguity.


A \emph{knowledge base} (KB) $\mathcal K = (F, \mathcal R)$ is composed of a fact (in database terms: a database instance) $F$ and a finite set of rules (in database terms: a TGD set) $\mathcal R$. W.l.o.g. we assume that the rules have pairwise \emph{disjoint} sets of variables. We denote by $\mathcal C$  the set of constants occurring in $(F, \mathcal R)$ and by $T_0$ (called the set of ``initial terms'') the set $\vars{F} \cup \mathcal C$, i.e., $T_0$ includes not only the terms from $F$ but also the constants occurring in rules.
Next, we formally define the problem considered by us.


\begin{definition}[\textsc{BCQ-Entailment}]
The decision problem of \emph{entailment of Boolean conjunctive queries under existential rules} is defined as follows:
\begin{itemize}
\item \textsc{Name: BCQ-Entailment}
\item \textsc{Input}: A knowledge base $\mathcal{K}=(F,\mathcal{R})$ and a Boolean conjunctive query $Q$.
\item \textsc{Output}: YES iff $\mathcal{K} \models Q$, NO otherwise.
\end{itemize}
\end{definition}

Depending on which parts of the input are assumed to be fixed, we distinguish the following three complexity notions when investigating \textsc{BCQ-Entailment}:

\begin{itemize}
\item When considering \emph{data complexity}, we assume $\mathcal{R}$ and $Q$ to be fixed, only $F$ (the data) can vary.
\item When investigating \emph{query complexity}, $F$ and $\mathcal{R}$ are fixed and $Q$ may vary.
\item In case of \emph{combined complexity}, $F$, $\mathcal{R}$ and $Q$ can all change arbitrarily.
\end{itemize}

We now define the fundamental notions of rule application and $\mathcal R$-derivation, which we relate to the \emph{chase} procedure in databases. 

\begin{definition}[Application of a Rule] A rule $R$ is \emph{applicable} to a fact $F$ if there is a homomorphism $\pi$ from $\bod{R}$ to $F$; the result of the \emph{application of $R$ to $F$ w.r.t. $\pi$} is a fact $\alpha(F, R, \pi) = F \cup \pi^\mathrm{safe}(\head{R})$ where $\pi^\mathrm{safe}$ is a substitution of $\head{R}$, which replaces each $x \in \fr{R}$ with $\pi(x)$, and each other variable with a fresh variable. As $\alpha(F,R,\pi)$ does not depend on the whole $\pi$, but only on $\pi_{|\fr{R}}$ (the restriction of $\pi$ to $\fr{R}$), we also write $\alpha(F, R, \pi_{|\fr{R}})$.
\end{definition}

\begin{example}\label{ex-DT}    
Let $F = \{r(a,b), r(c,d),p(d)\}$
and $R = r(x,y) \rightarrow r(y,z)$. 
There are two applications of $R$ to $F$, respectively by $\pi_1 = \{x{\mapsto}a, y{\mapsto}b)\}$ and $\pi_2 = \{x{\mapsto}c,y{\mapsto}d\}$. We obtain $F_1 = \alpha(F, R, \pi_1) = F \cup \{r(b, z_1)\}$ and $F_2 = \alpha(F, R, \pi_2) = F \cup \{r(d, z_2)\}$.
\end{example}

%
\begin{definition}[$\mathcal R$-Derivation] Let $F$ be a fact and $\mathcal R$ be a set of rules. An $\mathcal R$-derivation (from $F$ to $F_k$) is a finite sequence $(F_0=F),(R_0, \pi_0, F_1), \ldots, (R_{k-1}, \pi_{k-1}, F_k)$
such that  for all $0 \leq i < k$, $R_i \in \mathcal R$ and $\pi_i$ is a homomorphism from $\bod{R_i}$ to $F_i$ such that  $F_{i+1} = \alpha(F_i, R_i, \pi_i)$. When only the successive facts are needed, we note $(F_0=F),F_1, \ldots, F_k$.
\end{definition}

The following theorem essentially stems from earlier results on conceptual graph rules \cite{salvat-mugnier:96}.

\begin{theorem}[Soundness and Completeness of $\mathcal R$-Derivations] Let $\mathcal{K}=(F,\mathcal{R})$ be a KB and $Q$ be a CQ. Then $F, \mathcal R \models Q$ iff there exists an $\mathcal R$-derivation from $F$ to some $F_k$ such that $F_k \models Q$.
\end{theorem}


It follows that a breadth-first forward chaining mechanism yields a positive answer in finite time when $\mathcal K \models Q$. Let $F_0=F$ be the initial fact. Each step is as follows: (1) check if $Q$ maps by homomorphism to the current fact, say $F_{i-1}$ at step $i$ ($i > 0$): if it is the case, $Q$ has a positive answer; (2) otherwise, produce a fact $F_{i}$ from $F_{i-1}$, by computing all new homomorphisms from each rule body to $F_{i-1}$, then performing all corresponding rule applications. A homomorphism to $F_{i-1}$ is said to be \emph{new} if it has not been already computed at a previous step, i.e., it uses at least an atom added at step $i-1$. The fact $F_k$ obtained at the end of step $k$ is called the\emph{ $k$-saturation} of $F$ and is denoted by $\alpha_k(F, \mathcal R)$; we define the \emph{saturation} of $F$ by $\mathcal R$ as 
$\alpha_\infty(F, \mathcal R) = \cup_{k \geq 0} \alpha_k(F, \mathcal R)$. 



Preceding notions are closely related to classical database notions. Forward  chaining (with existential rules) is known as the \emph{chase} (with TGDs) \cite{mms:79,abu:79}. Hence, the notion of an $\mathcal R$-derivation corresponds to a chase sequence. The chase is seen as a tool for computing the saturation of a database with respect to a set of TGDs. Several variants of the chase are known, which all produce a result homomorphically equivalent to $\alpha_\infty(F, \mathcal R)$. The chase yields a \emph{canonical model } of $(F, \mathcal R)$, which is isomorphic to the output of the chase, and has the property of being \emph{universal}, which means that it maps by homomorphism to any model of $(F, \mathcal R)$. It follows that $(F, \mathcal R) \models Q$ if and only if $Q$ maps by homomorphism to $\alpha_\infty(F, \mathcal R)$ \cite{deutsch:08} (and \cite{blms:11} in the setting of existential rules).

We now formally specify some other notions that we have already introduced informally. A fact can naturally be seen as a hypergraph whose nodes are the terms in the fact and whose hyperedges encode the atoms. The \emph{primal graph} of this hypergraph has the same set of nodes and there is an edge between two nodes if they belong to the same hyperedge. The \emph{treewidth} of a fact is defined as the treewidth of its associated primal graph. Given a fact $F_i$, a derivation $S$ yielding $F_i$, or a tree decomposition $\mathfrak{T}$ of $F_i$, we let $\atoms{S}=\atoms{\mathfrak{T}} = F_i$.

\begin{definition}[Tree Decomposition and Treewidth of a Fact] Let $F$ be a (possibly infinite) fact. A \emph{tree decomposition} of $F$ is a (possibly infinite) tree $\mathfrak{T}$, with set of nodes $\mathcal B = \{B_0, \ldots, B_k, \ldots\}$, and two functions $\fun{terms}: \mathcal{B} \to 2^{\terms{F}}$ and $\fun{atoms}: \mathcal{B} \to 2^{{F}}$, where:
\begin{enumerate}

    \item $\bigcup_{i} \terms{B_i} =\terms{F}$;

    \item $\bigcup_{i} \atoms{B_i} = F$;

    \item  For each $B_i \in \mathcal B$ holds $\terms{\atoms{B_i}}\subseteq \terms{B_i}$;

\item For each term $e$ in $F$, the subgraph of $\mathfrak{T}$ induced by the nodes $B_i$ with $e \in \terms{B_i}$ is connected (``Running intersection property'').
\end{enumerate}

The \emph{width} of a tree decomposition $\mathfrak{T}$ is the size of the largest node of $\mathfrak{T}$, minus 1. The \emph{treewidth} of a fact $F$ is the minimal width among all its possible tree decompositions.
\end{definition}

A set of rules $\mathcal R$ is called a \emph{bounded-treewidth set (\RbRtRsR)} if for any fact $F$ there exists an integer $b$ such that the treewidth of any fact $F'$ that can be $\mathcal R$-derived from $F$ is less or equal to $b$. 
The \textsc{entailment} problem is decidable when $\mathcal R$ is \RbRtRsR \cite{cali-gottlob-kifer:08,blms:11}. The main argument of the proof, introduced by \citeauthor{cali-gottlob-kifer:08}, relies on the observation that $\mathcal K \wedge \neg Q$ enjoys the bounded-treewidth model property, i.e,  has a model with bounded treewidth when it is satisfiable, i.e., when  $\mathcal K \not \models Q$. Decidability follows from the decidability of the satisfiability problem for classes of  first-order formulas having the bounded-treewidth property, a result from Courcelle \cite{courcelle:90}. However, the proof of this latter result does not lead (or at least not directly) to an algorithm for \textsc{BCQ-Entailment} under  \RbRtRsR rules. We now focus on ``concrete'' subclasses of  \RbRtRsR. 


A rule $R$ is \emph{guarded} {(\RgR)} if there is an atom $a \in \body{R}$ with $\vars{\body{R}} \subseteq \vars{a}$. We call $a$ a guard of the rule. $R$ is \emph{ weakly guarded} {(\RwRgR)} if there is $a \in \body{R}$ (called a weak guard) that contains all affected variables from $\body{R}$. The notion of affected variable is relative to the rule set: a variable is affected if it occurs only in affected predicate positions, which are positions that may contain an existential variable generated by forward chaining \citep{fagin-kolaitis-al:05}. More precisely, the set of affected positions w.r.t. $\mathcal R$ is the smallest set that satisfies the following conditions: (1)  if there is a rule head containing an atom with predicate $p$ and an existentially quantified variable in position $i$, then position $(p,i)$ is affected; (2) if a rule body contains a variable $x$ appearing in affected positions only and $x$ appears in the head of this rule in position $(q,j)$ then $(q,j)$ is affected. The important point is that a rule application necessarily maps non-affected variables to terms from the initial fact (and more generally to $T_0$ in the case where rules may add constants). The \RgR and \RwRgR rule classes were described and their complexity was analyzed by \citeauthor{cali-gottlob-kifer:08} \citeyear{cali-gottlob-kifer:08,DBLP:journals/jair/CaliGK13}.

$R$ is \emph{frontier-one} {(\RfRrRoR)} if $|\fr{R}|=1$ and it is \emph{guarded frontier-one} (\RgRfRrRoR) if it is both \RgR and \RfRrRoR. $R$ is \emph{frontier-guarded }{(\RfRgR)} if there is  $a \in \body{R}$  with $\vars{\fr{R}} \subseteq \vars{a}$. The weak versions of these classes---\emph{weakly frontier-one} {(\RwRfRrRoR)}, \emph{weakly guarded frontier-one} (\RwRgRfRrRoR) and \emph{weakly frontier-guarded} {(\RwRfRgR)} rules---are obtained by relaxing the above criteria so that they only need to be satisfied by the affected variables. The syntactic inclusions holding between these \RbRtRsR subclasses are displayed in Fig.~\ref{complexities-min}.

%% file: sec-gbts.tex

This section introduces greedy bounded-treewidth sets of rules (\RgRbRtRsR). It is pointed out that \RgRbRtRsR strictly contains the \RwRfRgR class. However, in some sense, \RgRbRtRsR is not more expressive than \RwRfRgR: indeed, we exhibit a polynomial translation $\tau$ from any KB $\mathcal K = (F, \mathcal R)$ to another KB $\tau (\mathcal K) = (\tau(F), \tau(\mathcal R))$ with $\tau(\mathcal R)$ being \RwRfRgR, which satisfies the following property:
 if $\mathcal R$ is \RgRbRtRsR, then $\mathcal  K$ and $\tau(\mathcal K)$ are equivalent.  This translation can thus be seen as a polynomial reduction from \textsc{BCQ-Entailment} under \RgRbRtRsR  to \textsc{BCQ-Entailment} under \RwRfRgR.


\subsection{Definition of the \RgRbRtRsR Class}

In a greedy derivation, every rule application maps the frontier of the rule in a special way: all the frontier variables that are mapped to terms introduced by rule applications are jointly mapped to variables added by one \emph{single} previous rule application.



\begin{definition}[Greedy Derivation]\label{def-greedy}
     An $\mathcal R$-der\-i\-va\-tion $(F_0=F), \ldots, F_k$ is said to be \emph{greedy} if, for all $i$ with $0 < i < k$,  there is $j < i$ such that $\pi_i(\fun{fr}(R_i)) \subseteq \fun{vars}(A_j) \cup \fun{vars}(F_0) \cup \mathcal C$, where $A_j = \pi_j^\mathrm{safe}(\fun{head}(R_j))$.
\end{definition}


Note that, in the above definition, any $j < i$ can be chosen if $\fr{R_i}$ is mapped to $\fun{vars}(F_0) \cup \mathcal C$.

\begin{example}[Non-Greedy Derivation]\label{ex-non-greedy} Let $\mathcal{R} = \{R_0, R_1\}$ where: \\
$$
\begin{array}{rll}
R_0 & = & r_1(x,y) \rightarrow r_2(y,z) \mbox{ and}\\
R_1 & = & r_1(x,y) \wedge r_2(x,z) \wedge r_2(y,t) \rightarrow r_2(z,t)\\
\end{array}
$$
Let $F_0 = \{r_1(a,b)\wedge r_1(b,c)\}$ and $S= F_0,\ldots, F_3$ with:\\
$$\begin{array}{rll@{\ \ \ }l}
F_1 & = & \alpha(F_0,R_0,\{(y{\mapsto}b)\}) & \mbox{with } A_0 = \{r_2(b,x_1)\},\\
F_2 & = & \alpha(F_1,R_0,\{(y{\mapsto}c)\}) & \mbox{with } A_1 = \{r_2(c,x_2)\},\\
F_3 & = & \alpha(F_2,R_1,\pi_2) & \mbox{with }{\pi_2}_{\mid \fr{R_1}} = \{z{\mapsto}x_1,t{\mapsto}x_2\}\\
\end{array}$$
Then $\fr{R_1} = \{z,t\}$ is mapped to newly introduced variables in $F_3$, however, there is no $A_j$ such that $\{\pi_2(z),\pi_2(t)\} \subseteq \vars{A_j}$. Thus $S$ is not greedy.
\end{example}

Any greedy derivation can be associated with a so-called derivation tree, formally defined below. Intuitively, the root of the tree corresponds to the initial fact $F_0$, and each other node corresponds to a rule application of the derivation. Each node is labeled by a set of terms and a set of atoms.
The set of terms assigned to the root is $T_0$, i.e., it includes the constants that are mentioned in rule heads. Moreover, $T_0$ is included in the set of terms of all nodes. This ensures that the derivation tree is a decomposition tree of the associated derived fact.

\begin{definition}[Derivation Tree]\label{def-dt} Let $S = (F_0=F), \ldots, F_k$ be a greedy derivation. The \emph{derivation tree} assigned to $S$, denoted by $\mathit{DT}(S)$, is a tree $\mathfrak{T}$ with nodes $\mathcal B = \{B_0, \ldots, B_k, \ldots\}$ and two functions $\fun{terms}: \mathcal{B} \to 2^{\terms{F_k}}$ and $\fun{atoms}: \mathcal{B} \to 2^{{F_k}}$, defined as follows:

\begin{enumerate}
    \item Let $T_0=\vars{F} \cup \mathcal C$. The root of the tree is $B_0$ with $\fun{terms}(B_0) = T_0$ and $\fun{atoms}(B_0) = \fun{atoms}(F)$.

    \item For $0 < i \leq k$, let $R_{i-1}$ be the rule applied according to homomorphism $\pi_{i-1}$ to produce $F_i$; then $\fun{terms}(B_{i}) = \fun{vars}(A_{i-1}) \cup T_0$ and $\fun{atoms}(B_{i}) = \fun{atoms}(A_{i-1})$. The parent of $B_i$ is the node $B_j$ for which $j$ is the smallest integer such that $\pi_{i-1}(\fr{R_{i-1}}) \subseteq \terms{B_j}$.
\end{enumerate}
The nodes of $\mathit{DT}(S)$ are also called \emph{bags}.
\end{definition}

\begin{example}[Example \ref{ex-DT} contd.]\label{ex-DT2}
 We consider  $F = \{r(a,b), r(c,d),p(d)\}$
and $R = r(x,y) \rightarrow r(y,z)$. We build $\mathit{DT}(S)$ for $S = (F_0=F), (R,\pi_1,F_1), (R,\pi_2,F_2)$ as depicted in Figure~\ref{DT}.
 Let $B_0$ be the root of $\mathit{DT}(S)$. $(R,\pi_1)$ yields a bag $B_1$ child of $B_0$, with $\atoms{B_1} = \{r(b, z_1)\}$ and $\terms{B_1}= \{a,b,c,d,z_1\}$. $(R,\pi_2)$ yields a bag $B_2$ with $\atoms{B_2} = \{r(d, z_2)\}$ and $\terms{B_2}= \{a,b,c,d,z_2\}$. $\fr{R_0} = \{y\}$ and $\pi_2(y) = d$, which is both in $\terms{B_0}$ and $\terms{B_1}$, $B_2$ is thus added as a child of the highest bag, i.e., $B_0$.
$R$ can be applied again, with homomorphisms $\pi_3 = \{x{\mapsto}b, y{\mapsto}z_1\}$ and $\pi_4 = \{x{\mapsto}d, y{\mapsto}z_2\}$, which leads to create two bags, $B_3$ and $B_4$, under $B_1$ and $B_2$ respectively. Clearly, applications of $R$ can be repeated indefinitely.
\end{example}

\begin{figure}
\begin{center}
\begin{tikzpicture}[line cap=round,line join=round,>=triangle 45,x=1.0cm,y=1.0cm]
\clip(-4.3,-3.45) rectangle (5.56,5.07);
\draw [rotate around={0:(0.76,3.62)}] (0.76,3.62) ellipse (1.52cm and 0.82cm);
\draw [rotate around={0:(-1.68,1.04)}] (-1.68,1.04) ellipse (1.52cm and 0.82cm);
\draw [rotate around={0:(3.53,1.05)}] (3.53,1.05) ellipse (1.52cm and 0.82cm);
\draw [rotate around={0:(-1.78,-2.02)}] (-1.78,-2.02) ellipse (1.52cm and 0.82cm);
\draw [rotate around={0:(3.53,-2.02)}] (3.53,-2.02) ellipse (1.52cm and 0.82cm);
\draw (-0.26,4.36) node[anchor=north west] {$a$};
\draw (-0.64,3.85) node[anchor=north west] {$r(a,b)$};
\draw (0.38,3.85) node[anchor=north west] {$r(c,d)$};
\draw (1.41,3.87) node[anchor=north west] {$p(d)$};
\draw (0.21,4.47) node[anchor=north west] {$b$};
\draw (0.64,4.36) node[anchor=north west] {$c$};
\draw (1.02,4.45) node[anchor=north west] {$d$};
\draw (-2.87,1.72) node[anchor=north west] {$a$};
\draw (-2.45,1.81) node[anchor=north west] {$b$};
\draw (-1.98,1.72) node[anchor=north west] {$c$};
\draw (-1.55,1.81) node[anchor=north west] {$d$};
\draw (-1.11,1.72) node[anchor=north west] {$z_1$};
\draw (-3.04,-1.38) node[anchor=north west] {$a$};
\draw (-2.62,-1.3) node[anchor=north west] {$b$};
\draw (-2.19,-1.38) node[anchor=north west] {$c$};
\draw (-1.83,-1.3) node[anchor=north west] {$d$};
\draw (2.26,1.66) node[anchor=north west] {$a$};
\draw (2.66,1.75) node[anchor=north west] {$b$};
\draw (3.17,1.66) node[anchor=north west] {$c$};
\draw (3.64,1.75) node[anchor=north west] {$d$};
\draw (2.34,-1.43) node[anchor=north west] {$a$};
\draw (2.68,-1.34) node[anchor=north west] {$b$};
\draw (3.05,-1.43) node[anchor=north west] {$c$};
\draw (3.41,-1.34) node[anchor=north west] {$d$};
\draw (4.15,1.64) node[anchor=north west] {$z_2$};
\draw (-1.45,-1.38) node[anchor=north west] {$z_1$};
\draw (-1,-1.41) node[anchor=north west] {$z_3$};
\draw (3.81,-1.43) node[anchor=north west] {$z_2$};
\draw (4.32,-1.45) node[anchor=north west] {$z_4$};
\draw (-2.45,1.23) node[anchor=north west] {$r(b,z_1)$};
\draw (2.79,1.19) node[anchor=north west] {$r(d,z_2)$};
\draw (-2.58,-1.96) node[anchor=north west] {$r(z_1,z_3)$};
\draw (2.81,-1.98) node[anchor=north west] {$r(z_2,z_4)$};
\draw (-1.06,5.05) node[anchor=north west] {$B_0$};
\draw (-3.34,2.32) node[anchor=north west] {$B_1$};
\draw (1.64,2.11) node[anchor=north west] {$B_2$};
\draw (-3.43,-0.75) node[anchor=north west] {$B_3$};
\draw (2.04,-0.68) node[anchor=north west] {$B_4$};
\draw (0.69,2.8)-- (-1.81,1.86);
\draw (0.69,2.8)-- (3.59,1.87);
\draw (3.59,0.23)-- (3.59,-1.2);
\draw (-1.83,0.22)-- (-1.85,-1.2);
\draw (-2.11,2.98) node[anchor=north west] {$(R_0,b)$};
\draw (2.85,2.92) node[anchor=north west] {$(R_0,d)$};
\draw (-1.45,-0.17) node[anchor=north west] {$(R_0,z_1)$};
\draw (4,-0.21) node[anchor=north west] {$(R_0,z_2)$};
\end{tikzpicture}
\end{center}

\caption{ Derivation tree of Example \ref{ex-DT2}. Only the image of the single frontier variable from $R_0$ is mentioned in edge labels.}\label{DT}
\end{figure}
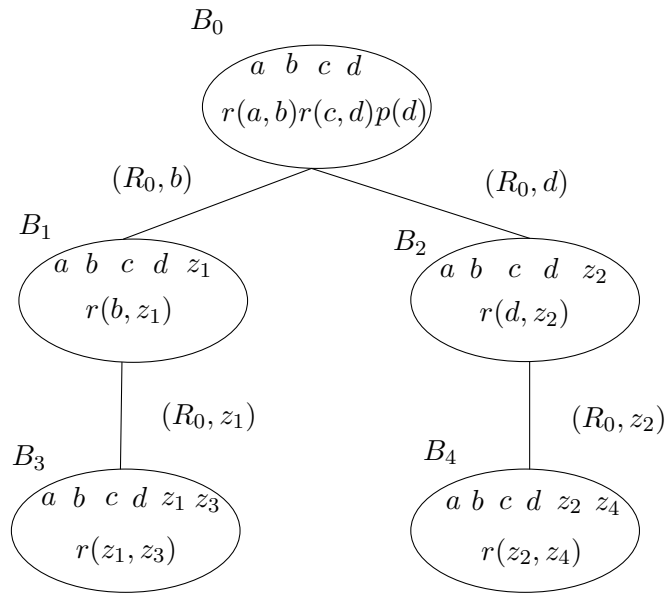

Note that, in the second point of the definition of a derivation tree, there is at least one $j$ with $\pi_{i-1}(\fr{R_{i-1}}) \subseteq \terms{B_j}$ because $S$ is greedy. The following property is easily checked, noticing that $T_0$ occurs in each bag, which ensures that the running intersection property is satisfied.

\begin{property} Let $S = F_0 \ldots, F_k$ be a greedy derivation. Then $\mathit{DT}(S)$ is a tree decomposition of $F_k$ of width bounded by
$|\vars{F}| + |\mathcal C| + \max(|\vars{\head{R}}|_{R \in \mathcal R})$.
\end{property}

\begin{definition}[Greedy Bounded-Treewidth Set of Rules (\RgRbRtRsR)]\label{def-gbts}
$\mathcal R$ is said to be a \emph{greedy} \emph{bounded-treewidth set} (\RgRbRtRsR) if (for any fact $F$) any $\mathcal R$-derivation (of $F$) is greedy.
\end{definition}

The class \RgRbRtRsR is a strict subclass of \RbRtRsR and does not contain \RfReRsR (e.g., in Example \ref{ex-non-greedy}: $\mathcal R$ is \RfReRsR but not \RgRbRtRsR). It is nevertheless an expressive subclass of \RbRtRsR since it contains \RwRfRgR:

\begin{property}
Any set of \RwRfRgR rules is \RgRbRtRsR.
\end{property}

\begin{proof}
 Let $\mathcal{R}$ be a \RwRfRgR rule set. Given any $\mathcal R$-derivation, consider the application of a rule $R_i$, with weak frontier-guard $g$. Let $a =\pi_i(g)$. Either $a \in F$ or $a \in A_j$ for some $j \leq i$. In the first case, $\pi_i(\fun{fr}(R_i)) \subseteq \terms{F} \subseteq \fun{vars}(F_0) \cup \mathcal C$; in the second case, $\pi_i(\fun{fr}(R_i)) \subseteq \terms{A_j} \subseteq  \fun{vars}(A_j) \cup \mathcal C$. We conclude that $\mathcal{R}$ is \RgRbRtRsR.
 \end{proof}

The obtained inclusion is strict since there are \RgRbRtRsR rule sets which are not \RwRfRgR as shown in the following example.

\begin{example}[\RgRbRtRsR but not \RwRfRgR]\label{ex-gbts-not-wfg}
Let $ R = r_1(x,y)\wedge r_2(y,z) \rightarrow r(x,x')\wedge r(y,y')\wedge r(z,z')\wedge r_1(x',y')\wedge r_2(y',z')$.
$\{R\}$ is \RgRbRtRsR, but not \RwRfRgR  (nor \RfReRsR). First, let us notice that all positions of $r_1$ and $r_2$ are affected, and that $x,y$ and $z$ belong to the frontier of $R$. Thus, $\{R\}$ is not \RwRfRgR. Moreover, let us consider $F = \{r_1(a,b), r_1(b,c)\}$. $R$ is applicable to $F$, which leads to create $r(a,x_1),r(b,y_1),r(c,z_1),r_1(x_1,y_1),$ and $r_2(y_1,z_1)$, as shown in Fig.~\ref{fig-gbts-not-wfg}. $R$ is thus newly applicable, mapping its frontier to $x_1$, $y_1$, and $z_1$. This can be repeated infinitely often, therefore $\{R\}$ is not \RfReRsR. Last, the only way to map the body of $R$ to terms that do not belong to an arbitrary initial fact is to map the frontier of $R$ to terms that have been created in the same bag (for instance, to the atoms in $A_1$ in Fig.~\ref{fig-gbts-not-wfg}), thus ensuring that $\{R\}$ is \RgRbRtRsR.
\end{example}

\begin{figure}
\begin{center}\begin{tikzpicture}[line cap=round,line join=round,>=triangle 45,x=1.0cm,y=1.0cm]
\draw [->] (-1.96,0.62) -- (-0.12,0.62); 
\draw [->] (-0.12,0.62) -- (1.62,0.62); 
\draw [rotate around={0:(-0.02,0.64)}] (-0.02,0.64) ellipse (3.03cm and 0.8cm); 
\draw (-1.56,0.64) node[anchor=north west] {$r_1$}; 
\draw (0.26,0.62) node[anchor=north west] {$r_2$}; 
\draw [->] (-1.96,3.08) -- (-0.12,3.08); 
\draw [->] (-0.12,3.08) -- (1.62,3.08); 
\draw [rotate around={0:(-0.02,3.1)}] (-0.02,3.1) ellipse (3.03cm and 0.8cm); 
\draw (-1.56,3.1) node[anchor=north west] {$r_1$}; 
\draw (0.26,3.08) node[anchor=north west] {$r_2$}; 
\draw [->] (-1.96,0.62) -- (-1.96,3.08); 
\draw [->] (-0.12,0.62) -- (-0.12,3.08); 
\draw [->] (1.62,0.62) -- (1.62,3.08); 
\draw (-2.64,2.2) node[anchor=north west] {$r$}; 
\draw (-0.68,2.22) node[anchor=north west] {$r$}; 
\draw (1.04,2.22) node[anchor=north west] {$r$}; 
\draw (-2.38,0.74) node[anchor=north west] {$a$}; 
\draw (-0.54,0.64) node[anchor=north west] {$b$}; 
\draw (1.36,0.62) node[anchor=north west] {$c$}; 
\draw (-2.42,3.7) node[anchor=north west] {$x_1$}; 
\draw (-0.62,3.78) node[anchor=north west] {$y_1$}; 
\draw (1.16,3.74) node[anchor=north west] {$z_1$};
\draw (3.4,3.4) node[anchor=north west] {$A_1$};
\draw (3.4,0.9) node[anchor=north west] {$F$};
\begin{scriptsize}
\fill [color=black] (-1.96,0.62) circle (1.5pt); 
\fill [color=black] (-0.12,0.62) circle (1.5pt); 
\fill [color=black] (1.62,0.62) circle (1.5pt); 
\fill [color=black] (-1.96,3.08) circle (1.5pt); 
\fill [color=black] (-0.12,3.08) circle (1.5pt); 
\fill [color=black] (1.62,3.08) circle (1.5pt);
\end{scriptsize}
\end{tikzpicture}\end{center}
\caption{\label{fig-gbts-not-wfg} Illustration of Example \ref{ex-gbts-not-wfg}}
\end{figure}
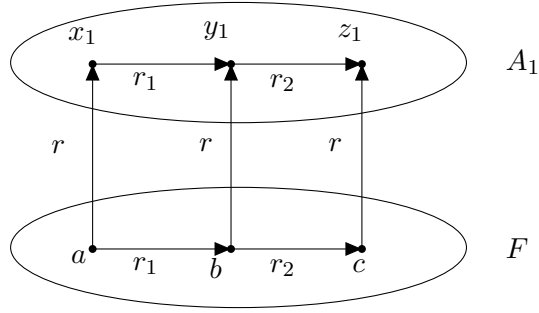


\subsection{A Translation into Weakly Frontier-Guarded Rules}

Next we will present a translation applicable to any set of existential rules. This translation can be computed in polynomial time, its result is always \RwRfRgR and it is guaranteed to preserve query answers if the input is \RgRbRtRsR.

The translation introduces two new predicates: a unary predicate $\mathit{initial}$ and a predicate \emph{samebag} of higher arity. Intuitively, $\mathit{initial}$ will mark terms from the initial fact $F$, as well as constants added by rule applications, and $\mathit{samebag}$ will gather terms that are ``in the same bag''.

\begin{definition}[\RwRfRgR Translation]
Let $\mathcal K = (F, \mathcal R)$ be a KB. The \emph{\RwRfRgR translation} of $\mathcal K$ is the KB $\tau(\mathcal K) = (\tau(F), \tau(\mathcal R))$ where
$\tau(F) = F \cup \{\mathit{initial}(t) | t \in \terms{F}\}$
and 
$\tau(\mathcal R) =\mathcal{R}^\mathrm{same} \cup \mathcal{R}^\mathrm{trans}$, where $\mathcal{R}^\mathrm{same}$ and $\mathcal{R}^\mathrm{trans}$ are defined as follows (where $\mathit{initial}$ is a fresh unary predicate and \emph{samebag} is fresh predicate with arity $q = \max (|\terms{\head{R}}|_{R \in \mathcal R}) + |T_0|$):
\begin{itemize}
\item $\mathcal{R}^\mathrm{same}$ contains the following rules:
\begin{enumerate}
\item[] $R^\mathrm{same}_1 = \mathit{initial}(x) \rightarrow \mathit{samebag}(x, \ldots, x)$,
\item[] $R^\mathrm{same}_2 = \mathit{samebag}(x_1, x_2, \ldots, x_q) \wedge \mathit{initial}(x)  \rightarrow \mathit{samebag}(x, x_2 \ldots, x_q)$,
\item[] one rule of the following type for each $1 \leq i \leq q$:\\
       $R^\mathrm{same}_{3i} = \mathit{samebag}(x_1, \ldots, x_i, \ldots, x_q)  \rightarrow \mathit{samebag}(x_i, \ldots, x_1, \ldots x_q)$, and
\item[] $R^\mathrm{same}_4 = \mathit{samebag}(x_1, \ldots, x_{q-1}, x_q)  \rightarrow \mathit{samebag}(x_1, \ldots, x_{q-1}, x_1)$.
\end{enumerate}
\item $\mathcal{R}^\mathrm{trans}$ contains one translated rule $\tau(R)$ for every rule $R$ from $\mathcal R$:
 for some rule $R = B[\bold{x},\bold{y}] \rightarrow H[\bold{y},\bold{z}]$ with $c_1, \ldots, c_k$ being the constants occurring in $H$, we let
  $\tau(R) = B[\bold{x},\bold{y}] \wedge \mathit{samebag}(\bold{y}, \bold{v}) \rightarrow H[\bold{y},\bold{z}] \wedge \mathit{samebag}(\bold{y},\bold{z}, \bold{w}) \wedge_{i:1, \ldots,k} \mathit{initial}(c_i)$,\\
  where $\bold{w}\subseteq\bold{v}$ and $\bold{v}$ is a set of fresh variables.
  \end{itemize}
\end{definition}


 Intuitively, Rules $R^\mathrm{same}_1$ and $R^\mathrm{same}_2$ express that the initial terms (as well as constants added by rule applications) are in all bags;  rules $R^\mathrm{same}_{3i}$ and rule $R^\mathrm{same}_4$ respectively allow any permutation and any duplication of arguments in an  atom with predicate  \emph{samebag}. In the translation of the rules from $\mathcal R$, the sets of variables $\bold{v}$ and  $\bold{w}$ are used to fill the atoms with predicate  \emph{samebag} to obtain arity $q$.

\begin{property}
For any set $\mathcal{R}$ of existential rules, $\tau(\mathcal{R})$ is \RwRfRgR.
\end{property}

\begin{proof}
  $\mathcal{R}^\mathrm{same} \setminus \{R^\mathrm{same}_2\}$ is \emph{guarded}. $R^\mathrm{same}_2$ is \RfRgR.  No rule affects the position in the unary  predicate $\mathit{initial}$, thus all affected variables in $R^\mathrm{same}_2$ are guarded by the atom with predicate \emph{samebag}, hence $\tau(\mathcal R)$ is \RwRfRgR.
\end{proof}

We next establish that, assuming we do not consider consequences involving $\mathit{initial}$ or $\mathit{samebag}$, $\tau(\mathcal R)$ is sound with respect to $\mathcal R$ and it is even complete in case $\mathcal{R}$ is \RgRbRtRsR.
 
 \begin{property}
 For any Boolean CQ $Q$ over the initial vocabulary, if $\tau(\mathcal K) \models Q$ then $\mathcal K \models Q$. Moreover, if $\mathcal R$ is \RgRbRtRsR, then the reciprocal holds, i.e., $\tau(\mathcal K)$ and $\mathcal K$  are equivalent (w.r.t. the initial vocabulary).
 \end{property}

 \begin{proof}
 $\Rightarrow$: Any $\tau(\mathcal R)$-derivation $\mathcal S'$ from $\tau(F)$ can be turned into an $\mathcal R$-derivation $\mathcal S$ from $F$ by simply ignoring the applications of rules from $\mathcal{R}^\mathrm{same}$ and replacing each application of a rule $\tau(R_i)$ by an application of the rule $R_i$ with ignoring the atoms with predicate \emph{samebag}. Moreover, the facts respectively obtained by both derivations are equal on the initial vocabulary  (i.e., when considering only the atoms with predicate in the initial vocabulary, and up to a variable renaming).
\\
 $\Leftarrow$: We assume that $\mathcal R$ is \RgRbRtRsR. We show that any $\mathcal R$-derivation $\mathcal S= (F_0=F), F_1, \ldots, F_k$ can be turned into a $\tau(\mathcal R)$-derivation $\mathcal S' = (F'_0 =\tau(F)), \ldots, F'_1, \ldots F'_k$ that satisfies: (a) for all $i$ such that $0 \leq i \leq k$, $F_i$ and $F'_i$ are equal on the initial vocabulary; and, (b) for all $i$ such that  $0 \leq i < k$, $F'_{i+1}$ is obtained by applying $\tau(R_i)$ with a homomorphism $\pi'_i$ that extends $\pi_i$. The proof is  by induction on the length $\ell$ of $\mathcal S$. The property is true for $\ell=0$. Assume it is true for $\ell = n$. Consider the application of $R_n$ with homomorphism $\pi_n$ from $\body{R_n}$ to $F_n$. We note $\fr{R_n}= \{y_1 \ldots y_p\}$ such that $\body{\tau(R_n)}$ contains the atom $\mathit{samebag}(y_1, \ldots, y_p, \ldots)$. Since $\mathcal R$ is \RgRbRtRsR, there is an $A_j$ such that some variables from $\fr{R_n}$, say $y_{i_1} \ldots y_{i_m}$ are mapped to $\fun{vars}(A_j)$, and the remaining variables from $\fr{R_n}$, say  $y_{i_{m+1}} \ldots y_{i_p}$ are mapped to $T_0 = \vars{F} \cup \mathcal C$. The application of $\tau(R_j)$ in $\mathcal S'$ has produced a \emph{samebag} atom $s_1$ that contains $\pi_n(y_{i_1}) \ldots \pi_n(y_{i_m})$ (by induction hypothesis (b)). By applying Rules $R^\mathrm{same}_{3i}$ and Rule $R^\mathrm{same}_{4}$, we permute, and duplicate if needed (i.e., if some $y_{i_1} \ldots y_{i_m}$ have the same image by $\pi$), the arguments in $s_1$ to obtain the atom $s_2 = \mathit{samebag}(\pi_n(y_{i_1}), \ldots, \pi_n(y_{i_m}), \ldots)$. Then, with Rule $R^\mathrm{same}_2$, we add each $\pi_n(y_{i_j})$ for $m < j \leq p$ (note that $F'_n$ necessarily contains $\mathit{initial}(\pi_n(y_{i_j})))$ and build the atom $s_3 = \mathit{samebag}(\pi_n(y_{i_{m+1}}), \ldots, \pi_n(y_{i_p}), \pi_n(y_{i_1}) \ldots \pi_n(y_{i_m}), \ldots)$. Finally, with Rules $R^\mathrm{same}_{3i}$, we permute the $p$ first arguments in $s_3$ to obtain $s_4 = \mathit{samebag}(\pi_n(y_{1}), \ldots, \pi_n(y_{p}), \ldots)$. Since $F_n$ and $F'_n$ are equal on the initial vocabulary by induction hypothesis (a), the fact obtained from $F'_n$ after application of the previous rules from $\mathcal{R}^\mathrm{trans}$ is still equal to $F_n$ on the initial vocabulary. We build $\pi'_n$ by extending  $\pi_n$ such that the atom with predicate \emph{samebag} in $\body{\tau(R_n)}$ is mapped to $s_4$. Parts (a) and (b) of the induction property are thus satisfied for $\ell= n+1$.
\end{proof}

%% file: outline-kr.tex



We give here an informal high-level description of the PatSat algorithm (for pattern saturation). Due to the existentially quantified variables in rule heads, a forward chaining mechanism does not halt in general. However, as we have seen in the preceding section, for \RgRbRtRsR, each sequence of rule applications gives rise to a so-called \emph{derivation tree}, which is a decomposition tree of the derived fact; moreover, this tree can be built in a \emph{greedy} way: each rule application produces a new tree node $B$ (called a \emph{bag}), which contains the atoms created by the rule application, such that the derived fact is the union of all bag atoms from this tree. The derived fact is potentially infinite, but thanks to its tree-like structure, the forward chaining process can be stopped after a finite number of rule applications as some periodic behavior will eventually occur.

The PatSat algorithm proceeds in two steps: first, it computes a finite tree, called a \emph{(full) blocked tree}, which finitely represents all possible derivation trees; second, it evaluates a query against this blocked tree. Building a blocked tree relies on the following notions:
\begin{itemize}
\item \emph{bag patterns}: Each bag $B$ is associated with a \emph{pattern} $P$, which stores all ways of mapping any (subset of any) rule body to the current fact (that is: the intermediate fact associated with the tree at the current stage of the construction process), while using some terms from $\terms{B}$. It follows that a rule is applicable to the current fact if and only if one of the bag patterns contains a mapping of its entire rule body. Then, the forward chaining can be performed ``on the bag-level'' by forgetting about the underlying facts and considering solely the derivation tree decorated with patterns. At each step, patterns are maintained and kept up-to-date by a propagation procedure based on a \emph{join} operation between the patterns of adjacent bags.
\item an \emph{equivalence} relation on bags: Thanks to patterns, an equivalence relation can be defined on bags, so that two bags are equivalent if and only if the ``same'' derivation subtrees can be built under them. The algorithm develops (that is: adds children nodes to) only one node per equivalence class, while the other nodes of that class are \emph{blocked} (note, however, that equivalence classes evolve during the computation, thus a blocked node can later become non-blocked, and vice-versa). This tree grows until no new rule application can be performed to non-blocked bags: the \emph{full blocked tree} is then obtained.
\item \emph{creation} and \emph{evolution rules}: The equivalence relation that we propose is however not directly computable: the ``natural'' way to compute would require to have already computed the greedy tree decomposition of the canonical model. In order to compute a full blocked tree, we make use of \emph{creation rules} and \emph{evolution rules}. These rules are meant to describe the patterns that may appear in the tree decomposition of the canonical model, as well as the relationships between patterns. For instance, creation rules intuitively state that any bag of pattern $P$ that appears in the tree decomposition of the canonical model has a child of pattern $P'$. We propose such rules, and show how to infer new rules in order to get a complete -- but finite -- description of the tree decomposition of the canonical model.
\end{itemize}

A first way to perform query answering is then to consider the query as a rule with a head reduced to a nullary prediate. In that case, it is enough to check if one pattern contains the entire body of this added rule. If one do not want to consider the query as a rule, one has to be more cautious. Indeed, the evaluation of a Boolean conjunctive query against a blocked tree cannot be performed by a simple homomorphism test. Instead, we define the notion of an APT-mapping, which can be seen as a homomorphism to an ``unfolding'' or ``development'' of this blocked tree. As the length of the developed paths that is relevant for query answering is bounded with an exponent that depends only on the rule set (more precisely, the exponent is the maximal number of variables shared by the body and the head of a rule), checking if there is an APT-mapping from a conjunctive query to a blocked tree is time polynomial in data complexity and nondeterministically time polynomial in query complexity.

In order to illustrate the numerous definitions of this section, we will employ a running example. This example has been designed with the following requirements in mind. First, it should be easy enough to understand. Second, it should illustrate every aspect of our approach, and explain why simpler approaches we could think of are not sufficient. Last, it should not be expressible by means of description logics.

\begin{example}[Running Example]
\label{ex-running}
Let us consider $\mathcal{R}^\mathrm{ex}=\{R_1^\mathrm{ex},\ldots,R_7^\mathrm{ex}\}$ defined as follows:
\begin{itemize}
\item $R^\mathrm{ex}_1 = q_1(x_1,y_1,z_1) \rightarrow s(y_1,t_1) \wedge r(z_1,t_1) \wedge q_2(t_1,u_1,v_1)$;
\item $R^\mathrm{ex}_2 = q_2(x_2,y_2,z_2) \rightarrow s(y_2,t_2) \wedge r(z_2,t_2) \wedge q_3(t_2,u_2,v_2)$;
\item $R^\mathrm{ex}_3 = q_3(t_3,u_3,v_3) \rightarrow h(t_3)$;
\item $R^\mathrm{ex}_4 = q_2(x_4,y_4,z_4) \wedge s(y_4,t_4) \wedge r(z_4,t_4) \wedge h(t_4) \rightarrow h(x_4) \wedge p_1(y_4) \wedge p_2(z_4)$;
\item $R^\mathrm{ex}_5 = q_1(x_5,y_5,z_5) \wedge s(y_5,t_5) \wedge r(z_5,t_5) \wedge h(t_5) \rightarrow p_1(y_5) \wedge p_2(z_5)$;
\item $R^\mathrm{ex}_6 = p_1(x_p) \wedge i(x_p) \rightarrow r(x_p,y_p) \wedge p_2(y_p) \wedge i(y_p)$;
\item $R^\mathrm{ex}_7 = p_2(x_q) \wedge i(x_q) \rightarrow s(x_q,y_q) \wedge p_1(y_q) \wedge i(y_q)$.
\end{itemize}

The initial fact will be: $$F^\mathrm{ex} = q_1(a,b,c) \wedge q_1(d,c,e) \wedge q_1(f,g,g) \wedge i(c) \wedge i(g).$$

\end{example}

The subset $\{R^\mathrm{ex}_1,R^\mathrm{ex}_2,R^\mathrm{ex}_3\}$ is a finite expansion set\footnote{Because, for example, their graph of rule dependency is acyclic \cite{blms:09}}. Applying these rules will create some existentially quantified variables. A first interesting phenomenon is that these existential variables allow to infer some new information about the initial terms. Last, $R^\mathrm{ex}_4$ and $R^\mathrm{ex}_5$ will generate infinitely many fresh existential variables, which will allow us to illustrate both the blocking procedure and the querying operation.
While it can be argued that these rules are slightly complicated, it will allow to illustrate why we cannot block nodes without being careful.

Let us illustrate this rule set with an example of greedy derivation of $F^\mathrm{ex}$ under $\mathcal{R}^\mathrm{ex}$.

\begin{example}
\label{ex-greedy-derivation} Let us consider the following sequence of rule applications:
\begin{itemize}
\item $R^\mathrm{ex}_1$ is applied to $F^\mathrm{ex}$ by $\pi_1 = \{x_1{\mapsto}a, y_1{\mapsto}b, z_1{\mapsto} c\}$, creating $\{s(b,t_1^1),$ $r(c,t_1^1),$ $q_2(t_1^1,u_1^1,v_1^1)\}$.
\item $R^\mathrm{ex}_2$ is applied to $\alpha(F^\mathrm{ex},R^\mathrm{ex}_1,\pi_1)$ by $\pi_2 =\{x_2{\mapsto} t_1^1, y_2{\mapsto} u_1^1,z_2{\mapsto} v_1^1\}$, creating $\{s(u_1^1,t_2^1),$ $r(v_1^1,t_2^1),$ $q_3(t_2^1,u_2^1,v_2^1)\}$
\item $R^\mathrm{ex}_3$ is applied on the resulting fact by $\pi_3 = \{x_3{\mapsto}t_1^2, y_3{\mapsto}u_1^2, z_3{\mapsto}v_1^2\}$, creating a single new atom $h(t_2^1)$.
\end{itemize}

This derivation is greedy, and its derivation tree is represented in Figure \ref{fig-DT}.
\end{example}

\sloppypar{More generally, let us take a look at $k$-saturations of $F^\mathrm{ex}$ with respect to $\mathcal R^\mathrm{ex}$. On $F^\mathrm{ex}$, only $R^\mathrm{ex}_1$ is applicable by three homomorphisms, creating three sets of three new atoms: $\{s(b,t_1^1),r(c,t_1^1),q_2(t_1^1,u_1^1,v_1^1)\}, \{s(c,t_1^2),r(e,t_1^2),q_2(t_1^2,u_1^2,v_1^2)\}$ and $\{s(g,t_1^3),r(g,t_1^3),q_2(t_1^3,u_1^3,v_1^3)\}$. $\alpha_1(F^\mathrm{ex},\mathcal R^\mathrm{ex})$ is equal to the union of $F^\mathrm{ex}$ and these three sets of atoms. On $\alpha_1(F^\mathrm{ex},\mathcal R^\mathrm{ex})$, three new rule applications are possible, each of them mapping the body of $R^\mathrm{ex}_2$ to one of the atoms with predicate $q_2$. Again, three new sets of atoms are introduced, which are $\{s(u_1^1,t_2^1),r(v_1^1,t_2^1),q_3(t_2^1,u_2^1,v_2^1)\}, \{s(u_1^2,t_2^2),r(v_1^2,t_2^2),q_3(t_2^2,u_2^2,v_2^2)\}$ and $\{s(u_1^3,t_2^3),r(v_1^3,t_2^3),q_3(t_2^3,u_2^3,v_2^3)\}$. This yields $\alpha_2(F^\mathrm{ex},\mathcal R^\mathrm{ex})$. On this fact, three new rule applications of $R^\mathrm{ex}_3$ are possible, which introduce $h(t_2^1),h(t_2^2),h(t_2^3)$. The introduction of these atoms triggers new applications of $R^\mathrm{ex}_4$, creating $h(t_1^1),h(t_1^2),h(t_1^3), p_1(u_1^1),p_1(u_1^2),p_1(u_1^3), p_2(v_1^1),p_2(v_1^2),p_2(v_1^3)$. $R^\mathrm{ex}_5$ is now triggered, creating $p_1(b),p_2(c),p_1(c),p_2(e),p_1(f),p_2(f)$. The union of all atoms considered so far is equal to $\alpha_5(F^\mathrm{ex},\mathcal R^\mathrm{ex})$.
 $R^\mathrm{ex}_{6}$ and $R^\mathrm{ex}_{7}$ are now applicable, both mapping their frontier to $c$ and $g$. They will create infinite branches.

\begin{figure}
\begin{center}
\begin{tikzpicture}[line cap=round,line join=round,>=triangle 45,x=1.0cm,y=1.0cm]
\clip(-2.57,-5.8) rectangle (4.75,5.45); \draw (-1.06,4.92) node[anchor=north west] {\parbox{1.75 cm}{$q_1(a,b,c), q_1(d,c,e)\\ \\ q_1(f,g,g), i(c), i(g)$}}; \draw (-0.77,2.32) node[anchor=north west] {\parbox{1.83 cm}{$s(b,t_1^1),r(c,t_1^1)\\ \\ q_2(t_1^1,u_1^1,v_1^1)$}}; \draw (-0.81,-0.49) node[anchor=north west] {\parbox{2.07 cm}{$s(u_1^1,t_2^1),r(v_1^1,t_2^1)\\ \\ q_3(t_2^1,u_2^1,v_2^1)$}}; \draw [rotate around={-0.23:(0.83,4.05)}] (0.83,4.05) ellipse (2.89cm and 1.05cm); \draw [rotate around={-0.23:(0.72,1.41)}] (0.72,1.41) ellipse (2.89cm and 1.05cm); \draw [rotate around={-0.23:(0.85,-1.34)}] (0.85,-1.34) ellipse (2.89cm and 1.05cm); \draw (0.7,3)-- (0.68,2.46); \draw (0.64,0.36)-- (0.64,-0.29); \draw (4.05,4.52) node[anchor=north west] {$B_0$}; \draw (4.07,1.77) node[anchor=north west] {$B_1$}; \draw (4.05,-0.8) node[anchor=north west] {$B_2$}; \draw (-0.07,-3.88) node[anchor=north west] {$h(t_2^1)$}; \draw [rotate around={-0.23:(0.89,-4.33)}] (0.89,-4.33) ellipse (2.89cm and 1.05cm); \draw (0.57,-2.39)-- (0.57,-3.29); \draw (4.07,-3.86) node[anchor=north west] {$B_3$};
\end{tikzpicture}
\end{center}
\caption{The derivation tree associated with Example \ref{ex-greedy-derivation}} \label{fig-DT}
\end{figure}
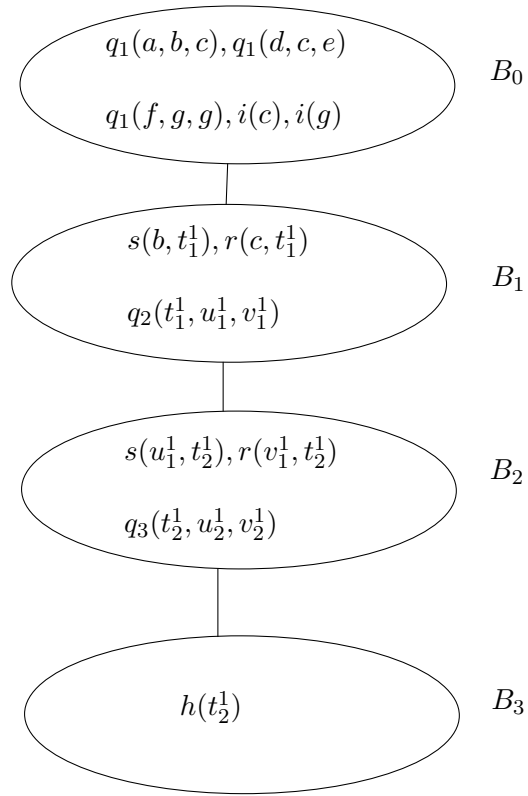

%% file: fc2pfc-kr.tex

This section focuses on \emph{bag patterns}. 
For all following considerations, we assume an arbitrary but fixed rule set $\mathcal{R}$ which is \RgRbRtRsR.
We first show that forward chaining can be performed by considering solely the derivation tree endowed with bag patterns. Then we define \emph{joins} on patterns in order to update them incrementally after each rule application. We last explain why patterns are interesting: they allow to formally capture some notion of ``regularity'' in a derivation tree, which will be exploited in the next section to finitely represent potentially infinite derivation trees.
\begin{definition}[Pattern, Patterned Derivation Tree]
\index{pattern} A \emph{pattern} of a bag $B$
is a set of pairs $(G, \pi)$, where $G$ is a conjunction of atoms and $\pi$ is a partial mapping from $\terms{G}$ to $\terms{B}$. $G$ and $\pi$ are possibly empty.

For any $\mathcal{R}$-derivation $S$ with derivation tree $\mathit{DT}(S)$, we obtain a \emph{patterned derivation tree}, noted $(\mathit{DT}(S),P)$, where $P$ is a function assigning a pattern $P(B)$ to each bag $B$ of $\mathit{DT}(S)$.
\end{definition}


The patterns that we consider are subsets of the rule bodies in $\mathcal{R}$.

\begin{definition}[Pattern Soundness and Completeness] Let $F_k$ be a fact obtained via a derivation $S$ and let $B$ be a bag in $(\mathit{DT}(S),P)$. $P(B)$ is said to be \emph{sound} w.r.t.\ $F_k$ if for all $(G,\pi) \in P(B)$, $\pi$ is extendable to a homomorphism from $G$ to $F_k$.  $P(B)$ is said to be \emph{complete} w.r.t. $F_k$ (and $\mathcal R$), if for any $R \in \mathcal R$, any $sb_R \subseteq \fun{body}(R)$ and any homomorphism $\pi$ from $sb_R$ to $F_k$, $P(B)$ contains $(sb_R, \pi')$, where $\pi'$ is the restriction of $\pi$ to the inverse image of the terms of $B$, i.e., $\pi'=\pi_{\mid \pi^{-1}(\terms{B})}$. Finally, $(\mathit{DT}(S),P)$ is said to be \emph{sound} and \emph{complete} w.r.t.\ $F_k$ if for all its bags $B$, $P(B)$ is sound and complete w.r.t.\ $F_k$.
\end{definition}


Provided that $(\mathit{DT}(S),P)$ is sound and complete w.r.t.\ $F_k$, a rule $R$ is applicable to $F_k$ iff there is a bag in $(\mathit{DT}(S),P)$ whose pattern contains a pair $(\fun{body}(R), -)$; then, the bag created by a rule application $(R,\pi)$ to $F_k$ has parent $B_j$  in $\mathit{DT}(S)$ iff $B_j$ is the bag in $(\mathit{DT}(S),P)$ with the smallest $j$ such that $P(B_j)$ contains $(\fun{body}(R), \pi')$ for some $\pi'$ which coincides with $\pi$ on $\fr{R}$, i.e., $\pi_{\mid\fr{R}}=\pi'_{\mid\fr{R}}$. Patterns are managed as follows: (1) The pattern of $B_0$ is the maximal sound and complete pattern with respect to $F$; (2) after each addition of a bag $B_i$, the patterns of all bags are updated to ensure their soundness and completeness with respect to $F_i$. It follows that we can define a \emph{patterned} derivation, where rule applicability is checked on patterns, and the associated sound and complete patterned derivation tree, which can be shown to be isomorphic to the derivation tree associated with the (regular) derivation.

Remember that our final rationale is to avoid computations on the ``fact level''. We will instead incrementally maintain sound and complete patterns by a propagation mechanism on patterns. This is why we need to consider patterns with subsets of rule bodies and not just full rule bodies. We recall that the rules have pairwise disjoint sets of variables.

\begin{definition}[Elementary Join]
Let $B_1$ and $B_2$ be two bags, $e_1= (sb^1_R, \pi_1) \in P(B_1)$ and $e_2=(sb^2_R, \pi_2) \in P(B_2)$  where $sb^1_R$ and $sb^2_R$ are subsets of $\fun{body}(R)$ for some rule $R$.
 Let $V = \vars{sb^1_R} \cap \vars{sb^2_R}$. The \emph{(elementary) join} of $e_1$ with $e_2$, noted $J(e_1, e_2)$, is defined if for all  $x \in V$, $\pi_1(x)$ and $\pi_2(x)$ are both defined and $\pi_1(x) = \pi_2(x)$. Then $J(e_1,e_2) = (sb_R, \pi)$, where $sb_R = sb^1_R \cup sb^2_R$ and $\pi = \pi_1 \cup \pi'_2$, where $\pi'_2$ is the restriction of $\pi_2$ to the inverse image of $\terms{B_1}$ (i.e., the domain of $\pi'_2$ is the set of terms with image in  $\terms{B_1}$).
\end{definition}

Note that $V$ may be empty. The elementary join is not a symmetrical operation since the range of the obtained mapping is included in $\terms{B_1}$.

\begin{example}
Let us consider the bags $B_1$ and $B_2$ in Figure \ref{fig-DT}. Let $e_1 = (\{q_2(x_4,y_4,z_4)\}, \pi=\{x_4{\mapsto}t_1^1, y_4{\mapsto}u_1^1, z_4{\mapsto}v_1^1\})$ be in the pattern of $B_1$, and $e_2 = (\{s(y_4,t_4),r(z_4,t_4),h(t_4)\},\pi'=\{y_4{\mapsto}u_1^1, z_4{\mapsto}v_1^1, t_4{\mapsto}t_2^1\})$ be in the pattern of $B_2$. The elementary join of $e_1$ with $e_2$ is $(\{q_2(x_4,y_4,z_4),s(y_4,t_4),r(z_4,t_4),h(t_4)\},\pi)$.
\end{example}
\begin{definition}[Join]\label{def-join}
\index{join} Let $B_1$ and $B_2$ be two bags with respective patterns $P(B_1)=P_1$ and $P(B_2)=P_2$. The join of $P_1$ with $P_2$, denoted $J(P_1, P_2)$, is the set of all defined $J(e_1,e_2)$, where $e_1= (sb^1_R, \pi_1) \in P_1$, $e_2=(sb^2_R, \pi_2) \in P_2$.
\end{definition}

Note that $P_1 \subseteq J(P_1, P_2)$ since each pair from $P_1$ can be obtained by an elementary join with $(\emptyset,\emptyset)$. Similarly, $J(P_1, P_2)$ contains all pairs $(G,\pi)$ obtained from $(G,\pi_2) \in P_2$ by restricting $\pi_2$ to the inverse image of $\terms{B_1}$. Note that join preserves soundness, as stated in the following property.

%
\begin{property} If $P_1$ and $P_2$ are sound w.r.t. $F_{i}$ then $J(P_1,P_2)$ is sound w.r.t. $F_i$.
\end{property}

\begin{proof} Follows from the definitions: for all $(G,\pi) \in J(P_1,P_2)$, either $(G,\pi) \in P_1$, or is obtained by restricting an element of $P_2$, or is equal to $J(e_1,e_2)$ for some $e_1= (sb^1_R, \pi_1) \in P_1$ and $e_2=(sb^2_R, \pi_2) \in P_2$. In the latter case, let us consider two homomorphisms, $h_1$ and $h_2$ with co-domain $F_i$, which respectively extend $\pi_1$ and $\pi_2$. The union of $h_1$ and $h_2$ is a mapping from $terms(G)$ to $F_i$ (remember that $h_1$ and $h_2$ are equal on the intersection of their domains). Moreover, it is a homomorphism, because every atom in $G$ is mapped to an atom in $F_i$ by $h_1$ or by $h_2$.
\end{proof}

We consider the step from $F_{i-1}$ to $F_{i}$ in a (patterned) derivation sequence: let $B_c$ be the bag created in this step and let $B_p$ be its parent in $(\mathit{DT}(S),P)$.

\begin{definition}[Initial Pattern]
\index{pattern!initial pattern} The \emph{initial pattern} of a bag $B_c$, denoted by $P_\mathrm{init}(B_c)$, is the set of pairs $(G,\pi)$ such that $G$ is a subset of some rule body of $\mathcal{R}$ and $\pi$ is a \emph{homomorphism} from $G$ to $\atoms{B_c}$.
\end{definition}

\begin{example}[Initial Pattern]
\label{ex-initial-pattern} Let us consider the initial pattern of $B_2$ in Figure \ref{fig-DT}. The atoms of $B_2$ are:

$$\{s(u_1^1,t_2^1),r(v_1^1,t_2^1),q_3(t_2^1,u_2^1,v_2^1)\}.$$

For rules $R^\mathrm{ex}_1,R^\mathrm{ex}_2,R^\mathrm{ex}_{6}$ and $R^\mathrm{ex}_{7}$, no subset of a rule body maps to the atoms of $B_2$. Thus, they do not contribute to the initial pattern of $B_2$. There is one homomorphism from the body of $R^\mathrm{ex}_3$ to atoms of $B_2$, and thus its initial pattern contains:

$$(\{q_3(t_3,u_3,v_3)\}, \{t_3{\mapsto}t_2^1, u_3{\mapsto}u_2^1, v_3{\mapsto}v_2^1).$$

As for subsets of the body of $R^\mathrm{ex}_4$, there are three elements added to the initial pattern of $B_2$:
\begin{itemize}
\item $(\{s(y_4,t_4)\},\{y_4{\mapsto}u_1^1,t_4{\mapsto}t_2^1\}),$
\item $(\{r(z_4,t_4)\},\{t_4{\mapsto}t_2^1,z_4{\mapsto}v_1^1\}),$
\item $(\{s(y_4,t_4),r(z_4,t_4)\},\{y_4{\mapsto}u_1^1,t_4{\mapsto}t_2^1,z_4{\mapsto}v_1^1\}).$
\end{itemize}

Similar elements are added by taking subsets of the body of $R^\mathrm{ex}_5$.
\end{example}

\begin{property}[Soundness of Initial Pattern of $B_c$ w.r.t. $F_i$]
The initial pattern of $B_c$ is \emph{sound} with respect to $F_i$.
\end{property}

\begin{proof}
For any $(G, \pi)\in P_\mathrm{init}(B_c)$, $\pi$ is a homomorphism from $G$ to $\fun{atoms}(B_c) \subseteq F_{i}$.
\end{proof}

Obviously, if a pattern is sound w.r.t. $F_{i-1}$ then it is sound w.r.t. $F_i$. The following property focus on completeness.
%

\begin{property}[Completeness of $J(P(B_c), P(B_p))$ w.r.t. $F_i$]\label{prop-complete-Bc}
Assume that $P(B_p)$ is complete w.r.t. $F_{i-1}$ and $\mathcal R$. Then $J(P_{init}(B_c), P(B_p))$ is complete w.r.t. $F_{i}$.
\end{property}

\begin{proof}
Let $\pi$ be a homomorphism from $sb_R \subseteq \fun{body}(R)$ to $F_i$, for some rule $R$. We show that $J(P_\mathrm{init}(B_c), P(B_p))$ contains $(sb_R, \pi')$, where $\pi'$ is the restriction of $\pi$ to the inverse image of $\terms{B_c}$. Let us partition $sb_R$ into $b_{i-1}$, the subset of atoms mapped by $\pi$ to $F_{i-1}$, and $b_i$ the other atoms from $sb_R$, which are necessarily mapped by $\pi$ to $F_i \setminus F_{i-1}$, i.e., $\atoms{B_c}$. If $b_i$ is not empty, by definition of the initial pattern, $P_\mathrm{init}(B_c)$ contains $(b_i, \pi_c)$, where $\pi_c$ is the restriction of $\pi$ to $\terms{b_i}$.  If $b_{i-1}$ is not empty, by hypothesis (completeness of $P(B_p)$ w.r.t. $F_{i-1}$), $P_p$ contains $(b_{i-1}, \pi_p)$, where $\pi_p$ is the restriction of $\pi_{|b_{i-1}}$ to the inverse image of $\terms{B_p}$. If  $b_{i-1}$ or $b_{i}$ is empty, $(sb_R, \pi')$ belongs to $J(P_\mathrm{init}(B_c), P(B_p))$. 
Otherwise, consider $J((b_i, \pi_c),(b_{i-1}, \pi_p))$: it is equal to $(sb_R, \pi')$. 
\end{proof}


\begin{property}[Completeness of Join-Based Propagation]
\label{prop-pattern-update} Assume that $(\mathit{DT}(S),P)$ is complete w.r.t. $F_{i-1}$, and $P(B_c)$ is computed by $J(P_\mathrm{init}(B_c), P(B_p))$. Let $d(B)$ denote the distance of a bag $B$ to $B_c$ in $(\mathit{DT}(S),P)$. Updating a bag $B$ consists in performing $J(P(B), P(B'))$, where $B'$ is the neighbor of $B$ s.t. $d(B') < d(B)$. Let $(\mathit{DT}(S),P')$ be obtained from $(\mathit{DT}(S),P)$ by updating all patterns by increasing value of $d$ of the corresponding bags.
Then $(\mathit{DT}(S),P')$ is complete w.r.t. $F_i$.
\end{property}


\begin{proof}
From Property~\ref{prop-complete-Bc}, we know that $P'(B_c)$ is complete w.r.t. $F_i$.
 It remains to prove the following property: let $P'(B)$ be obtained by computing $J(P(B), P'(B'))$; if $P'(B')$ is complete w.r.t. $F_i$, then $J(P(B), P'(B'))$ is complete w.r.t. $F_i$.
We partition $sb_R$ in the same way as in the proof of Property~\ref{prop-complete-Bc}. If one of the subsets is empty, we are done. Otherwise, the partition allows to select an element $e_1$ from $P(B)$ and an element $e_2$ from $P'(B')$, and $J(e_1,e_2)$ is the element we want to find. The crucial point is that if $\pi$ maps an atom $a$ of $sb_R$ to an atom $b$ of $F_i \setminus F_{i-1}$, and $b$ shares a term $e$ with $B$, then $e \in \terms{B_c}$, hence, thanks to the running intersection property of a decomposition tree, $e \in \terms{B'}$, thus $(e,\pi(e))$ will be propagated to $P'(B)$.
\end{proof}

It follows that the following steps performed at each bag creation (where $B_c$ is introduced as a child of $B_p$) allow to maintain the soundness and completeness of the patterned derivation tree throughout the derivation:
\begin{enumerate}
\item initialize: compute $P_\mathrm{init}(B_c)$ for the newly created pattern $B_c$;
\item update: $P'(B_c) = J(P_\mathrm{init}(B_c), P(B_p))$;
\item propagate: first, propagate from $P(B_c)$ to $P(B_p)$, i.e., $P'(B_p)=J(P(B_p), P'(B_c))$; then, for each bag $B$ updated from a bag $B'$, update its children $B_i$ (for $B_i \neq B')$ by $P'(B_i)=J(P(B_i), P'(B))$ and its parent $B_j$ by $P'(B_j)=J(P(B_j), P'(B))$. Iterate this step until every pattern is updated (i.e., $P'(B)$ is determined for every bag $B$ of the current derivation tree).
\end{enumerate}

%% file: pfc2bpfc-kr.tex

We now show how bag patterns allow us to identify a certain kind of regularity in a derivation tree. We first need some technical, but nonetheless natural definitions. We start with the notion of a \emph{fusion} of the frontier induced by a rule application: given a rule application, it summarizes which frontier terms are mapped to the same term, and if they are mapped to a term of $T_0$ (that is, an initial term or a constant).

\begin{definition}[Fusion of the Frontier Induced by $\pi$]
\index{fusion of the frontier} Let $R$ be a rule and $V$ be a set of variables with $V \cap T_0 = \emptyset$. Let $\pi$ be a substitution of $\fun{fr}(R)$ by $T_0 \cup V$. The \emph{fusion} of $\fun{fr}(R)$ induced by $\pi$, denoted by $\sigma_\pi$, is the substitution of $\fun{fr}(R)$ by $\fun{fr}(R) \cup T_0$ such that for every variable $x \in \fun{fr}(R)$, if $\pi(x) \in V$ then $\sigma_\pi(x)$ is the smallest\footnote{We assume variables to be totally ordered (for instance by lexicographic order).} variable $y$ of $\fun{fr}(R)$ such that $\pi(x) = \pi(y)$; otherwise $\sigma_\pi(x) = \pi(x) \in T_0$.
\end{definition}

\begin{example}
Let us consider $R^\mathrm{ex}_2 = q_2(x_2,y_2,z_2) \rightarrow s(y_2,t_2) \wedge r(z_2,t_2) \wedge q_3(t_2,u_2,v_2)$. Let $\pi_1=\{y_2{\mapsto}y_0,z_2{\mapsto}y_0\}$. The substitution of the frontier of $R_2$ induced by $\pi_1$ is defined by $\sigma_{\pi_1}=\{y_2{\mapsto}y_2,z_2{\mapsto}y_2\}$. Let $b$ be a constant, and $\pi_2$ be a substitution of the frontier of $R_1$ defined by $\pi_2=\{y_2{\mapsto}b,z_2{\mapsto}b\}$. The fusion of the frontier induced by $\pi_2$ is defined by $\sigma_{\pi_2}=\{y_2{\mapsto}b,z_2{\mapsto}b\}$. Last, if $\pi_3$ maps $y_2$ and $z_2$ to two different existentially quantified variables, then $\sigma_{\pi_3}$ is the identity on the frontier of $R_2$.
\end{example}

This notion of fusion is the main tool to define \emph{structural equivalence}, which is an equivalence relation on the bags of a derivation tree.

\begin{definition}[Structural Equivalence]
\index{equivalence!structural} Let $B$ and $B'$ be two bags
created by applications $(R, \pi_i)$ and $(R, \pi_j)$, respectively, of the same rule $R$. $B$ and $B'$ are \emph{structurally equivalent}, written $B \simeq B'$ if the fusions of $\fun{fr}(R)$ induced by the restrictions of $\pi_i$ and $\pi_j$ to $\fun{fr}(R)$ are equal.
\end{definition}

We will see later that structural equivalence is not sufficient to formalize regularity in a derivation tree. However, there is already a strong similarity between structurally equivalent bags: the purpose of Definition \ref{def-structural-equivalence} is to formalize it.

\begin{definition}[Natural Bijection]
\index{natural bijection} \label{def-structural-equivalence} Let $B$ and $B'$ be two structurally equivalent bags created by applications $(R,\pi_i)$ and $(R,\pi_j)$. The \emph{natural bijection} from $\fun{terms}(B)$ to $\fun{terms}(B')$ (in short from $B$ to $B'$), denoted  $\psi_{B\rightarrow B'}$,  is defined as follows:
\begin{itemize}
\item if $x \in T_0$, let $\psi_{B\rightarrow B'}(x) = x$
\item otherwise, let $\fun{orig}(x) = \{u\in \fun{vars}(\head{R}) | \pi_i^{\mathrm{safe}}(u) = x\}$. Since $B$ and $B'$ are structurally equivalent, $\forall u,u' \in \fun{orig}(x), \pi_j^{\mathrm{safe}}(u) = \pi_j^{\mathrm{safe}}(u')$. We define $\psi_{B \rightarrow B'}(x) =\pi_j^{\mathrm{safe}}(u)$.
\end{itemize}
\end{definition}

The natural bijection is thus an isomorphism between two bags. This natural bijection between structurally equivalent bags gives us a way to partially order patterns, by ensuring that the ranges of partial applications are on the same set of terms.

\begin{definition}[Pattern Inclusion, Pattern Equivalence]
\index{equivalence!pattern} \label{def-pattern-inclusion} Let $B$ and $B'$ be two bags,
with respective patterns $P(B)$ and $P(B')$. We say that $P(B')$ \emph{includes} $P(B)$, denoted by $P(B) \sqsubseteq P(B')$, if :
\begin{itemize}
\item $B$ and $B'$ are structurally equivalent, i.e., $B\simeq B'$,
\item $P(B')$ contains all elements from $P(B)$, up to a variable renaming given by the natural bijection: $(G,\pi) \in P(B) \Rightarrow (G,\psi_{B \rightarrow B'}\circ \pi) \in P(B')$.
\end{itemize}
We say that $P(B)$ and $P(B')$ are equivalent, denoted $P(B)\sim P(B')$, if $P(B') \sqsubseteq P(B)$ and $P(B) \sqsubseteq P(B')$. By extension, two bags are said to be equivalent if their patterns are equivalent.
\end{definition}

Property \ref{prop-pattern-equivalence} helps to understand why Definition \ref{def-pattern-inclusion} provides us with a good notion of pattern equivalence, by linking the equivalence of patterns to the applicability of rules on bags having these patterns. Let us note that this property does not hold if we put structural equivalence in place of pattern equivalence.

\begin{property}
\label{prop-pattern-equivalence} Let $S$ be a derivation, and $B$ and $B'$ two bags of $(\fun{DT}(S),P)$ such that $P(B) \sim P(B')$. If a rule $R$ is applicable to $B$ by $\pi$, then $R$ is applicable to $B'$ by  $\psi_{B\rightarrow B'}\circ\pi$.
\end{property}

\begin{proof}
Since $R$ is applicable to $B$ by $\pi$, $(\fun{body}(R), \pi_{\mid \fun{fr}(R)})$ belongs to $P(B)$. By definition of the equivalence of patterns, $(\fun{body}(R), \psi_{B \rightarrow B'} \circ \pi_{\mid \fun{fr}(R)})$ belongs to $P(B')$, which implies that $R$ is applicable to $B'$.
\end{proof}

We now present how this equivalence relation will be used to finitely represent the (potentially infinite) set of derived facts. Intuitively, a blocked tree $\mathfrak{T}_b$ is a subtree (with the same root) of a patterned derivation tree $(\fun{DT}(S),P)$ of a sufficiently large derivation sequence $S$. Additionally every bag in $\mathfrak{T}_b$ is marked by either ``blocked'' or ``non-blocked''. Assuming that we know which length of derivation is enough, $\mathfrak{T}_b$ is constructed such that it has the following properties:
\begin{itemize}
\item for each equivalence class appearing in $(\fun{DT}(S),P)$, there is exactly one non-blocked node of $\mathfrak{T}_b$ of that class;
\item if a bag $B$ is blocked in $\mathfrak{T}_b$, it is a leaf, i.e., it has no child in $\mathfrak{T}_b$ (although it may have children in $(\fun{DT}(S),P)$);
\item if a bag is non-blocked in $\mathfrak{T}_b$, all children of $B$ in $(\fun{DT}(S),P)$ are present in $\mathfrak{T}_b$.
\end{itemize}

\begin{definition}[Blocked Tree]
\index{blocked tree} A blocked tree is a structure $(\mathfrak{T}_b,\sim)$, where $\mathfrak{T}_b$ is an initial segment of a patterned derivation tree and $\sim$ is the equivalence relation on the bags of $\mathfrak{T}_b$ such that for each $\sim$-class, all but one bag are said to be \emph{blocked}; this non-blocked bag is called the \emph{representative} of its class and is the only one that may have children.
\end{definition}


A blocked tree $\mathfrak{T}_b$ can be associated with a possibly infinite set of  decomposition trees obtained by iteratively copying its bags.
We first define the bag copy operation:

\begin{definition}[Bag Copy]
\index{bag copy} Let $B_1$ and $B_2$ be structurally equivalent bags with natural bijection $\psi_{B_1\rightarrow B_2}$. Let $B'_1$ be a child of $B_1$. \emph{Copying $B'_1$ under $B_2$} (according to $\psi_{B_1\rightarrow B_2}$) is performed by adding a child $B'_2$ to $B_2$, such that $\fun{terms}(B'_2)=\{\psi_{B'_1\rightarrow B'_2}(t) \mid t \in \fun{terms}(B'_1)\}$ and $\fun{atoms}(B'_2) = \{\psi_{B'_1\rightarrow B'_2}(a) \mid a\in \fun{atoms}(B'_1)\}$, where $\psi_{B'_1\rightarrow B'_2}$ is defined as follows: for all $x \in \fun{terms}(B'_1)$, if $x \in  \fun{terms}(B_1)$ then $\psi_{B'_1\rightarrow B'_2}(x) =  \psi_{B_1\rightarrow B_2}(x)$, otherwise $\psi_{B'_1\rightarrow B'_2}(x)$ is a fresh variable. \label{def-bag-copy}
\end{definition}


Assume that, in the previous definition, the bag  $B'_1$ has been created by $(R, \pi)$. Then $B'_2$ can be seen as obtained by the fusion of $\fr{R}$ induced by the potential application of $R$ to $B_2$ with the homomorphism $\psi_{B_1\rightarrow B_2} \circ \pi$. Since the fusions of $\fr{R}$ induced by $\pi$ and $\psi_{B_1\rightarrow B_2} \circ \pi$ are equal, $B'_1$ and $B'_2$ are structurally equivalent, which justifies the use of $\psi_{B'_1\rightarrow B'_2}$ for the bijection.

Starting from a blocked tree $\mathfrak{T}_b$ and using iteratively the copy operation when applicable, one can build a possibly infinite set of trees, that we denote by $G(\mathfrak{T}_b)$. This set contains pairs, whose first element is a tree, and the second element is a mapping from the bags of this tree to the bags of $\mathfrak{T}_b$, which encodes which bags of $\mathfrak{T}_b$ have been copied to create the bags of the generated tree.


\begin{definition}[Trees Generated by a Blocked Tree]  Given a blocked tree $\mathfrak{T}_b$, let the \emph{set $G(\mathfrak{T}_b)$ of trees generated by $\mathfrak{T}_b$} be inductively defined as follows:
\begin{itemize}
\item Let $B_0$ be the root of $\mathfrak{T}_b$; the pair $(\{B_0\},\{B_0{\mapsto}B_0\})$ belongs to $G(\mathfrak{T}_b)$.
\item Given a pair $(\mathfrak{T},f) \in G(\mathfrak{T}_b)$, let $B$ be a bag in $\mathfrak{T}$, and $B' = f(B)$;
let $B'_r$ be the representative of the $\sim$-class containing $B'$ (i.e., $B'_r \neq B'$ if $B'$ is blocked) and let $B'_c$ be a child of $B'_r$. If $B$ has no child mapped to $B'_c$ by $f$, let $\mathfrak{T}_{new}$ be obtained from $\mathfrak{T}$ by copying $B'_c$ under $B$ (according to $\psi_{B'_r\rightarrow B}$), which yields a new bag $B_c$. Then $(\mathfrak{T}_{new}, f \cup (B_c{\mapsto}B'_c))$ belongs to $G(\mathfrak{T}_b)$.
\end{itemize}
For each pair $(\mathfrak{T},f) \in G(\mathfrak{T}_b)$, $\mathfrak{T}$ is said to be \emph{generated} by $\mathfrak{T}_b$ via $f$. The tree $\mathfrak{T}$ is said to be generated by $\mathfrak{T}_b$ if there exists an $f$ such that $\mathfrak{T}$ is generated by $\mathfrak{T}_b$ via $f$.
\end{definition}

Note that a patterned decomposition tree thus generated is not necessarily a derivation tree, but it is an initial segment of a derivation tree. Among blocked trees, so-called \emph{full blocked trees} are of particular interest.


\begin{definition}[Full Blocked Tree]
A \emph{full blocked tree} $\mathfrak{T}^*$ (of $F$ and $\mathcal R$) is a blocked tree satisfying the two following properties:
\begin{itemize}
\item (Soundness) If $\mathfrak{T}'$ is generated by $\mathfrak{T}^*$, then there is some $\mathfrak{T}''$ generated by $\mathfrak{T}^*$ and an $\mathcal R$-derivation $S$ from $F$ such that
    $atoms(\mathfrak{T}'') = \fun{atoms}(\mathit{DT}(S))$ (up to fresh variable renaming) and $\mathfrak{T}'$ is an initial segment of $\mathfrak{T}''$.
\item (Completeness) For all $\mathcal R$-derivations from $F$, $\mathit{DT}(S)$ is generated by $\mathfrak{T}^*$.
\end{itemize}
\label{def-full-blocked-tree}
\end{definition}

The procedure outlined above (considering a particular tree prefix of a sufficiently large derivation tree) is however not constructive. We show how to circumvent this problem in the next section.

%% file: blocked-kr-revised.tex

We now aim at computing a full blocked tree. To this end, we fix a representative for each structural equivalence class, as well as for each (pattern-based) equivalence class. This is the purpose of \emph{abstract bags} and \emph{abstract patterns}. We also need to describe on an abstract level how bags of a derivation tree are related to each other: \emph{links} are introduced to that aim. Having defined these basic components, we will focus on getting structural knowledge about the derivation trees that can be created starting from a fact and a set of rules: creation rules and evolution rules will be defined. In the last step, we use these rules to compute a full blocked tree.

We start by defining abstract bags. Each abstract bag can be seen as a canonical representative of a class of the structural equivalence relation. In order to have a uniform presentation, we consider the initial fact as a rule with empty body. 

\begin{definition}[Abstract Bag, Frontier Terms, Generated Variables]
\index{abstract! bag} \label{def-abstract-bag} Let $R$ be a rule from $\mathcal{R}$ and $\sigma$ a fusion of $\fun{fr}(R)$. The \emph{abstract bag} associated with $R$ and $\sigma$ (notation: $\mathbb{B}(R,\sigma)$) is defined by $\fun{terms}(\mathbb{B}(R,\sigma))  = \sigma(\fun{terms}(head(R))) \cup T_0$ and $\fun{atoms}(\mathbb{B}(R,\sigma)) = \sigma(head(R))$. The \emph{frontier terms} of $\mathbb{B}(R,\sigma)$ are the elements of $\sigma(\fr{R})$. Variables from $\fun{terms}(\mathbb{B}(R,\sigma))$ that are not frontier terms are called \emph{generated variables}.
\end{definition}

The notion of the natural bijection between structurally equivalent bags is extended to abstract bags in the straightforward way (note that there is exactly one abstract bag per structural equivalence class).

\begin{example}[Abstract Bag]
Let us consider $R^\mathrm{ex}_2 = q_2(x_2,y_2,z_2) \rightarrow s(y_2,t_2) \wedge r(z_2,t_2) \wedge q_3(t_2,u_2,v_2)$, and three fusions of its frontier, namely: $\sigma_{\pi_1}=\{y_2{\mapsto}y_2,z_2{\mapsto}y_2\}$, $\sigma_{\pi_2}=\{y_2{\mapsto}b,z_2{\mapsto}b\}$  and $\sigma_{\pi_3}=\{y_2{\mapsto}y_2,z_2{\mapsto}z_2\}$. The abstract bag $\mathbb{B}(R^\mathrm{ex}_2,\sigma_{\pi_1})$ associated with $R^\mathrm{ex}_2$ and $\sigma_{\pi_1}$ has as terms $\{y_2,t_2,u_2,v_2\}$ and as atoms $\{s(y_2,t_2), r(y_2,t_2),q_3(t_2,u_2,v_2)\}$. The terms of the abstract bag $\mathbb{B}(R^\mathrm{ex}_2,\sigma_{\pi_2})$ are $\{b,t_2,u_2,v_2\}$; its atoms are $\{s(b,t_2),r(b,t_2),q_3(t_2,u_2,v_2)\}$. For $\mathbb{B}(R^\mathrm{ex}_2,\sigma_{\pi_3})$, its terms are $\{y_2,t_2,u_2,v_2,z_2\}$ and its atoms are $\{s(y_2,t_2),r(z_2,t_2),q_3(t_2,u_2,v_2)\}$.
\end{example}

Since abstract bags provide us with a canonical representative for each structural equivalence class, we can now define a canonical representative for each class of equivalent patterns: abstract patterns. To distinguish the abstract bags and patterns from their concrete counterparts, we will denote them by $\mathbb{B}$ and $\mathbb{P}$ (possibly with subscripts) instead of $B$ and $P$.

\begin{definition}[Abstract Pattern, Support]
\index{abstract!pattern} \label{def-abstract-pattern} Let $\mathcal R$ be a set of rules, $R$ be a rule and $\sigma$ be a fusion of $\fun{fr}(R)$. An \emph{abstract pattern} $\mathbb{P}$ with \emph{support} $\mathbb{B} = \mathbb{B}(R,\sigma)$ is a set of pairs $(G,\pi)$ where $G$ is a subset of a rule body (of some rule of $\mathcal R$) and $\pi$ is a partial mapping from $\fun{terms}(G)$ to $\fun{terms}(\mathbb{B})$. $G$ and $\pi$ are possibly empty.
 \end{definition}

\begin{example}[Abstract Pattern]
Let us consider again the initial pattern described in Example~\ref{ex-initial-pattern}. This pattern contains the following elements:
\begin{itemize}
\item $(\{q_3(t_3,u_3,v_3)\}, \{t_3{\mapsto}t_2^1,u_3{\mapsto}u_2^1,v_3{\mapsto}v_2^1\}),$
\item $(\{s(y_4,t_4)\},\{t_4{\mapsto}t_2^1,y_5{\mapsto}u_1^1\}),$
\item $(\{r(z_4,t_4)\},\{t_4{\mapsto}t_2^1,z_4{\mapsto}v_1^1\}),$
\item $(\{s(y_4,t_4),r(z_4,t_4)\},\{t_4{\mapsto}t_2^1,y_4{\mapsto}u_1^1,t_4{\mapsto}v_1^1\}),$
\item $(\{s(y_5,t_5)\},\{t_5{\mapsto}t_2^1,y_5{\mapsto}u_1^1\}),$
\item $(\{r(z_5,t_5)\},\{t_5{\mapsto}t_2^1,z_5{\mapsto}v_1^1\}),$
\item $(\{s(y_5,t_5),r(z_5,t_5)\},\{t_5{\mapsto}t_2^1,y_5{\mapsto}u_1^1,z_5{\mapsto}v_1^1\}).$
\end{itemize}
%
%
This pattern is associated with a bag equivalent to the abstract bag $\mathbb{B}(R_2,id)$. Thus, the abstract pattern $\mathbb{P}$ associated with this initial pattern contains the same elements, where the mappings are modified by substituting $t_2^1$ by $t_2$, $u_2^1$ by $u_2$, $v_2^1$ by $v_2$, $u_1^1$ by $y_2$ and $v_1^1$ by $z_2$.
 \end{example}

%
%
%

\begin{definition}[Initial Abstract Pattern]
Let $\mathbb{B}$ be an abstract bag. The \emph{initial abstract pattern} of $\mathbb{B}$, denoted by $\mathbb{P}_\mathrm{init}(\mathbb{B})$ is the set of pairs $(G,\pi)$ such that $G$ is a subset of a rule body and $\pi$ is a (full) homomorphism from $G$ to $\fun{atoms}(\mathbb{B})$.
\end{definition}

Let $B_1$ and $B_2$ be two bags of a derivation tree such that $B_2$ is a child of $B_1$. $B_1$ and $B_2$ share some terms. Let us assume that $B_1$ is structurally equivalent to an abstract bag $\mathbb{B}_1$ and that $B_2$ is structurally equivalent to an abstract bag $\mathbb{B}_2$. If we only state that a bag equivalent to $\mathbb{B}_2$ is a child of a bag equivalent to $\mathbb{B}_1$, we miss some information about the above mentioned shared terms. Capturing this information is the purpose of the notion of \emph{link}.

\begin{definition}[Link]
\label{def-link}
 Let $\mathbb{B}_1$ and $\mathbb{B}_2$ be two abstract bags. A \emph{link} from $\mathbb{B}_2$ to $\mathbb{B}_1$ is an injective mapping $\lambda$ from the frontier terms of $\mathbb{B}_2$ to the terms of $\mathbb{B}_1$ such that the range of $\lambda$ has a non-empty intersection with the generated terms of $\mathbb{B}_1$.
\end{definition}

Please note that we define a link from a bag to its parent. It ensures that each bag has exactly one link. We will thus refer without ambiguity to the link of an abstract bag. 

\begin{example}[Link]
Let us consider $R^\mathrm{ex}_1 = q_1(x_1,y_1,z_1) \rightarrow s(y_1,t_1) \wedge r(z_1,t_1) \wedge q_2(t_1,u_1,v_1)$ and $R^\mathrm{ex}_2 = q_2(x_2,y_2,z_2) \rightarrow s(y_2,t_2) \wedge r(z_2,t_2) \wedge q_3(t_2,u_2,v_2)$, and the two abstract bags $\mathbb{B}_1 = \mathbb{B}(R_1,\fun{id})$ and $\mathbb{B}_2 = \mathbb{B}(R_2,\fun{id})$. Then $\lambda=\{y_2{\mapsto}u_1,z_2{\mapsto}v_1\}$ is a link from $\mathbb{B}_2$ to $\mathbb{B}_1$.
\end{example}

We are also interested in the link that describes a particular situation in a derivation tree, hence the notion of \emph{induced} link.

\begin{definition}[Induced Link]
Let $B_1$ and $B_2$ be two bags of a derivation tree such that $B_2$ is a child of $B_1$. Let $\mathbb{B}_1$ and $\mathbb{B}_2$ be two abstract bags such that $\mathbb{B}_1\simeq B_1$ and $\mathbb{B}_2\simeq B_2$. The link \emph{induced} by $B_1$ and $B_2$ is the mapping $\lambda$ of the frontier terms of $\mathbb{B}_2$ to $\fun{terms}(\mathbb{B}_1)$ defined by
%
%
$\lambda(y) = \psi_{B_1\rightarrow \mathbb{B}_1} ( \psi_{\mathbb{B}_2\rightarrow B_2}(y)).$
We then also say that $B_2$ is linked to $B_1$ by $\lambda$.
\end{definition}

Previous Property \ref{prop-pattern-equivalence} states that the pattern of a bag determines the rules that are applicable on it. We will thus gather information relative to the structure of derivation trees by means of ``saturation rules'' whose intuition is explained by the following example. Note that these rules have nothing to do with existential rules.


\begin{example}
\label{ex-pattern-evolution} Let us consider $R_1 = r(x_1,y_1) \rightarrow s(x_1,y_1)$ and $R_2 = s(x_2,y_2) \rightarrow p(x_2)$. Let $P_1$ be the following pattern: $\{r(x,y),\{x{\mapsto}a,y{\mapsto} b\}\}$. For any bag $B$ of a derivation tree $\fun{DT}(S)$ such that $P_1 \sqsubseteq P(B)$. $R_1$ is applicable by mapping $x_1$ to $a$ and $y_1$ to $b$. This allows to derive $s(a,b)$ (which may be used to apply $R_2$). Thus, the pattern of $B$ in some derivation starting with $S$ contains $P_2 = \{(r(x_1,y_1),\{x_1{\mapsto} a,y_1{\mapsto} b\}),(s(x_2,y_2),\{x_2{\mapsto} a,y_2{\mapsto} b\})\}$. Let us point out that this pattern inclusion is valid in ``sufficiently complete'' derivations, but not necessarily in the derivation tree of each derivation sequence.
\end{example}

Example \ref{ex-pattern-evolution} gives the intuition behind \emph{evolution rules}: it exhibits a case where we can infer that if a bag has a pattern including $P_1$, then its pattern can evolve into a pattern including $P_2$. Such information will be gathered by evolution rules, and will be denoted by $\mathbb{P}_1 \rightsquigarrow \mathbb{P}_2$ with $\mathbb{P}_1$ and $\mathbb{P}_2$ being the abstract counterparts of $P_1$ and $P_2$, respectively. To deal with the creation of new bags, we design \emph{creation rules}. They allow us to derive information about the children that a bag with a given pattern must have. Such a rule will be denoted by $\mathbb{P}_1 \rightsquigarrow \lambda.\mathbb{P}_2$, and intuitively means that rules may be applied to ensure that any bag $B_1$ with pattern $P_1$ has a child $B_2$ with pattern $P_2$ such that the link induced by $B_1$ and $B_2$ is $\lambda$ and $\mathbb{P}_1$ and $\mathbb{P}_2$ are again the abstract counterparts of $P_1$ and $P_2$, respectively.

In the following, we show how to derive a set of \emph{sound} creation and evolution rules by means of Properties \ref{prop-creation-rule} to \ref{prop-lower-saturation}.


\begin{definition}[Sound Creation Rule]
\index{creation rule} \label{def-creation rule} Let $\mathbb{P}_1,\mathbb{P}_2$ be two abstract patterns, and $\lambda$ be a link between the support of $\mathbb{P}_2$ and the support of $\mathbb{P}_1$. A \emph{creation rule} is a rule of the following form:

$$\gamma_c: \mathbb{P}_1 \rightsquigarrow \lambda.\mathbb{P}_2.$$

$\gamma_c$ is \emph{sound} if for any derivation $S$, for any bag $B_1$ of $(\fun{DT}(S),P)$ such that $\mathbb{P}_1 \sqsubseteq P(B_1)$, there exists a derivation $S'$ extending $S$ with a child $B_2$ of $B_1$ in $(\fun{DT}(S'),P')$ linked by $\lambda$ to $B_1$, and for which $\mathbb{P}_2 \sqsubseteq P'(B_2)$.
\end{definition}


\begin{definition}[Sound Evolution Rule]
\index{evolution rule} \label{def-evolution rule} Let $\mathbb{P}_1,\mathbb{P}_2$ be two abstract patterns. An \emph{evolution rule} is a rule of the following form:
$$\gamma_e: \mathbb{P}_1 \rightsquigarrow \mathbb{P}_2.$$
$\gamma_e$ is \emph{sound} if $\mathbb{P}_1 \sqsubseteq \mathbb{P}_2$ and for any derivation $S$ and for any bag $B$ of $(\fun{DT}(S),P)$ satisfying $\mathbb{P}_1 \sqsubseteq P(B)$, there exists a derivation $S'$ extending $S$ with patterned derivation tree $(\fun{DT}(S'),P')$ such that $\mathbb{P}_2 \sqsubseteq P'(B)$.
\end{definition}

We now exhibit properties allowing to build sound rules.

\begin{property}
\label{prop-creation-rule} Let $\mathbb{P}$ be an abstract pattern with support $\mathbb{B}$, let $R$ be a rule from $\mathcal{R}$, and let $\pi$ be a mapping from $\fun{fr}(R)$ to $\fun{terms}(\mathbb{B})$ such that its range has a non empty intersection with the generated terms in $\mathbb{P}$. Let $(\fun{body}(R),\pi)$ be an element of $\mathbb{P}$. Let $\sigma$ be the fusion of $\fun{fr}(R)$ induced by $\pi$. Then $\mathbb{P} \rightsquigarrow \lambda.\mathbb{P}_\mathrm{init}(\mathbb{B}(R,\sigma))$ is a sound creation rule, where
$\lambda$ is equal to $\pi$ restricted to $\{\sigma(y)\mid y\in \fun{fr}(R)\}$.
\end{property}

\begin{proof}
Since the range of $\pi$ has a non-empty intersection with the set of generated terms of $\mathbb{B}(R,\sigma)$, $\lambda$ is a link from $\mathbb{B}(R,\sigma)$ to the support of $\mathbb{P}$. Moreover, let $B$ be a bag of a derivation tree such that $\mathbb{P} \sqsubseteq P(B)$. Then $(\fun{body}(R),\psi_{\fun{support}(\mathbb{P})\rightarrow B}\circ\pi) \in P(B)$. Thus, $R$ is applicable, by mapping its frontier to $\fun{terms}(B)$ (and at least one term generated in $B$ is the image of an element of the frontier). Thus $B$ has a child with link $\lambda$ and with a pattern that includes $\mathbb{P}_\mathrm{init}(\mathbb{B}(R,\sigma))$.
\end{proof}

We now define the counterpart of elementary joins for abstract patterns. The main difference is that the relationships between terms of different abstract patterns cannot be checked by equality as it was done previously. We thus define abstract elementary joins, where these relationships are specified by the link between two abstract patterns. A link between two patterns is not symmetric: we thus define two join operations, to update either the abstract pattern that is the domain of the link or the abstract pattern that is the range of the link.

\begin{definition}[Elementary Abstract Upper/Lower Join]
Let $\mathbb{P}_1$ and $\mathbb{P}_2$ be two abstract patterns, and let $\lambda$ be a link from $\mathbb{P}_2$ to $\mathbb{P}_1$. Let $R$ be a rule in $\mathcal{R}$ and let $(sb_1,\pi_1) \in \mathbb{P}_1$ and $(sb_2,\pi_2) \in \mathbb{P}_2$ for $sb_1,sb_2\subseteq \fun{body}(R)$. The \emph{elementary abstract upper and lower joins} of $(sb_1,\pi_1)$ with $(sb_2,\pi_2)$ are defined if  $\pi_1(x)$ and $\lambda(\pi_2(x))$ are defined and equal for all $x \in \fun{vars}(sb_1) \cap \fun{vars}(sb_2)$. In that case, it is the pair $(sb_1 \cup sb_2,\pi)$ with:
\begin{itemize}
\item  $\pi = \pi_1 \cup \lambda\circ \pi_2'$, where $\pi_2'$ is the restriction of $\pi_2$ to $\pi_2^{-1}(\fun{domain}(\lambda))$, for the upper join;
\item $\pi = \pi_2 \cup \lambda^{-1}\circ\pi_1'$, where $\pi_1'$ is the restriction of $\pi_1$ to $\pi_1^{-1}(\fun{range}(\lambda))$, for the lower join.
\end{itemize}
\end{definition}

%

%

\begin{definition}[Abstract Upper/Lower Join]
Let $\mathbb{P}_1$ and $\mathbb{P}_2$ be two abstract patterns, and let $\lambda$ be a link from $\mathbb{P}_2$ to $\mathbb{P}_1$.

The \emph{abstract upper join} of $\mathbb{P}_1$  w.r.t. $(\lambda,\mathbb{P}_2)$ is the set of all existing elementary abstract upper joins of $(sb_1,\pi_1) \in \mathbb{P}_1$ with $(sb_2,\pi_2)\in \mathbb{P}_2$, where $sb_1,sb_2\subseteq \fun{body}(R)$ for some $R\in\mathcal{R}$. It is denoted by $\fun{Join$_\mathrm{u}$}(\mathbb{P}_1,\lambda,\mathbb{P}_2)$.

The \emph{abstract lower join} of $\mathbb{P}_2$ w.r.t. $(\lambda,\mathbb{P}_1)$ is the set of all existing elementary abstract lower joins of $(sb_1,\pi_1) \in \mathbb{P}_1$ with $(sb_2,\pi_2)\in \mathbb{P}_2$, where $sb_1,sb_2\subseteq \fun{body}(R)$ for some $R\in\mathcal{R}$. It is denoted by $\fun{Join$_\mathrm{l}$}(\mathbb{P}_1,\lambda,\mathbb{P}_2)$.

\end{definition}

We now exploit this notion of join in order to define new sound creation and evolution rules.

\begin{property}
\label{prop-abstract-lower-join} If $\mathbb{P}_1 \rightsquigarrow \lambda.\mathbb{P}_2$ is a sound creation rule, then so is $\mathbb{P}_1 \rightsquigarrow \lambda.\fun{Join$_\mathrm{l}$}(\mathbb{P}_1,\lambda,\mathbb{P}_2)$.
\end{property}

\begin{proof}
Let $S$ be a derivation, $B_1$ be a bag of $(\fun{DT}(S),P)$ such that $\mathbb{P}_1 \sqsubseteq P(B_1)$. Since $\mathbb{P}_1 \rightsquigarrow \lambda.\mathbb{P}_2$ is sound, there are a derivation $S'$ with patterned derivation tree $(\fun{DT}(S'),P')$ and a child $B_2$ of $B_1$ in $S'$ linked to $B_1$ by $\lambda$ such that $\mathbb{P}_2 \sqsubseteq P'(B_2)$.
 By soundness of join propagation, $\fun{Join}(P'(B_2),P'(B_1)) \sqsubseteq P'(B_2)$. By monotonicity of the join operation, we obtain that $\mathbb{P}_1 \rightsquigarrow \lambda.\fun{Join$_\mathrm{l}$}(\mathbb{P}_1,\lambda,\mathbb{P}_2))$ is a sound rule.
\end{proof}

\begin{property}
\label{prop-abstract-upper-join} If $\mathbb{P}_1 \rightsquigarrow \lambda.\mathbb{P}_2$ is a sound creation rule, then $\mathbb{P}_1 \rightsquigarrow \fun{Join$_\mathrm{u}$}(\mathbb{P}_1,\lambda,\mathbb{P}_2)$ is a sound evolution rule.
\end{property}

\begin{proof}
Similar to the proof of Prop \ref{prop-abstract-lower-join}.
\end{proof}

\begin{property}
\label{prop-trans-evol}
If $\mathbb{P}_1 \rightsquigarrow \mathbb{P}_2$ and $\mathbb{P}_2 \rightsquigarrow \mathbb{P}_3$ are sound evolution rules, then $\mathbb{P}_1 \rightsquigarrow \mathbb{P}_3$ is also a sound evolution rule.
\end{property}

\begin{proof}
Let $S$ be a derivation, and let $B$ be a bag of $(\fun{DT}(S),P)$ such that $\mathbb{P}_1 \sqsubseteq P(B)$. Since $\mathbb{P}_1 \rightsquigarrow \mathbb{P}_2$ is sound, there exists a derivation $S'$ extending $S$ such that $\mathbb{P}_2 \sqsubseteq P'(B)$. Since $\mathbb{P}_2 \rightsquigarrow \mathbb{P}_3$ is sound, there exists a derivation $S''$ extending $S'$ such that $\mathbb{P}_3 \sqsubseteq P''(B)$. Since $S''$ is also a derivation extending $S$, it holds that $\mathbb{P}_1 \rightsquigarrow \mathbb{P}_3$ is sound.
\end{proof}

\begin{property}
\label{prop-evolution-creation-combined}
If $\mathbb{P}_1 \rightsquigarrow \mathbb{P}_2$ and $\mathbb{P}_1 \rightsquigarrow \lambda.\mathbb{P}_3$ are sound evolution/creation rules, then $\mathbb{P}_2 \rightsquigarrow \lambda.\mathbb{P}_3$ is a sound creation rule.
\end{property}

\begin{proof}
The property holds by monotonicity of the join operation, and by the condition that $\mathbb{P}_1 \rightsquigarrow \mathbb{P}_2$ being sound implies $\mathbb{P}_1 \sqsubseteq \mathbb{P}_2$.
\end{proof}

\begin{property}
\label{prop-lower-saturation} If $\mathbb{P}_1 \rightsquigarrow \lambda.\mathbb{P}_2$ and $\mathbb{P}_2 \rightsquigarrow \mathbb{P}_3$ are sound creation/evolution rules, then $\mathbb{P}_1 \rightsquigarrow \lambda.\mathbb{P}_3$ is a sound creation rule.
\end{property}

\begin{proof}
Let $S$ be a derivation, and let $B_1$ be a bag of $(\fun{DT}(S),P)$ such that $\mathbb{P}_1 \sqsubseteq P(B_1)$. Since $\mathbb{P}_1 \rightsquigarrow \lambda.\mathbb{P}_2$ is sound, there are a derivation $S'$ with patterned derivation tree $(\fun{DT}(S'),P')$ and a child $B_2$ of $B_1$ in $S'$ that is linked to $B_1$ by $\lambda$ such that $\mathbb{P}_2 \sqsubseteq P'(B_2)$. Since $\mathbb{P}_2 \rightsquigarrow \mathbb{P}_3$ is sound, there exists a derivation $S''$ extending $S'$ such that $\mathbb{P}_3 \sqsubseteq P''(B_2)$. $S''$ extends $S$ as well, and thus $\mathbb{P}_1 \rightsquigarrow \lambda.\mathbb{P}_3$ is a sound creation rule.
\end{proof}

We call \emph{(abstract) pattern saturation} the already outlined procedure that builds all creation and evolution rules w.r.t.~$F$ and $\mathcal{R}$, obtained via an exhaustive application of all deduction rules displayed in Fig.~\ref{fig-deductioncalculus}. We now prove that this process terminates.


\begin{figure}

\hfill
\begin{tabular}{c}
\\\hline
$\mathbb{P} \rightsquigarrow \pi_{\mid \sigma(\fun{fr}(R))}.\mathbb{P}_\mathrm{init}(\mathbb{B}(R,\sigma))$
\end{tabular}
Prop.~\ref{prop-creation-rule}
\hfill~

\bigskip

\hfill
\begin{tabular}{c}
$\mathbb{P}_1 \rightsquigarrow \lambda.\mathbb{P}_2$ \\\hline
$\mathbb{P}_1 \rightsquigarrow \lambda.\fun{Join$_\mathrm{l}$}(\mathbb{P}_1,\lambda,\mathbb{P}_2)$
\end{tabular}Prop.~\ref{prop-abstract-lower-join}
\hfill
\begin{tabular}{c}
$\mathbb{P}_1 \rightsquigarrow \lambda.\mathbb{P}_2$\\\hline
$\mathbb{P}_1 \rightsquigarrow \fun{Join$_\mathrm{u}$}(\mathbb{P}_1,\lambda,\mathbb{P}_2)$
\end{tabular}Prop.~\ref{prop-abstract-upper-join}
\hfill~

\bigskip

\hfill\begin{tabular}{c}
$\mathbb{P}_1 \rightsquigarrow \mathbb{P}_2$ \ \ \  $\mathbb{P}_2 \rightsquigarrow \mathbb{P}_3$\\\hline
$\mathbb{P}_1 \rightsquigarrow \mathbb{P}_3$
\end{tabular}Prop.~\ref{prop-trans-evol}
\hfill~

\bigskip

\hfill
\begin{tabular}{c}
$\mathbb{P}_1 \rightsquigarrow \mathbb{P}_2$ \ \ \  $\mathbb{P}_1 \rightsquigarrow \lambda.\mathbb{P}_3$\\\hline
$\mathbb{P}_2 \rightsquigarrow \lambda.\mathbb{P}_3$
\end{tabular}Prop.~\ref{prop-evolution-creation-combined}
\hfill
\begin{tabular}{c}
$\mathbb{P}_1 \rightsquigarrow \lambda.\mathbb{P}_2$ \ \ \  $\mathbb{P}_2 \rightsquigarrow \mathbb{P}_3$\\\hline
$\mathbb{P}_1 \rightsquigarrow \lambda.\mathbb{P}_3$
\end{tabular}Prop.~\ref{prop-lower-saturation}
\hfill~

\caption{\label{fig-deductioncalculus}Deduction calculus for the pattern saturation rules. For the first deduction rule, $R\in \mathcal{R}\cup\{\to F\}$, $\pi$ is a homomorphism from $\fun{fr}(R)$ to $\fun{terms}(\fun{support}(\mathbb{P}))$ such that $(\fun{body}(R),\pi)\in \mathbb{P}$, and $\sigma$ is the fusion of $\fun{fr}(R)$ induced by $\pi$.
}
\end{figure}

\begin{property}[Termination]
For any fact $F$ and any \RgRbRtRsR set of rules $\mathcal R$, abstract pattern saturation terminates.
\end{property}

\begin{proof}
There is a finite number of abstract patterns and links between them, and thus a finite number of evolution and creation rules. At each step, the number of created rules can only increase, which shows the termination of pattern saturation.
\end{proof}

For technical purposes, we will use the \emph{rank} of an evolution/creation rule. 

\begin{definition}[Rank]
The \emph{rank} of an evolution or a creation rule is the minimum number of deduction rules (Figure \ref{fig-deductioncalculus}) necessary to derive that rule.
\end{definition}

This notion of rank helps us to prove the following technical lemma, that states that the pattern saturation respects some notion of monotonicity: at least as much information can be derived from an abstract pattern $\mathbb{P}_2$ as from an abstract pattern $\mathbb{P}_1$ if $\mathbb{P}_1 \sqsubseteq \mathbb{P}_2$.

\begin{lemma}
\label{lemma-monotonic-saturation} Let $\mathbb{P}_1$ and $\mathbb{P}_2$ be two abstract patterns such that $\mathbb{P}_1 \sqsubseteq \mathbb{P}_2$. For any rule $\mathbb{P}_1 \rightsquigarrow \mathbb{P}'_1$ (resp. $\mathbb{P}_1 \rightsquigarrow \lambda.\mathbb{P}'_1$) in the pattern saturation, there exists a rule $\mathbb{P}_2 \rightsquigarrow \mathbb{P}'_2$ (resp. $\mathbb{P}_2 \rightsquigarrow \lambda.\mathbb{P}'_2$) in the pattern saturation such that $\mathbb{P}'_1 \sqsubseteq \mathbb{P}'_2$.
\end{lemma}

\begin{proof}
We prove the result by induction on the rank of $\mathbb{P}_1 \rightsquigarrow \mathbb{P}'_1$ (resp. $\mathbb{P}_1 \rightsquigarrow \lambda.\mathbb{P}'_1$). At rank $0$, the result is vacuously true.
\begin{itemize}
\item Let $\mathbb{P}_1\rightsquigarrow \mathbb{P}'_1$ be a rule of rank $n$ of the pattern saturation. It has been obtained by applying Property \ref{prop-abstract-upper-join} or Property \ref{prop-trans-evol} to rules of rank strictly smaller than $n$. Let us first consider that Property \ref{prop-abstract-upper-join} has been applied. Let $\mathbb{P}_1 \rightsquigarrow \lambda.\tilde{\mathbb{P}_1}$ be the rule on which Property \ref{prop-abstract-upper-join} has been applied. By induction hypothesis, there exists a rule $\mathbb{P}_2 \rightsquigarrow \lambda.\tilde{\mathbb{P}_2}$ in the pattern saturation such that $\tilde{\mathbb{P}_1}\sqsubseteq\tilde{\mathbb{P}_2}$. By monotonicity of the join operation, it holds that $\mathbb{P}'_1 = \fun{Join$_\mathrm{u}$}(\mathbb{P}_1,\lambda,\tilde{\mathbb{P}_1}) \sqsubseteq \fun{Join$_\mathrm{u}$}(\mathbb{P}_2,\lambda,\tilde{\mathbb{P}_2}) = \mathbb{P}'_2$. Thus $\mathbb{P}_2 \rightsquigarrow \mathbb{P}'_2$ is in the pattern saturation and $\mathbb{P}'_1 \sqsubseteq \mathbb{P}'_2$. Let us now consider that Property \ref{prop-trans-evol} has been used to create $\mathbb{P}_1 \rightsquigarrow \mathbb{P}'_1$. Then, there are two rules of rank strictly smaller than $n$, namely $\mathbb{P}_1 \rightsquigarrow \mathbb{P}''_1$ and $\mathbb{P}''_1 \rightsquigarrow \mathbb{P}'_1$. By induction hypothesis, there is a rule $\mathbb{P}_2 \rightsquigarrow \mathbb{P}''_2$ in the pattern saturation such that $\mathbb{P}''_1 \sqsubseteq \mathbb{P}''_2$. We can once again apply the induction hypothesis, and we conclude that there exists a rule $\mathbb{P}''_2 \rightsquigarrow \mathbb{P}'_2$ in the pattern saturation, where $\mathbb{P}'_1 \sqsubseteq \mathbb{P}'_2$. By applying Property \ref{prop-trans-evol}, we conclude that $\mathbb{P}_2 \rightsquigarrow \mathbb{P}'_2$ is in the pattern saturation.
\item Let $\gamma_e: P_1 \rightsquigarrow \lambda.\mathbb{P}'_1$ be a rule of rank $n$ of the pattern saturation. It may have been created by application of Properties \ref{prop-creation-rule}, \ref{prop-abstract-lower-join}, \ref{prop-evolution-creation-combined} or \ref{prop-lower-saturation}.
  \begin{itemize}
  \item If $\gamma_e$ has been created by Property \ref{prop-creation-rule}, then the rule $\mathbb{P}_2 \rightsquigarrow \lambda.\mathbb{P}'_1$ can also be created thanks to this property, since $\mathbb{P}_1 \sqsubseteq \mathbb{P}_2$.
  \item If $\gamma_e$ has been created by Property \ref{prop-abstract-lower-join}, then there is a rule $\mathbb{P}_1 \rightsquigarrow \lambda.\mathbb{P}''_1$  of rank strictly smaller than $n$ in the pattern saturation, which has been used to create $\gamma_e$. By induction hypothesis, there is a rule $\mathbb{P}_2 \rightsquigarrow \lambda.\mathbb{P}''_2$. We define $\mathbb{P}'_2 = \fun{Join$_\mathrm{l}$}(\mathbb{P}_2,\lambda,\mathbb{P}''_2)$. By monotonicity of the join operation, we have that $\mathbb{P}'_1 = \fun{Join$_\mathrm{l}$}(\mathbb{P}_1,\lambda,\mathbb{P}''_1) \sqsubseteq \fun{Join$_\mathrm{l}$}(\mathbb{P}_2,\lambda,\mathbb{P}''_2) = \mathbb{P}'_2$. By applying Property \ref{prop-abstract-lower-join}, we create $\mathbb{P}_2 \rightsquigarrow \lambda.\mathbb{P}'_2$, which shows the claim.
  \item If $\gamma_e$ has been created by Property \ref{prop-evolution-creation-combined}, there are two rules $\mathbb{P}''_1 \rightsquigarrow \mathbb{P}_1$ and $\mathbb{P}''_1 \rightsquigarrow \lambda.\mathbb{P}'_1$ of rank strictly less than $n$ in the pattern saturation. Since $\mathbb{P}''_1 \sqsubseteq \mathbb{P}_1 \sqsubseteq \mathbb{P}_2$, we can directly apply the induction hypothesis and state the existence of $\mathbb{P}_2 \rightsquigarrow \lambda.\mathbb{P}'_2$ with $\mathbb{P}'_1 \sqsubseteq \mathbb{P}'_2$.
  \item If $\gamma_e$ has been created by Property \ref{prop-lower-saturation}, there exists two rules $\mathbb{P}_1 \rightsquigarrow \lambda.\mathbb{P}''_1$ and $\mathbb{P}''_1 \rightsquigarrow \mathbb{P}'_1$ of rank strictly less than $n$ in the pattern saturation. By induction hypothesis, there exists a rule $\mathbb{P}_2 \rightsquigarrow \lambda.\mathbb{P}''_2$ in the pattern saturation, with $\mathbb{P}''_1 \sqsubseteq \mathbb{P}''_2$. We can once again apply the induction hypothesis, inferring the existence of $\mathbb{P}''_2 \rightsquigarrow \mathbb{P}'_2$ in the pattern saturation, with $\mathbb{P}'_1 \sqsubseteq \mathbb{P}'_2$. By applying Property \ref{prop-lower-saturation}, we infer the existence of $\mathbb{P}_2 \rightsquigarrow \lambda.\mathbb{P}'_2$, which concludes the proof.
  \end{itemize}
\end{itemize}
 \end{proof}

In the obtained fixpoint, some rules are redundant. For instance, if there exist two rules $\mathbb{P}_1 \rightsquigarrow \lambda.\mathbb{P}_2$ and $\mathbb{P}_1 \rightsquigarrow \lambda.\mathbb{P}_3$, with $\mathbb{P}_2 \sqsubseteq \mathbb{P}_3$, then the first rule is implied by the second one. This motivate the definition of \emph{most informative rules}.

\begin{definition}[Most informative rules]
Let $F$ be a fact and $\mathcal R$ be a \RgRbRtRsR set of rules. The set of most informative rules associated with $F$ and $\mathcal R$, denoted by $\mathcal I_{F,\mathcal R}$ is the maximal subset of the abstract pattern saturation such that:
\begin{itemize}
\item a creation rule $\mathbb{P}_1 \rightsquigarrow \lambda.\mathbb{P}_2$ belongs to $\mathcal I_{F,\mathcal R}$ if there is no rule in the abstract $\mathbb{P}_1 \rightsquigarrow \lambda.\mathbb{P}_2$ with $\mathbb{P}_2 \not = \mathbb{P}_3$ and $\mathbb{P}_2 \sqsubseteq \mathbb{P}_3$;
\item an evolution rule $\mathbb{P}_1 \rightsquigarrow \mathbb{P}_2$ belongs to $\mathcal I_{F,\mathcal R}$ if there is no rule $\mathbb{P}_1 \rightsquigarrow \mathbb{P}_2$ in the abstract pattern saturation that satisfies $\mathbb{P}_2 \not = \mathbb{P}_3$ and $\mathbb{P}_2 \sqsubseteq \mathbb{P}_3$.
\end{itemize}
\end{definition}

Let us notice that we can without ambiguity speak about \emph{the} evolution rule of the most informative rule set having a given abstract pattern as left hand side (when it exists), as there is at most one such rule. 
\begin{property}
Let $F$ be a fact, $\mathcal R$ be a \RgRbRtRsR set of rules, and $\mathbb P$ be an abstract pattern. $\mathcal I_{F,\mathcal R}$ contains at most one evolution rule and at most one creation rule having $\mathbb P$ as left-hand side.
\end{property}

\begin{proof}
We show that if $\mathbb P \rightsquigarrow \lambda.\mathbb P_1$ and $\mathbb P \rightsquigarrow \lambda.\mathbb P_2$ (resp. $\mathbb P \rightsquigarrow \mathbb P_1$ and $\mathbb P \rightsquigarrow \mathbb P_2$) belong to the pattern saturation, then there exists $\mathbb P_3$ with $\mathbb P_1 \sqsubseteq \mathbb P_3$ and $\mathbb P_2 \sqsubseteq \mathbb P_3$ such that $\mathbb P \rightsquigarrow \lambda. \mathbb P_3$ (resp. $\mathbb P \rightsquigarrow \mathbb P_3$) belongs to the pattern saturation as well.	

We assume without loss of generality that $\mathbb P \rightsquigarrow \lambda.\mathbb P_1$ (resp. $\mathbb P \rightsquigarrow \mathbb P_1$) is the rule of smallest rank, and we prove the result by induction on that rank. A rule can be of rank $1$ if and only if it has been created thanks to Property \ref{prop-creation-rule}. Since the only way to create a rule of the shape $\mathbb P \rightsquigarrow \lambda.\mathbb P_1$ is to use that property and that other deduction rules may only make the patterns grow, it holds that $\mathbb P_1 \sqsubseteq \mathbb P_2$, and we can take $\mathbb P_3 = \mathbb P_2$. The results thus holds if the first rule is of rank $1$. Let us assume the result to be true for any rule up to rank $n$, and let us show that it is true as well for any rule of rank $k+1$. We first consider creation rules and we distinguish three cases:
\begin{itemize}
\item $\mathbb P \rightsquigarrow \lambda.\mathbb P_1$ has been created by Property \ref{prop-abstract-lower-join}. We can apply the induction hypothesis on the premises of that deduction rule, say $\mathbb P \rightsquigarrow \lambda.\mathbb P'_1$. There exists thus $\mathbb P_3'$ such that $\mathbb P \rightsquigarrow \lambda.\mathbb P_3'$ belongs to the pattern saturation and $P_1' \sqsubseteq P_3'$ and $P_2 \sqsubseteq P_3'$. By applying then Property \ref{prop-abstract-lower-join}, and by monotonicity of the join operation, one get $\mathbb P \rightsquigarrow \lambda.\mathbb P_3$  as desired.
\item $\mathbb P \rightsquigarrow \lambda.\mathbb P_1$ has been created by Property \ref{prop-evolution-creation-combined}. We apply the induction hypothesis on the premise that is a creation rule, which allows us to conclude.
\item $\mathbb P \rightsquigarrow \lambda.\mathbb P_1$ has been created by Property \ref{prop-lower-saturation}. We apply the induction hypothesis on the premise that is a creation rule; Lemma \ref{lemma-monotonic-saturation} then allows us to conclude.
\end{itemize}
We now consider evolution rules. We distinguish two cases:
\begin{itemize}
\item $\mathbb P \rightsquigarrow \mathbb P_1$ has been created by Proposition \ref{prop-abstract-upper-join}. As in the first case of the creation rules, the result follow by induction hypothesis and monotonicity of the join operation.
\item $\mathbb P \rightsquigarrow \mathbb P_1$ has been created by Proposition \ref{prop-trans-evol}. The result follow by induction hypothesis on the first premise and by Lemma \ref{lemma-monotonic-saturation}.
\end{itemize}
\end{proof}

We illustrate pattern saturation by expanding the running example. Writing down absolutely every element of each pattern would impede the ease of reading. We will thus allow ourselves to skip some elements, and focus on the most important ones.

\begin{example}
The initial pattern $\mathbb{P}_F$ of $F_{ex}$ (Example \ref{ex-running}) contains the following elements:$(q_1(x_1,y_1,z_1),\{x_1{\mapsto}a,y_1{\mapsto}b,z_1{\mapsto}c\})$, $(q_1(x_1,y_1,z_1),\{x_1{\mapsto}d,y_1{\mapsto}c,z_1{\mapsto}e\})$ and $(q_1(x_1,y_1,z_1),\{x_1{\mapsto}f,y_1{\mapsto}g,z_1{\mapsto}g\}$. By application of Property \ref{prop-creation-rule},
 three novel rules are created: $\mathbb{P}_0 \rightsquigarrow \emptyset.\mathbb{P}_1^{b,c}$, $\mathbb{P}_0 \rightsquigarrow \emptyset.\mathbb{P}_1^{c,e}$ and $\mathbb{P}_0 \rightsquigarrow \emptyset.P_1^{g,g}$, where $\mathbb{P}_1^{b,c}, P_1^{c,e}$, and $\mathbb{P}_1^{g,g}$ are described below.

The atoms of the abstract bag associated with $\mathbb{P}_1^{b,c}$ are $\{s(b,t_1),r(c,t_1),q_2(t_1,u_1,v_1)\}$, and its link is empty (since the whole frontier of $R^\mathrm{ex}_1$ is mapped to constants). The atoms of the abstract bag associated with $\mathbb{P}_1^{c,e}$ are $\{s(c,t_1),r(e,t_1),q_2(t_1,u_1,v_1)\}$, and those of the abstract bag associated with $\mathbb{P}_1^{g,g}$ are $\{s(g,t_1),r(g,t_1),q_2(t_1,u_1,v_1)\}$.

$\mathbb{P}_1^{b,c}$ contains the following pairs:
\begin{itemize}
\item $(\{q_2(x_2,y_2,z_2)\},\{x_2{\mapsto}t_1,y_2{\mapsto}u_1,z_2{\mapsto}v_1\})$;
\item $(\{q_2(x_4,y_4,z_4)\},\{x_4{\mapsto}t_1,y_4{\mapsto}u_1,z_4{\mapsto}v_1\})$;
\item $(\{s(y_4,t_4)\},\{y_4{\mapsto}b, t_4{\mapsto}t_1\})$;
\item $(\{r(z_4,t_4)\},\{z_4{\mapsto}c,t_4{\mapsto}t_1\})$;
\item $(\{s(y_4,t_4),r(z_4,t_4)\},\{y_4{\mapsto}b,z_4{\mapsto}c,t_4{\mapsto}t_1\})$;
\item $(\{s(y_5,t_5)\},\{y_5{\mapsto}b, t_5{\mapsto}t_1\})$;
\item $(\{r(z_5,t_5)\},\{z_5{\mapsto}c, t_5{\mapsto}t_1\})$;
\item $(\{s(y_5,t_5),r(z_5,t_5)\},\{y_5{\mapsto}b,z_5{\mapsto}c,t_5{\mapsto}t_1\})$.
\end{itemize}

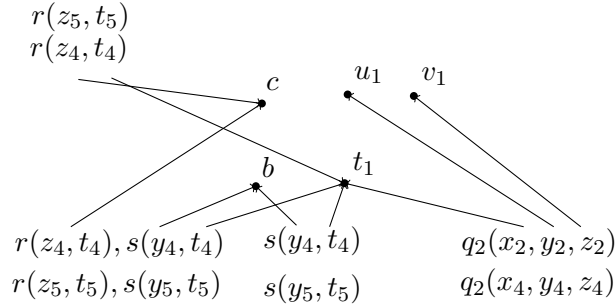
\begin{figure}
\begin{center}

\begin{tikzpicture}
\draw (-1.06,4.24) node[anchor=north west] {$s(b,t_1),r(c,t_1),q_2(t_1,u_1,v_1)$}; \draw (2.28,-0.24) node[anchor=north west] {$q_2(x_2,y_2,z_2)$}; \draw [->] (3.24,-0.36) -- (0.86,0.22); \draw [->] (3.64,-0.36) -- (0.9,1.4); \draw [->] (4,-0.36) -- (1.78,1.38); \draw (2.26,-0.78) node[anchor=north west] {$q_2(x_4,y_4,z_4)$}; \draw (-0.36,-0.22) node[anchor=north west] {$s(y_4,t_4)$}; \draw (-0.36,-0.86) node[anchor=north west] {$s(y_5,t_5)$}; \draw [->] (0.22,-0.34) -- (-0.32,0.18); \draw [->] (0.66,-0.36) -- (0.86,0.22); \draw (-3.46,2.34) node[anchor=north west] {$r(z_4,t_4)$}; \draw (-3.44,2.76) node[anchor=north west] {$r(z_5,t_5)$}; \draw [->] (-2.24,1.62) -- (0.86,0.22); \draw [->] (-2.68,1.6) -- (-0.24,1.28); \draw (-3.66,-0.26) node[anchor=north west] {$r(z_4,t_4),s(y_4,t_4)$}; \draw [->] (-2.78,-0.36) -- (-0.24,1.28); \draw [->] (-1.6,-0.36) -- (-0.32,0.18); \draw [->] (-0.98,-0.36) -- (0.86,0.22); \draw (-2.94,4.7) node[anchor=north west] {Atoms of the abstract bag:}; \draw (-3.7,-0.8) node[anchor=north west] {$r(z_5,t_5),s(y_5,t_5)$}; \fill [color=black] (0.86,0.22) circle (1.5pt); \draw[color=black] (1.12,0.48) node {$t_1$}; \fill [color=black] (-0.32,0.18) circle (1.5pt); \draw[color=black] (-0.16,0.44) node {$b$}; \fill [color=black] (-0.24,1.28) circle (1.5pt); \draw[color=black] (-0.1,1.54) node {$c$}; \fill [color=black] (0.9,1.4) circle (1.5pt); \draw[color=black] (1.18,1.66) node {$u_1$}; \fill [color=black] (1.78,1.38) circle (1.5pt); \draw[color=black] (2.06,1.64) node {$v_1$};
\end{tikzpicture}
\end{center}
\caption{A graphical representation of $\mathbb{P}_1^{b,c}$\label{fig-abstract-pattern} }
\end{figure}
$\mathbb{P}_1^{c,e}$ contains the same pairs, except that every occurrence of $b$ is replaced by $c$ and every occurrence of $c$ is replaced by $e$, whereas
$\mathbb{P}_1^{g,g}$ contains the same pairs, except that every occurrence of $b$ is replaced by $g$ and every occurrence of $c$ is replaced by $g$.

$\mathbb{P}_1^{b,c}$ is graphically represented in Figure \ref{fig-abstract-pattern}.

These three patterns contain $(\{q_2(x_2,y_2,z_2)\},\{x_2{\mapsto}t_1,y_2{\mapsto}u_1,z_2{\mapsto}v_1\})$, and we thus create the three following rules:
\begin{itemize}
\item $\mathbb{P}_1^{b,c} \rightsquigarrow \lambda'.\mathbb{P}_2,$
\item $\mathbb{P}_1^{c,e} \rightsquigarrow \lambda'.\mathbb{P}_2,$
\item $\mathbb{P}_1^{g,g} \rightsquigarrow \lambda'.\mathbb{P}_2,$
\end{itemize}

where $\lambda'=\{y_2{\mapsto}u_1,z_2{\mapsto}v_1\}$ and $\mathbb{P}_2$ is defined below.


The atoms of $\mathbb{P}_2$ are $\{s(y_2,t_2),r(z_2,t_2),q_3(t_2,u_2,v_2)\}$. It contains the following elements:
\begin{itemize}
\item $(\{q_3(t_3,u_3,v_3)\},\{t_3{\mapsto}t_2,u_3{\mapsto}u_2,v_3{\mapsto}v_2\})$,
\item $(\{s(y_4,t_4)\},\{y_4{\mapsto}y_2, t_4{\mapsto}t_2\})$,
\item $(\{r(z_4,t_4)\},\{z_4{\mapsto}z_2, t_4{\mapsto}t_2\})$,
\item $(\{s(y_4,t_4),r(z_4,t_4)\},z_4{\mapsto}z_2,y_4{\mapsto}y_2,t_4{\mapsto}t_2\})$,
\item $(\{s(y_5,t_5)\},\{y_5{\mapsto}y_2, t_5{\mapsto}t_2\})$,
\item $(\{r(z_5,t_5)\},\{z_5{\mapsto}z_2, t_5{\mapsto}t_2\})$,
\item $(\{s(y_5,t_5),r(z_5,t_5)\},\{y_5{\mapsto}y_2, z_5{\mapsto}z_2, t_5{\mapsto}t_2\})$.
\end{itemize}

The element $(\{q_3(t_3,u_3,v_3)\},\{t_3{\mapsto}t_2,u_3{\mapsto}u_2,v_3{\mapsto}v_2\})$ belongs to $\mathbb{P}_2$, and thus, we create a rule $\mathbb{P}_2 \rightsquigarrow \lambda''.\mathbb{P}_3)$, where $\lambda''=\{t_3{\mapsto}t_2\}$ and $\mathbb{P}_3$ contains the following elements:
\begin{itemize}
\item $(\{h(t_4)\},\{t_4{\mapsto}t_2\})$,
\item $(\{h(t_5)\},\{t_5{\mapsto}t_2\})$.
\end{itemize}

At this point, we cannot create any new rule by Property \ref{prop-creation-rule}. However, Property \ref{prop-abstract-upper-join} may be used to derive an evolution of $\mathbb{P}_2$. Indeed, since $\mathbb{P}_2 \rightsquigarrow \lambda''.\mathbb{P}_3$ has been derived, we can derive $\mathbb{P}_2 \rightsquigarrow \mathbb{P}'_2$ with $\mathbb{P}'_2 = \fun{Join$_\mathrm{u}$}(\mathbb{P}_2,\lambda'',\mathbb{P}_3)$. Note that $\mathbb{P}'_2$ is a superset of $\mathbb{P}_2$ that additionally contains the following elements:
\begin{itemize}
\item $(\{s(y_4,t_4),h(t_4)\},\{y_4{\mapsto}y_2,t_4{\mapsto}t_2\})$,
\item $(\{r(z_4,t_4),h(t_4)\},\{z_4{\mapsto}z_2,t_4{\mapsto}t_2\})$,
\item $(\{s(y_4,t_4),r(z_4,t_4),h(t_4)\},\{z_4{\mapsto}z_2,y_4{\mapsto}y_2,t_4{\mapsto}t_2\})$,
\item $(\{s(y_5,t_5),h(t_5)\},\{y_5{\mapsto}y_2, t_5{\mapsto}t_2\})$,
\item $(\{r(z_5,t_5),h(t_5)\},\{z_5{\mapsto}z_2, t_5{\mapsto}t_2\})$,
\item $(\{s(y_5,t_5),r(z_5,t_5),h(t_5)\},\{y_5{\mapsto}y_2, z_5{\mapsto}z_2, t_5{\mapsto}t_2\})$,
\item $(\{h(t_4)\}, \{t_4{\mapsto}t_2\})$,
\item $(\{h(t_5)\}, \{t_5{\mapsto}t_2\})$.
\end{itemize}

By Property \ref{prop-lower-saturation}, the following sound rules can then be obtained:
\begin{itemize}
\item $\mathbb{P}_1^{b,c} \rightsquigarrow \lambda'.\mathbb{P}_2',$
\item $\mathbb{P}_1^{c,e} \rightsquigarrow \lambda'.\mathbb{P}_2',$
\item $\mathbb{P}_1^{g,g} \rightsquigarrow \lambda'.\mathbb{P}_2'.$
\end{itemize}

Applying once more Property \ref{prop-abstract-upper-join} yields new sound rules such as:

$$\mathbb{P}_1^{b,c} \rightsquigarrow \mathbb{P}_1^{b,c'},$$

where $\mathbb{P}_1^{b,c'}$ is a superset of $\mathbb{P}_1^{b,c}$ that additionally contains, among others, the following element:

$$(\{q_2(x_4,y_4,z_4),r(z_4,t_4),s(y_4,t_4),h(t_4)\},\{x_4{\mapsto}t_1,y_4{\mapsto}u_1,z_4{\mapsto}v_1\}).$$

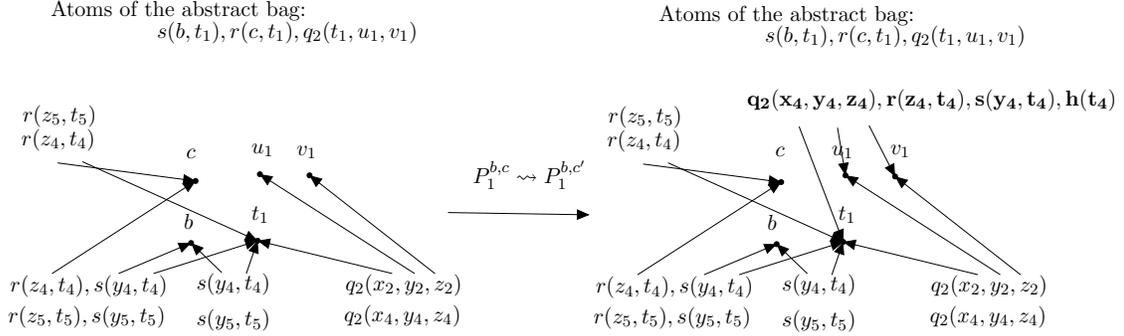
\begin{figure}
\begin{center}
\scalebox{0.75}{
\begin{tikzpicture}[line cap=round,line join=round,>=triangle 45,x=1.0cm,y=1.0cm]
\clip(-4.06,-3.28) rectangle (18.32,5.9);
\draw (-1.06,4.24) node[anchor=north west] {$s(b,t_1),r(c,t_1),q_2(t_1,u_1,v_1)$}; 
\draw (2.28,-0.24) node[anchor=north west] {$q_2(x_2,y_2,z_2)$}; 
\draw [->] (3.24,-0.36) -- (0.86,0.22); \draw [->] (3.64,-0.36) -- (0.9,1.4); 
\draw [->] (4,-0.36) -- (1.78,1.38); 
\draw (2.26,-0.78) node[anchor=north west] {$q_2(x_4,y_4,z_4)$}; 
\draw (-0.36,-0.22) node[anchor=north west] {$s(y_4,t_4)$}; 
\draw (-0.36,-0.86) node[anchor=north west] {$s(y_5,t_5)$};
\draw [->] (0.22,-0.34) -- (-0.32,0.18); 
\draw [->] (0.66,-0.36) -- (0.86,0.22); 
\draw (-3.46,2.34) node[anchor=north west] {$r(z_4,t_4)$}; 
\draw (-3.44,2.76) node[anchor=north west] {$r(z_5,t_5)$}; 
\draw [->] (-2.24,1.62) -- (0.86,0.22); 
\draw [->] (-2.68,1.6) -- (-0.24,1.28); 
\draw (-3.66,-0.26) node[anchor=north west] {$r(z_4,t_4),s(y_4,t_4)$}; 
\draw [->] (-2.78,-0.36) -- (-0.24,1.28); 
\draw [->] (-1.6,-0.36) -- (-0.32,0.18); 
\draw [->] (-0.98,-0.36) -- (0.86,0.22); 
\draw (-2.94,4.6) node[anchor=north west] {Atoms of the abstract bag:}; 
\draw (-3.7,-0.8) node[anchor=north west] {$r(z_5,t_5),s(y_5,t_5)$}; 
\draw (12.66,-0.26) node[anchor=north west] {$q_2(x_2,y_2,z_2)$}; 
\draw [->] (13.62,-0.38) -- (11.24,0.2); 
\draw [->] (14.02,-0.38) -- (11.28,1.38); 
\draw [->] (14.38,-0.38) -- (12.16,1.36); 
\draw (12.64,-0.8) node[anchor=north west] {$q_2(x_4,y_4,z_4)$}; 
\draw (10.02,-0.24) node[anchor=north west] {$s(y_4,t_4)$}; 
\draw (10.02,-0.88) node[anchor=north west] {$s(y_5,t_5)$}; 
\draw [->] (10.6,-0.36) -- (10.06,0.16);
\draw [->] (11.04,-0.38) -- (11.24,0.2); 
\draw (6.92,2.32) node[anchor=north west] {$r(z_4,t_4)$}; 
\draw (6.94,2.74) node[anchor=north west] {$r(z_5,t_5)$}; 
\draw [->] (8.14,1.6) -- (11.24,0.2); 
\draw [->] (7.7,1.58) -- (10.14,1.26); 
\draw (6.72,-0.28) node[anchor=north west] {$r(z_4,t_4),s(y_4,t_4)$}; 
\draw [->] (7.6,-0.38) -- (10.14,1.26); 
\draw [->] (8.78,-0.38) -- (10.06,0.16); 
\draw [->] (9.4,-0.38) -- (11.24,0.2); 
\draw (6.68,-0.82) node[anchor=north west] {$r(z_5,t_5),s(y_5,t_5)$}; 
\draw (9.72,4.18) node[anchor=north west] {$s(b,t_1),r(c,t_1),q_2(t_1,u_1,v_1)$}; 
\draw (7.84,4.52) node[anchor=north west] {Atoms of the abstract bag:}; 
\draw (9.9,2.02) node[anchor=north west] {$c$}; 
\draw (9.76,0.82) node[anchor=north west] {$b$}; 
\draw (11.02,0.96) node[anchor=north west] {$t_1$}; 
\draw (10.9,2.04) node[anchor=north west] {$u_1$}; 
\draw (11.94,2.06) node[anchor=north west] {$v_1$}; 
\draw (-0.54,1.98) node[anchor=north west] {$c$}; 
\draw (-0.58,0.84) node[anchor=north west] {$b$}; 
\draw (0.62,2.08) node[anchor=north west] {$u_1$}; 
\draw (0.62,0.96) node[anchor=north west] {$t_1$}; 
\draw (1.42,2.04) node[anchor=north west] {$v_1$}; 
\draw (9.4,3.04) node[anchor=north west] {$\mathit{\mathbf{q_2(x_4,y_4,z_4),r(z_4,t_4),s(y_4,t_4),h(t_4)}}$}; 
\draw [->] (10.46,2.24) -- (11.24,0.2); 
\draw [->] (11.14,2.22) -- (11.28,1.38); 
\draw [->] (11.7,2.26) -- (12.16,1.36); 
\draw [->] (4.24,0.72) -- (6.74,0.68); 
\draw (4.54,1.78) node[anchor=north west] {$P_1^{b,c} \rightsquigarrow P_1^{b,c'}$};
\begin{scriptsize}
\fill [color=black] (0.86,0.22) circle (1.5pt); \fill [color=black] (-0.32,0.18) circle (1.5pt); \fill [color=black] (-0.24,1.28) circle (1.5pt); \fill [color=black] (0.9,1.4) circle (1.5pt); \fill [color=black] (1.78,1.38) circle (1.5pt); \fill [color=black] (11.24,0.2) circle (1.5pt); \fill [color=black] (10.06,0.16) circle (1.5pt); \fill [color=black] (10.14,1.26) circle (1.5pt); \fill [color=black] (11.28,1.38) circle (1.5pt); \fill [color=black] (12.16,1.36) circle (1.5pt);
\end{scriptsize}
\end{tikzpicture}
}
\end{center}
\label{fig-evolution rule} \caption{Graphical representation of the rule $\mathbb{P}_1^{b,c} \rightsquigarrow \mathbb{P}_1^{b,c'}$. The new element of $ \mathbb{P}_1^{b,c'}$ is in bold.}
\end{figure}
Please note that in this case, $\pi=\{x_4{\mapsto}t_1,y_4{\mapsto}u_1,z_4{\mapsto}v_1\}$ does not map every variable appearing in the corresponding subset of a rule body. Indeed, $t_4$ is not mapped, since its image by the homomorphism extending $\pi$ does not belong to the terms relevant to the supporting bag.

We skip a part of the further development of this example. It can be checked that at some point, a rule $\mathbb{P}_F \rightsquigarrow \mathbb{P}'_F$ is created, where $\mathbb{P}'_F$ contains the following elements:

\begin{align*}
& (\{p_1(x_p),i(x_p)\},\{x_p{\mapsto}g\}) & (\{p_2(x_q),i(x_q)\},\{x_q{\mapsto}g\})
\end{align*}

The following two creation rules are thus sound and relevant:

\begin{align*}
& \mathbb{P}_0' \rightsquigarrow \emptyset.\mathbb{P}_p^i & \mathbb{P}_0' \rightsquigarrow \emptyset.\mathbb{P}_q^i
\end{align*}

where $\mathbb{P}_p^i$ contains in particular $(\{p_2(x_q),i(x_q)\},\{x_q{\mapsto}y_p\})$ and $\mathbb{P}_q^i$ contains $(\{p_1(x_p),i(x_p)\},\{x_p{\mapsto}y_q\})$. Since the body of $R^\mathrm{ex}_{6}$ belongs to $\mathbb{P}_p^i$, a new creation rule is added: $\mathbb{P}_p^i \rightsquigarrow \{x_p{\mapsto}y_q\}.\mathbb{P}_q$.

Likewise, since the body of $R^\mathrm{ex}_{7}$ belongs to $\mathbb{P}_q^i$, a new creation rule is added: $\mathbb{P}_q^i \rightsquigarrow \{x_q{\mapsto}y_p\}.\mathbb{P}_p$.

Last, two recursive rules are added:

\begin{align*}
& \mathbb{P}_p \rightsquigarrow \{x_p{\mapsto}y_q\}.\mathbb{P}_q, & \mathbb{P}_q \rightsquigarrow \{x_q{\mapsto}y_p\}.\mathbb{P}_p.
\end{align*}
\end{example}

\subsection{Computation of the Full Blocked Tree}


\begin{algorithm}[ht]
\KwData{A fact $F$, a set of \RgRbRtRsR rules $\mathcal R$, the set of most informative rules $\mathcal I_{F,\mathcal R}$.} 
\KwResult{$\mathfrak T_b^*(F,\mathcal R)$, a full blocked tree for $F$ and $\mathcal R$.} 
define the root of $\mathfrak T_b^*(F,\mathcal R)$ to be $B_F$\;
assign to $B_F$ a pattern $P_F \sim \mathbb P_F^*$, such that $\mathbb{P}_\mathrm{init}(\mathbb{B}(\to F,\emptyset))\rightsquigarrow \mathbb P_F^* \in \mathcal I_{F,\mathcal R}$ \;
blocked-patterns $:= P_F^*$\;
non-blocked-bags $:= B_F$\;
next-non-blocked $:= \emptyset$\;
\While{non-blocked-bags $\not = \emptyset$}{
next-non-blocked $:=\emptyset$\;
\For{$B_1 \in$ non-blocked-bags}{
$\mathbb P_1 := $ abstract pattern of $B_1$\;
\For{all creation rule $\mathbb P_1 \rightsquigarrow \lambda.\mathbb P_2 \in \mathcal I_{F,\mathcal R}$}
{
Add in $\mathfrak T_b^*(F,\mathcal R)$ a child $B_2$ to $B_1$, with induced link $\lambda$\;
Define the pattern of $B_2$ to be $P_2 \sim \mathbb P_2$\;
\If{$\mathbb P_2 \not \in $ blocked-patterns}{
next-non-blocked-bags := next-non-blocked-bags $\cup \{B_2\}$\;
blocked-patterns := blocked-patterns $\cup \{\mathbb P_2\}$\;
}
}
}
non-blocked-bags $:= $ next-non-blocked\;
}
\Return{$\mathfrak T_b^*(F,\mathcal R)$}\;
 \caption{Creation of a full blocked tree} 
\label{algo-generate-full-blocked-tree}
\end{algorithm}

Given a fact, a set of \RgRbRtRsRÊrules and their associated set of most informative rules, Algorithm \ref{algo-generate-full-blocked-tree} outputs a full blocked tree for $F$ and $\mathcal R$. 
We start by creating a bag with set of terms $T_0$. This bag is the root of the full blocked tree. We maintain a list of blocked patterns: any bag that is of that pattern and that is not labeled as non-blocked is thus blocked. We then consider the most informative evolution rule having $\mathbb{P}_F= \mathbb{P}_\mathrm{init}(\mathbb{B}(\to F,\emptyset))$ (i.e., the initial abstract pattern of $F$) as left-hand side, say $\mathbb{P}_F \rightarrow \mathbb{P}_F^*$. We label the root of the full blocked tree with the pattern $\mathbb{P}_F^*$.\footnote{Note that, technically, we abuse a abstract pattern as a non-abstract pattern here, but this is not a problem since no safe renaming is necessary for the (pattern of) bag $F$.} We mark this newly created root as being non-blocked. Then, as long as there exist a non-blocked bag $B_1$ and a most informative creation rule  $\mathbb{P}_1 \rightsquigarrow \lambda.\mathbb{P}_2$ with $\mathbb{P}_1\sim P(B_1)$ that has not been applied on $B_1$, we apply that rule. To apply it, we add a child $B_2$ to $B_1$ such that $P(B_2)\sim \mathbb{P}_2$ and the induced link from $B_2$ to $B_1$ is $\lambda$. $B_2$ is considered blocked (\emph{i.e.}, is not marked non-blocked) if there is already a bag $B_3$ with $P(B_2)\sim P(B_3)$ in the built structure, and non-blocked otherwise. This procedure halts, since there is a finite number of non-equivalent patterns, and the maximal degree of the built tree is also bounded. It creates a \emph{sound} blocked tree, since all creation and evolution rules are sound. It also creates a \emph{complete} blocked tree, and thus a \emph{full blocked tree}, as will be proven below.

Intuitively, the following property states that for any derivation tree associated with an $\mathcal R$-derivation of $F$, there exists an isomorphic tree generated by $\mathfrak T_b^*(F,\mathcal R)$.


 \begin{property}
 \label{prop-completeness-algo}
Let $F$ be a fact, $\mathcal R$ be \RgRbRtRsR. Let $S$ be an $\mathcal R$-derivation of $F$ and let $(\fun{DT}(S),P)$ be the according patterned derivation tree with root $B^\mathrm{root}$. Let $\mathfrak T_b^*(F,\mathcal R)$ be the corresponding full blocked tree with pattern-assigning function $P_{\mathfrak T_b^*}$ and root $B_{\mathfrak T_b^*}^\mathrm{root}$.\\
Then there exists a tree decomposition $\mathfrak T$ generated from $\mathfrak T_b^*(F,\mathcal R)$ via a mapping $f$, such that there exists a bijection $g$ from the bags of $(\fun{DT}(S),P)$ to the bags of $\mathfrak T$ that satisfies the following conditions:
\begin{enumerate}
\item $g(B^\mathrm{root}) = B_{\mathfrak T_b^*}^\mathrm{root}$, i.e., $g$ maps the root of $(\fun{DT}(S),P)$ to the root of $\mathfrak T_b^*(F,\mathcal R)$;
\item $P(B) \sqsubseteq P_{\mathfrak T_b^*}(f(g(B)))$ for all bags $B$ of $(\fun{DT}(S),P)$;
\item for all bags $B, B'$ of $(\fun{DT}(S),P)$ for which $B'$ is a child of $B$ with induced link $\lambda$, $g(B')$ is a child of $g(B)$ with induced link $\lambda$.
\end{enumerate}
 \end{property}

 \begin{proof}
 We prove the property by induction on the length of $S$.
 \begin{itemize}
 \item If $S$ is the empty derivation, its derivation tree is restricted to a single bag labeled by $F$. Such a tree can be generated from $\mathfrak T_b^*(F,\mathcal R)$, and the pattern of $B^\mathrm{root}_{\mathfrak{T}_b^*}$ is the root of $\mathfrak{T}$, is by construction greater than the initial pattern of the original fact.
 \item Let us assume that the property is true for all derivations of length $n \geq 0$, and let us show that it also holds for any derivation of length $n+1$. Let $S$ be a derivation of length $n+1$, and let $S_n$ be its restriction to the $n$ first rule applications. Let $B_{n+1}$ be the bag newly created in $\fun{DT}(S)$, and $B_n$ its parent. By induction hypothesis,  there exist $g$, $\mathfrak T_n$ and $f_n$ fulfilling the conditions from 1 to 3 for $S_n$. Let us consider $f_n(B_n)$. By condition $2$, we know that $P(f_n(g_n(B_n)))$ is greater than $P(B_n)$. By Lemma \ref{lemma-monotonic-saturation}, $f_n(g_n(B_n))$ has a child $B^*_{n+1}$ whose pattern includes that of $B_{n+1}$ and has induced link $\lambda$ with $f_n(g_n(B_n))$.  By definition of a tree generated from a blocked tree, it holds that $\mathfrak T_{n+1}$ can be generated from $\mathfrak T_b^*(F,\mathcal R)$ via $f_{n+1}$, where:
 \begin{itemize}
 \item $\mathfrak T_{n+1}$ is obtained from $\mathfrak T_n$ by copying $B^*_{n+1}$ under $g(B_n)$; we additionally define this bag as being $g(B_{n+1})$;
 \item $f_{n+1}$ is obtained by extending $f$ with $f_{n+1}(g(B_{n+1})) = B^*_{n+1}$.
 \end{itemize}
By induction hypothesis, it  holds that $g(B^\mathrm{root}) = B^\mathrm{root}_\mathfrak{T}$; Condition 1 is thus fulfilled. By construction of $g(B_{n+1})$, Condition 3 also. It remains to check Condition 2. This is not trivial, since the patterns of a bag in the fact associated with $S_n$ and with $S$ may be non-equivalent (\emph{i.e.}, the pattern may have ``grown''). Let us assume that there exists a bag $B^*$ such that $P(B^*) \not \sqsubseteq P_{\mathfrak T_b^*}(f(g(B^*)))$. Let us moreover assume that $B^*$ is (one of) the closest such bag to $B_{n+1}$. Let us first notice that it cannot be $B_{n+1}$. Indeed, $P(B_{n+1})$ is obtained by performing a join operation between its initial pattern and $P_{S_n}(B_n)$. By induction hypothesis, $P(B_n) \sqsubseteq P_{\mathfrak T_b^*}(f(g(B_n)))$. Thus, by Properties \ref{prop-creation-rule} and \ref{prop-abstract-lower-join}, the pattern saturation contains a rule allowing to create a child of $g(B_n)$ whose pattern includes $P_S(B_{n+1})$ and having induced link $\lambda$.  Thus $B^*$ is not $B_{n+1}$, and by Property \ref{prop-pattern-update}, $P(B^*)$ is obtained by performing a join between $P_{S_n}(B^*)$ and $P(B^*_k)$, where $B^*_k$ is the unique bag on the path from $B_{n+1}$ to $B^*$ that is either a child or a parent of $B^*$. Let us consider the case where $B^*_k$ is a parent of $B^*$ (the other case is similar). By induction hypothesis, $P_{S_n}(B^*) \sqsubseteq P_{\mathfrak T_b^*}(f(g(B^*)))$. By choice of $B^*$, $P(B^*_k) \sqsubseteq P_{\mathfrak T_b^*}(f(g(B^*_k)))$. By construction of $\mathfrak T_b^*(F,\mathcal R)$ and of its generated tree, there is a rule $\mathbb{P}_1 \rightsquigarrow \lambda.\mathbb{P}_2$ in the pattern saturation with $\mathbb{P}_1\sim P_{\mathfrak T_b^*}(f(g(B^*_k)))$ and $\mathbb{P}_2\sim P_{\mathfrak T_b^*}(f(g(B^*)))$. Moreover, since this rule is a most informative rule (by construction of $\mathfrak T_b^*(F,\mathcal R)$), $\fun{Join$_\mathrm{l}$}(\mathbb{P}_1,\lambda,\mathbb{P}_2) \sqsubseteq P(f(g(B^*)))$. However, by monotonicity of the join operation, this would imply that $P(B^*) \sqsubseteq P_{\mathfrak T_b^*}(f(g(B^*)))$, hence a contradiction.
\end{itemize}
\end{proof}



By preceding observations and Property \ref{prop-completeness-algo}, we are now able to state that Algorithm \ref{algo-generate-full-blocked-tree} is correct, as expressed by the next theorem. 

\begin{theorem}
Algorithm \ref{algo-generate-full-blocked-tree} outputs a full blocked tree.
\end{theorem}

Before turning to the more involved querying operation, let us stress that this first algorithm already provides a tight upper-bound for the combined complexity of query answering under \RgRbRtRsR rules. Indeed, the problem is already known to be \ExpExpTime{}-hard, since guarded rules - whose \ExpExpTime{} combined complexity was already shown \citep{cali-gottlob-kifer:08}, are a particular case of \RgRbRtRsR rules.


\begin{theorem} \textsc{BCQ-Entailment} for \RgRbRtRsR is in \ExpExpTime{} for combined complexity and in \ExpTime{} for data complexity.
\label{thm-easy-upper-bound}
\end{theorem}

\begin{proof}
Let us recall that $F,\mathcal{R}\models Q$ holds exactly if $F,\mathcal{R}\cup\{Q\to \mathit{match}\} \models \mathit{match}$, where $\mathit{match}$ is a fresh, nullary predicate. Note that $\mathcal{R}\cup\{Q\to \mathit{match}\}$ is still \RgRbRtRsR since $Q\to \mathit{match}$ is \RfRgR. $F,\mathcal{R}\cup\{Q\to \mathit{match}\} \models \mathit{match}$ can be easily checked given $\mathfrak{T}^*_b(F,\mathcal{R}\cup\{Q\to \mathit{match}\})$ by checking if any of the abstract patterns associated to any of the bags contain some pattern $(Q,\pi)$ for some $\pi$.
The computation of the full blocked tree is polynomial in the size of the computed creation/evolution rule set. The number of such rules is polynomial in the number of patterns and in the maximum degree of a derivation tree. The number of patterns is doubly exponential in the data in the worst-case, while the degree is at most exponential. When the rule set (including the query) is fixed, the data complexity falls to \ExpTime{}. Lower bounds for data complexity come from already known complexity results of weakly-guarded rules  \citep{cali-gottlob-kifer:08}, for instance.
\end{proof}

The algorithm we proposed is thus worst-case optimal both for combined and data complexities.

%% file: querying2.tex

We considered in previous sections the query to be implemented via a rule. This trick allowed to have a conceptually easy querying operation, because it was sufficient to check if some bag of the full blocked tree was labeled by the query and an arbitrary mapping. However, this comes with two drawbacks. The first one is that the query is needed at the time of the construction of the full blocked tree. In scenarios where different queries are evaluated against the same data, one would like to process data and rules independently from the query, and then to evaluate the query on the pre-processed structure. This is not possible if we consider the query to be expressed via a rule. The second drawback of taking the query into account while building the full blocked tree is that it may prevent us from adapting this construction when assumptions are made on the set of rules: can we devise a better algorithm if we have additional knowledge concerning the rule set, for instance, if we know that it is guarded, and not only \RgRbRtRsR?

\index{atom-term partition tree} This section is devoted to these issues. In the construction of the full blocked tree, we do not consider the query anymore. We still obtain a finite representation of arbitrarily deep patterned derivation trees for $F$ and $\mathcal R$, but we cannot just check if a bag is labeled by the query -- since the query does not necessarily appear in the considered patterns anymore. A simple homomorphism check is not sufficient either, as can be seen in Example \ref{ex-homomorhism-check-not-sufficient} below. To overcome this problem, we introduce a structure called \emph{atom-term partition tree} (APT). Such a structure is meant to encode a decomposition of the query induced by a homomorphism from that query to a derivation tree. A possible algorithm to check the existence of a homomorphism from a query $Q$ to a derivation tree would be to check if one of the APTs of $Q$ is the structure induced by some homomorphism $\pi$, \emph{i.e.} to \emph{validate} this APT. APTs and their validation in a derivation tree will be formalized in Section~\ref{sect-APT-DT}. We are well aware that this definition is more involved than the simple definition of homomorphism. However, our goal will be to validate APTs, not in the potentially ever-growing derivation trees, but in the \emph{finite} full blocked tree. In that case, APTs will still be used, but we will have to adjust their validation process (Section~\ref{sect-APT-BT}).

Let us first stress why the usual homomorphism check is not a suitable operation for querying a full blocked tree. To simplify the presentation, we will restrict the running example in the following way: we only consider rules $R^{ex'}_1 = p_1(x_p) \wedge i(x_p) \rightarrow r(x_p,y_p) \wedge p_2(y_p) \wedge i(y_p)$ and $R^{ex'}_2 = p_2(x_q) \wedge i(x_q) \rightarrow s(x_q,y_q) \wedge p_1(y_q) \wedge i(y_q)$ (this set will be denoted by $\mathcal R^{ex'}$), and the initial fact is restricted to $i(c) \wedge p_1(c) \wedge p_2(c)$ (denoted by $F^{ex'}$).

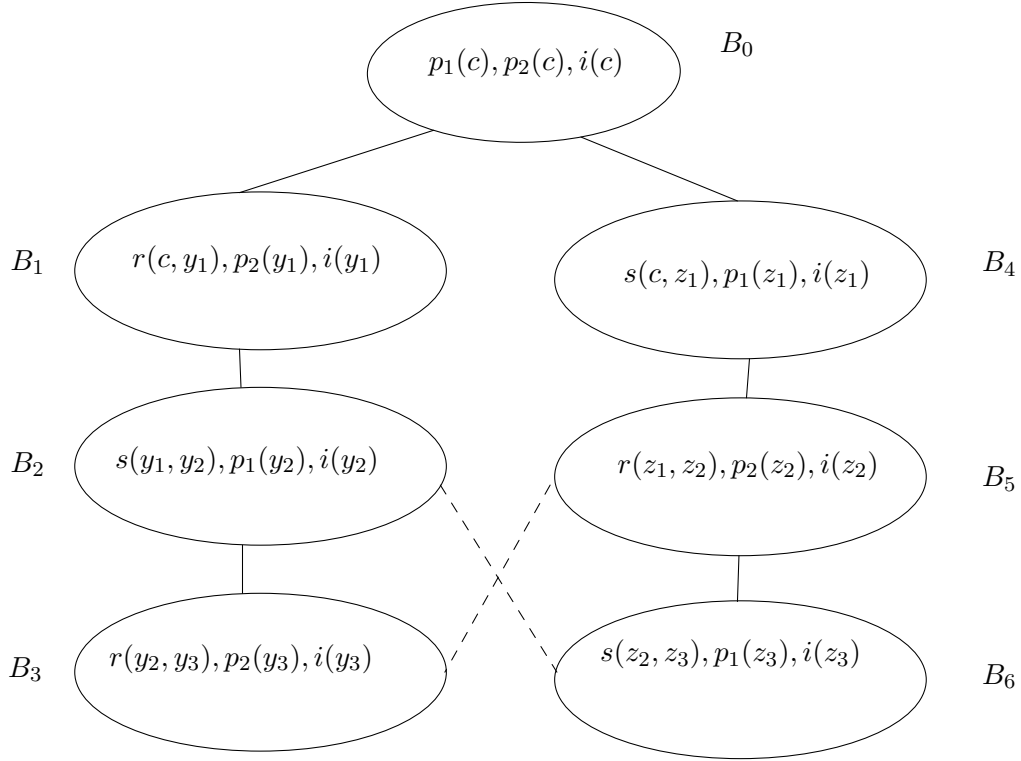
\begin{figure}
\begin{center}
\begin{tikzpicture}[line cap=round,line join=round,>=triangle 45,x=1.0cm,y=1.0cm]
\clip(-3.7,-5.82) rectangle (10.36,5.43);
\draw (2.26,4.34) node[anchor=north west] {$p_1(c), p_2(c), i(c)$}; 
\draw (-1.68,1.71) node[anchor=north west] {$r(c,y_1),p_2(y_1),i(y_1)$}; 
\draw (4.85,1.51) node[anchor=north west] {$s(c,z_1),p_1(z_1),i(z_1)$}; 
\draw (-1.91,-0.96) node[anchor=north west] {$s(y_1,y_2),p_1(y_2),i(y_2)$}; 
\draw (4.77,-1.02) node[anchor=north west] {$r(z_1,z_2),p_2(z_2),i(z_2)$}; 
\draw (-1.99,-3.59) node[anchor=north west] {$r(y_2,y_3),p_2(y_3),i(y_3)$}; 
\draw (4.54,-3.53) node[anchor=north west] {$s(z_2,z_3),p_1(z_3),i(z_3)$}; 
\draw [rotate around={0.62:(3.66,3.88)}] (3.66,3.88) ellipse (2.08cm and 0.93cm);
\draw [rotate around={0:(0.16,1.24)}] (0.16,1.24) ellipse (2.47cm and 1.05cm);
\draw [rotate around={0:(6.54,1.12)}] (6.54,1.12) ellipse (2.47cm and 1.05cm); 
\draw [rotate around={0:(0.04,-1.36)}] (0.16,-1.36) ellipse (2.47cm and 1.05cm); 
\draw [rotate around={0:(6.48,-1.5)}] (6.54,-1.5) ellipse (2.47cm and 1.05cm); 
\draw [rotate around={0:(-0.08,-4.1)}] (0.16,-4.1) ellipse (2.47cm and 1.05cm); 
\draw [rotate around={0:(6.34,-4.2)}] (6.54,-4.2) ellipse (2.47cm and 1.05cm); 
\draw (2.46,3.11)-- (-0.12,2.28); \draw (4.42,3.02)-- (6.5,2.17); 
\draw (-0.12,0.2)-- (-0.1,-0.31); \draw (6.66,0.07)-- (6.62,-0.45); 
\draw (6.52,-2.55)-- (6.5,-3.16); \draw (-0.08,-2.41)-- (-0.08,-3.05); 
\draw (6.13,4.55) node[anchor=north west] {$B_0$}; 
\draw (-3.3,1.65) node[anchor=north west] {$B_1$}; 
\draw (-3.3,-1.04) node[anchor=north west] {$B_2$}; 
\draw (-3.33,-3.76) node[anchor=north west] {$B_3$}; 
\draw (9.62,1.63) node[anchor=north west] {$B_4$}; 
\draw (9.62,-1.23) node[anchor=north west] {$B_5$}; 
\draw (9.62,-3.86) node[anchor=north west] {$B_6$}; 
\draw [dash pattern=on 4pt off 4pt] (2.55,-1.61)-- (4.1,-4.1); 
\draw [dash pattern=on 4pt off 4pt] (4.01,-1.61)-- (2.61,-4.1);
\end{tikzpicture}
\end{center}
\caption{The full blocked tree associated with $F^{ex'}$ and $\mathcal R^{ex'}$. $B_2$ and $B_6$ are equivalent, as well as $B_3$ and $B_5$.} \label{fig-incomplete-querying}
\end{figure}

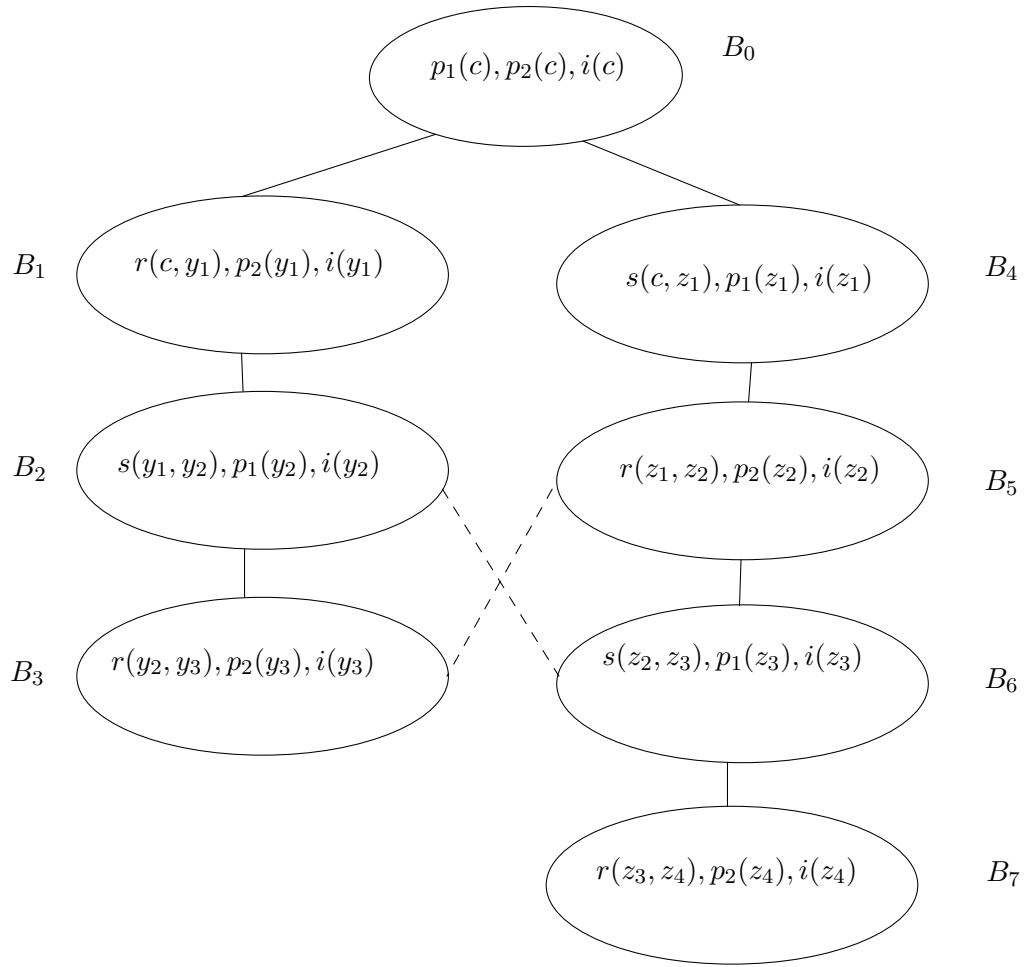
\begin{figure}
\begin{center}
\begin{tikzpicture}[line cap=round,line join=round,>=triangle 45,x=1.0cm,y=1.0cm]
\clip(-4.57,-8.35) rectangle (10.91,5.24);
\draw (2.26,4.34) node[anchor=north west] {$p_1(c), p_2(c), i(c)$}; 
\draw (-1.68,1.71) node[anchor=north west] {$r(c,y_1),p_2(y_1),i(y_1)$}; 
\draw (4.85,1.51) node[anchor=north west] {$s(c,z_1),p_1(z_1),i(z_1)$}; 
\draw (-1.91,-0.96) node[anchor=north west] {$s(y_1,y_2),p_1(y_2),i(y_2)$}; 
\draw (4.77,-1.02) node[anchor=north west] {$r(z_1,z_2),p_2(z_2),i(z_2)$}; 
\draw (-1.99,-3.59) node[anchor=north west] {$r(y_2,y_3),p_2(y_3),i(y_3)$}; 
\draw (4.54,-3.53) node[anchor=north west] {$s(z_2,z_3),p_1(z_3),i(z_3)$}; 
\draw [rotate around={0.62:(3.66,3.88)}] (3.66,3.88) ellipse (2.08cm and 0.93cm);
\draw [rotate around={0:(0.16,1.24)}] (0.16,1.24) ellipse (2.47cm and 1.05cm);
\draw [rotate around={0:(6.54,1.12)}] (6.54,1.12) ellipse (2.47cm and 1.05cm); 
\draw [rotate around={0:(0.04,-1.36)}] (0.16,-1.36) ellipse (2.47cm and 1.05cm); 
\draw [rotate around={0:(6.48,-1.5)}] (6.54,-1.5) ellipse (2.47cm and 1.05cm); 
\draw [rotate around={0:(-0.08,-4.1)}] (0.16,-4.1) ellipse (2.47cm and 1.05cm); 
\draw [rotate around={0:(6.34,-4.2)}] (6.54,-4.2) ellipse (2.47cm and 1.05cm); 
\draw (2.46,3.11)-- (-0.12,2.28); \draw (4.42,3.02)-- (6.5,2.17); 
\draw (-0.12,0.2)-- (-0.1,-0.31); \draw (6.66,0.07)-- (6.62,-0.45); 
\draw (6.52,-2.55)-- (6.5,-3.16); \draw (-0.08,-2.41)-- (-0.08,-3.05); 
\draw (6.13,4.55) node[anchor=north west] {$B_0$}; 
\draw (-3.3,1.65) node[anchor=north west] {$B_1$}; 
\draw (-3.3,-1.04) node[anchor=north west] {$B_2$}; 
\draw (-3.33,-3.76) node[anchor=north west] {$B_3$}; 
\draw (9.62,1.63) node[anchor=north west] {$B_4$}; 
\draw (9.62,-1.23) node[anchor=north west] {$B_5$}; 
\draw (9.62,-3.86) node[anchor=north west] {$B_6$}; 
\draw [dash pattern=on 4pt off 4pt] (2.55,-1.61)-- (4.1,-4.1); 
\draw [dash pattern=on 4pt off 4pt] (4.01,-1.61)-- (2.61,-4.1);
\draw (4.46,-6.37) node[anchor=north west] {$r(z_3,z_4),p_2(z_4),i(z_4)$}; 
\draw [rotate around={0:(6.4,-6.88)}] (6.4,-6.88) ellipse (2.47cm and 1.05cm); 
\draw (9.65,-6.43) node[anchor=north west] {$B_7$}; 
\draw (6.34,-5.25)-- (6.34,-5.83); 
\end{tikzpicture}
\end{center}
\caption{A tree generated by the full blocked tree of Figure \ref{fig-incomplete-querying}. $B_7$ is a copy of $B_3$ under $B_6$.} \label{fig-generated-tree}
\end{figure}

\begin{example}
 Let us consider the following query $Q_i$:

$$Q_i = p_i(x) \wedge s(x,y) \wedge r(y,z) \wedge s(z,t) \wedge r(t,u) \wedge r(x,v).$$

If we only look for a homomorphism with atoms belonging to the full blocked tree associated with $\mathcal R^{ex'}$ and $F^{ex'}$ and displayed in Figure \ref{fig-incomplete-querying}, we do not find any answer to this query. However, $B_2$ is equivalent to $B_6$, and by considering a derivation tree where $B_3$ would have a corresponding bag below $B_6$ (as $B_7$ in Figure \ref{fig-generated-tree}), one would find a (correct) mapping of $Q_i$. \label{ex-homomorhism-check-not-sufficient}
\end{example}

\subsubsection{Validation of an APT in a Derivation Tree}\label{sect-APT-DT}

Let $\pi$ be a homomorphism from $Q$ to the atoms of some derivation tree $\mathcal{T} = \mathit{DT}(S)$. From $\pi$, let us build an arbitrary mapping $\pi^\mathrm{a}_\mathcal{T}$ (out of the many possible ones), defined as follows: for every atom $a = p(t_1, \ldots, t_k)$ of $Q$, let us choose a bag $B$ of $\mathcal T$ with $\pi(a) = p(\pi(t_1), \ldots, \pi(t_k)) \in B$, and define $\pi^\mathrm{a}_\mathcal{T}(a) = (B, \pi(a))$. Then $\pi^\mathrm{a}_\mathcal{T}$ gives rise to a partitioning of the atoms of $Q$ into \emph{atom bags} $\mathit{Bags}^\mathrm{a}(Q)=\{Q^\mathrm{a}_1,\ldots,Q^\mathrm{a}_n\}$, where two atoms $a$ and $b$ of $Q$ are in the same atom bag $Q^\mathrm{a}_i$ if and only if there exists a bag $B$ of $T$ with $\pi^\mathrm{a}_\mathcal{T}(a) = (B,\pi(a))$ and $\pi^\mathrm{a}_\mathcal{T}(b) = (B,\pi(b))$.

On another note, it will turn out to be important, given a term $t$ of $Q$, to keep track of the bag of $T$ in which the term $\pi(t)$ appeared first. We note $\pi^\mathrm{t}_\mathcal{T}(t) = (B, \pi(t))$ when the term $\pi(t)$ appears first in the bag $B$ of $T$. Similar to above, $\pi^\mathrm{t}_\mathcal{T}$ gives rise to a partitioning of the terms of $Q$ into \emph{term bags} $\mathit{Bags}^\mathrm{t}(Q)=\{Q^\mathrm{t}_1,\ldots,Q^\mathrm{t}_m\}$, where two terms $u$ and $v$ of $Q$ are in the same term bag $Q^\mathrm{t}_i$ if and only if there exists a bag $B$ with $\pi^\mathrm{t}_\mathcal{T}(u) = (B, \pi(u))$ and $\pi^\mathrm{t}_\mathcal{T}(v) = (B, \pi(v))$.

From $\pi^\mathrm{a}_\mathcal{T}$ and $\pi^\mathrm{t}_\mathcal{T}$, we then obtain the function $\pi_\mathcal{T}$ mapping elements of $\mathit{Bags}^\mathrm{a}(Q)\cup \mathit{Bags}^\mathrm{t}(Q)$ to bags of $\mathcal{T}$ such that
$$
\pi_\mathcal{T}=\left\{
\begin{array}{ll}
Q^\mathrm{a} \mapsto B & \mbox{ where } \pi^\mathrm{a}_\mathcal{T}(a)=(B,\pi(a)) \mbox{ for any } a\in Q^\mathrm{a}\\
Q^\mathrm{t} \mapsto B' & \mbox{ where } \pi^\mathrm{t}_\mathcal{T}(a)=(B',\pi(t)) \mbox{ for any } t\in Q^\mathrm{t}\\
\end{array}\right.
$$

We can now define the \emph{atom-term} bags of $Q$ (induced by $\pi^\mathrm{a}_\mathcal{T}$) denoted by $\mathit{Bags}^\mathrm{at}(Q)$. If an atom bag $Q^\mathrm{a}$ and a term bag $Q^\mathrm{t}$ have the same image under $\pi_\mathcal{T}$, we obtain an atom-term bag $Q^\mathrm{at} = Q^\mathrm{a} \cup Q^\mathrm{t}$. If an atom bag $Q^\mathrm{a}$ (or a term bag $Q^\mathrm{t}$) has an image different from the image of any other term bag (or atom bag, respectively) of $Q$, then it is an atom-term bag by itself.

Finally, we provide these atom-term bags with a tree structure induced by the tree structure of $\mathcal{T}$. Let $Q^\mathrm{at}_1$ and $Q^\mathrm{at}_2$ be two atom-term bags of $Q$. Then $Q^\mathrm{at}_2$ is a child of $Q^\mathrm{at}_1$ iff {\sl (i)} $\pi_\mathcal{T}(Q^\mathrm{at}_2)$ is a descendant of of $\pi_\mathcal{T}(Q^\mathrm{at}_1)$ and {\sl (ii)} there is no atom-term bag $Q^\mathrm{at}_3$ of $Q$ such that $\pi_\mathcal{T}(Q^\mathrm{at}_3)$ is a descendant of of $\pi_\mathcal{T}(Q^\mathrm{at}_1)$ and $\pi_\mathcal{T}(Q^\mathrm{at}_2)$ is a descendant of of $\pi_\mathcal{T}(Q^\mathrm{at}_3)$. Note that since we only consider connected queries, the structure so created is indeed a tree (it could be a forest with disconnected queries). In what follows, we define an atom-term tree decomposition of a query by such a tree of atom-term bags, independently from $\mathcal{T}$ and $\pi$.

\begin{definition}[APT of a Query]\label{def-APT} Let $Q$ be a query. An atom-term partition of $Q$ is a partition of $\fun{atoms}(Q) \cup \fun{terms}(Q)$ (these sets being called \emph{atom-term bags}). An \emph{atom-term partition tree} (APT) of $Q$ is a tree whose nodes form an atom-term partition of $Q$.
\end{definition}

Figure \ref{fig-query-APT} represents an APT of the example query $Q_i$.

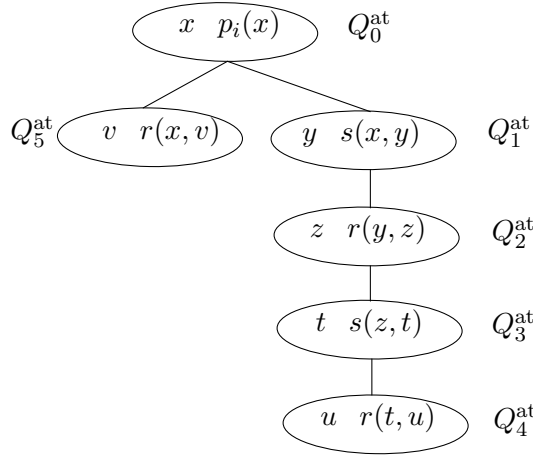
\begin{figure}
\begin{center}
\begin{tikzpicture}[line cap=round,line join=round,>=triangle 45,x=1.0cm,y=1.0cm]
\clip(-3.36,-1.02) rectangle (4.24,6.22); \draw (-0.6,6.04) node[anchor=north west] {$x \hspace{0.3cm} p_i(x)$}; \draw (1.04,4.6) node[anchor=north west] {$y \hspace{0.3cm} s(x,y)$}; \draw (1.14,3.32) node[anchor=north west] {$z \hspace{0.3cm} r(y,z)$}; \draw (1.22,2.06) node[anchor=north west] {$t \hspace{0.3cm} s(z,t)$}; \draw (1.28,0.82) node[anchor=north west] {$u \hspace{0.3cm} r(t,u)$}; \draw (-1.62,4.62) node[anchor=north west] {$v \hspace{0.3cm} r(x,v)$}; \draw [rotate around={-0.98:(0.15,5.62)}] (0.15,5.62) ellipse (1.23cm and 0.39cm); \draw [rotate around={-0.98:(-0.85,4.22)}] (-0.85,4.22) ellipse (1.23cm and 0.39cm); \draw [rotate around={-0.98:(1.99,4.18)}] (1.99,4.18) ellipse (1.23cm and 0.39cm); \draw [rotate around={-0.98:(2.03,2.9)}] (2.03,2.9) ellipse (1.23cm and 0.39cm); \draw [rotate around={-0.98:(2.07,1.66)}] (2.07,1.66) ellipse (1.23cm and 0.39cm); \draw [rotate around={-0.98:(2.15,0.4)}] (2.15,0.4) ellipse (1.23cm and 0.39cm); \draw (0.18,5.23)-- (-0.94,4.61); \draw (0.18,5.23)-- (2.08,4.56); \draw (2.08,3.79)-- (2.08,3.28); \draw (2.08,2.51)-- (2.08,2.04); \draw (2.1,1.27)-- (2.1,0.79); \draw (1.64,6.04) node[anchor=north west] {$Q^\mathrm{at}_0$}; \draw (3.5,4.64) node[anchor=north west] {$Q^\mathrm{at}_1$}; \draw (3.56,3.3) node[anchor=north west] {$Q^\mathrm{at}_2$}; \draw (3.56,2.06) node[anchor=north west] {$Q^\mathrm{at}_3$}; \draw (3.58,0.8) node[anchor=north west] {$Q^\mathrm{at}_4$}; \draw (-2.84,4.64) node[anchor=north west] {$Q^\mathrm{at}_5$};
\end{tikzpicture}
\end{center}
\caption{An atom-term partition of $Q_i = p_i(x) \wedge s(x,y) \wedge r(y,z) \wedge s(z,t) \wedge r(t,u) \wedge r(x,v)$} \label{fig-query-APT}
\end{figure}

\begin{definition}[(Valid) APT-Mapping]\label{def-APT-mapping}
\index{APT-mapping} Let $\mathcal{Q}$ be an APT of $Q$. Let $\mathcal{T}$ be a derivation tree. An \emph{APT-mapping} of $\mathcal{Q}$ to $\mathcal{T}$ is a tuple $\Gamma=(\Pi, \pi_1, \ldots, \pi_k)$ where $\Pi$ is an injective mapping from the atom-term bags of $\mathcal{Q}$ to the bags of $\mathcal{T}$ and, for each atom-term bag $Q^\mathrm{at}_i$ of $\mathcal{Q}$, $\pi_i$ is a substitution from the terms of $Q^\mathrm{at}_i$ (by this, we mean the terms of $Q^\mathrm{at}_i$, not the terms appearing in the atoms of $Q^\mathrm{at}_i$) to the terms that were created in the atoms of $\Pi(Q^\mathrm{at}_i)$.

Remark that if $u$ is a term of $Q$, then $u$ appears in only one atom-term bag $Q^\mathrm{at}_i$ of $\mathcal{Q}$. We can thus define $\pi_\Gamma = \bigcup_{1\leq i\leq k}\pi_i$.

Finally, we say that $(\Pi, \pi_1, \ldots, \pi_k)$ is \emph{valid} when $\pi_\Gamma$ is a homomorphism from $Q$ to the atoms of $\mathcal{T}$.

\end{definition}

\begin{example}[APT-Mapping]
We now present a valid APT-Mapping of the APT pictured Figure \ref{fig-query-APT} to the derivation tree represented in Figure \ref{fig-generated-tree}. We let $\Pi=\{Q^\mathrm{at}_0{\mapsto}B_0,Q^\mathrm{at}_1{\mapsto}B_4,Q^\mathrm{at}_2{\mapsto}B_5,Q^\mathrm{at}_3{\mapsto}B_6,
Q^\mathrm{at}_4{\mapsto}B_7,Q^\mathrm{at}_5{\mapsto}B_1\}$.
The corresponding mappings are:
$\pi_0=\{x{\mapsto}c\}$,
$\pi_1=\{y{\mapsto}z_1\}$,
$\pi_2=\{z{\mapsto}z_2\}$,
$\pi_3=\{t{\mapsto}z_3\}$,
$\pi_4=\{u{\mapsto}z_4\}$, and
$\pi_5=\{v{\mapsto}y_1\}$.
Then, $(\Pi,\pi_1,\pi_2,\pi_3,\pi_4,\pi_5)$ is a valid APT-mapping of the APT from Figure \ref{fig-query-APT} to the derivation tree from Figure \ref{fig-generated-tree}.
\end{example}

\begin{property}[Soundness and Completeness]\label{prop-valid-adt} Let $F$ be a fact, $\mathcal R$ be a set of \RgRbRtRsR rules, and $Q$ be a query. Then $F, \mathcal R \models Q$ if and only if there exists a derivation sequence $S$ from $F$ to $F_k$, an APT $\mathcal{Q}$ of $Q$, and a valid APT-mapping from $\mathcal{Q}$ to $\mathit{DT}(S)$.
\end{property}

\begin{proof} We successively prove both directions of the equivalence.
\begin{itemize}
    \item[$(\Leftarrow)$] Let us suppose that there exists a valid APT-mapping from $\mathcal{Q}$ to $\mathit{DT}(S)$. From Definition~\ref{def-APT-mapping}, it follows that there is a homomorphism $\pi$ from $Q$ to the atoms of $\mathit{DT}(S)$, \emph{i.e.} a homomorphism $\pi$ from $Q$ to $F_k$.

    \item[$(\Rightarrow)$] If $F, \mathcal R \models Q$, then there is a homomorphism $\pi$ from $Q$ to some $F_k$ obtained by means of a derivation $S$ from $F$. As in the construction given before Definition~\ref{def-APT}, we can choose some mapping $\pi^\mathrm{a}_\mathcal{T}$ of the atoms of $Q$, and build from this mapping an APT $\mathcal{Q}$ of $Q$. We can then build an APT-mapping $(\Pi, \pi_1, \ldots, \pi_k)$ as follows: $\Pi=\pi_\mathcal{T}$, and for each $Q^\mathrm{at}_i$ of $\mathcal{Q}$, $\pi_i$ is the restriction of $\pi$ to the terms of $Q^\mathrm{at}_i$. This APT-mapping is valid.

\end{itemize}
\end{proof}

This rather long and unnecessarily complicated way to prove the existence of a homomorphism will now be put to good use when querying the full blocked tree, without resorting to its potentially infinite development.

\subsubsection{Validation of an APT in a blocked tree}\label{sect-APT-BT}

Hence, let us now consider a blocked tree $\mathfrak{T}_b$ and some tree $(\mathfrak{T},f)\in G(\mathfrak{T}_b)$ generated by it.
Let us assume that we have an APT $\mathcal{Q}$ of a query $Q$ that corresponds to a mapping to $(\mathfrak{T},f)$. Thus, each bag $Q^\mathrm{at}_i$ of the APT is mapped to a bag $B$ of $\mathfrak{T}$. Intuitively, we represent this mapping on the full blocked tree by mapping $Q^\mathrm{at}_i$ to $B_\mathrm{rep}$ such that $B_\mathrm{rep} = f(B)$ (\emph{i.e.}, $B$ has been generated by copying $B_\mathrm{rep}$). We can enumerate all such mappings: the question is then to validate such a mapping, that is, to check that it actually corresponds to a valid APT-mapping in a tree generated by the full blocked tree.

\begin{definition}[Valid APT Mapping to a Blocked Tree]\label{DEF-APT-blocked}\label{def-valid-APT-blocked}
\index{APT-mapping} Let $\mathcal{Q}$ be an APT of a query $Q$ and $\Gamma = (\Pi, \pi_1, \ldots, \pi_k)$ be an APT-mapping from $\mathcal{Q}$ to a blocked tree $\mathfrak{T}_b$ (where $\Pi$ maps atom-term bags of $\mathcal{Q}$ to bags of the blocked tree). Then $\Gamma$ is said to be \emph{valid} if there exists a tree $(\mathfrak{T},f) \in G(\mathfrak{T}_b)$ generated from $\mathfrak{T}_b$ and a mapping $\Xi$ from the atom-term bags of $\mathcal{Q}$ to the bags of $(\mathfrak{T}, f)$ (we then call $((\mathfrak{T}, f), \Xi)$ a \emph{proof of $\Gamma$}) such that:
\begin{itemize}
    \item if $Q^\mathrm{at}$ is the root of $\mathcal{Q}$, then $\Xi(Q^\mathrm{at}) = \Pi(Q^\mathrm{at})$;
    \item if ${Q^\mathrm{at}}'$ is a child of $Q^\mathrm{at}$ in $\mathcal{Q}$, then $f(\Xi({Q^\mathrm{at}}')) = \Pi({Q^\mathrm{at}}')$ and $\Xi({Q^\mathrm{at}}')$ is a descendant of $\Xi(Q^\mathrm{at})$;
    \item
    The ADT mapping $(\Xi, \pi'_1, \ldots, \pi'_k)$ is valid in $\mathfrak{T}$, where for every atom-term bag $Q^\mathrm{at}_j$ in $\mathcal{Q}$ with $\Xi(Q^\mathrm{at}_j) = B_i$, we define $\pi'_j = \psi_{f(B_i)\to B_i} \circ \pi_j$.
\end{itemize}
\end{definition}

\begin{example}[APT-Mapping to a Blocked Tree]
We now present a valid APT-Mapping of the APT represented Figure \ref{fig-query-APT} to the derivation tree represented in Figure \ref{fig-generated-tree}. We define $\Pi =
\{
Q^\mathrm{at}_0{\mapsto}B_0,
Q^\mathrm{at}_1{\mapsto}B_4,
Q^\mathrm{at}_2{\mapsto}B_5,
Q^\mathrm{at}_3{\mapsto}B_6,
Q^\mathrm{at}_4{\mapsto}B_3,
Q^\mathrm{at}_5{\mapsto}B_1\}$.
Here, the only difference with the previous APT-mapping is the image of $Q^\mathrm{at}_4$, which is not $B_7$ (which does not exist in the blocked tree), but $B_3$. This is reflected in the definition of the $\pi_i$:
$\pi_0=\{x{\mapsto}c\}$,
$\pi_1=\{y{\mapsto}z_1\}$,
$\pi_2=\{z{\mapsto}z_2\}$,
$\pi_3=\{t{\mapsto}z_3\}$,
$\pi_4=\{u{\mapsto}y_2\}$, and
$\pi_5=\{v{\mapsto}y_1\}$.
Then, $(\Pi,\pi_1,\pi_2,\pi_3,\pi_4,\pi_5)$ is a valid APT-mapping of the APT from Figure \ref{fig-query-APT} to the blocked tree from Figure \ref{fig-incomplete-querying}, as witnessed by the derivation tree of Figure \ref{fig-generated-tree}, where the bag $B_7$ has been generated by $B_3$.
\end{example}

\begin{property}[Soundness and Completeness]\label{Prop-APT-blocked} Let $F$ be a fact, $\mathcal R$ be a set of \RgRbRtRsR rules, and $Q$ be a query. Then $F, \mathcal R \models Q$ if and only if there exists an APT $\mathcal{Q}$ of $Q$, and a valid APT-mapping from $\mathcal{Q}$ to the full blocked tree of $F$ and $\mathcal R$.
\end{property}

\begin{proof} We successively prove both directions of the equivalence.
\begin{itemize}
    \item[$(\Leftarrow)$] Let us suppose that there exists a valid APT-mapping from $\mathcal{Q}$ to the full blocked tree $\mathfrak{T}_b$ of $F$ and $\mathcal R$. From Definition~\ref{def-valid-APT-blocked}, there exists a valid APT mapping from $\mathcal{Q}$ to a $(\mathfrak{T}, f)$ generated from $\mathfrak{T}_b$, \emph{i.e.}, by Definition \ref{def-full-blocked-tree} a valid APT-mapping from  $\mathcal{Q}$ to some derivation tree $\mathcal{T}$ having $\mathfrak{T}$ as a prefix. We can conclude thanks to Property~\ref{prop-valid-adt}.

    \item[$(\Rightarrow)$] If $F, \mathcal R \models Q$, then there is a homomorphism $\pi$ from $Q$ to some $F_k$ obtained by means of a
        derivation $S$ from $F$ with derivation tree $\mathcal{T}=DT(S)$. As in the construction given before Definition~\ref{def-APT}, we can chose some mapping $\pi^\mathrm{a}_\mathcal{T}$ of the atoms of $Q$, and build from this mapping an APT $\mathcal{Q}$ of $Q$. Now, in this particular $\pi$, the root of $\mathcal{Q}$ can be mapped to any bag $B$ of the derivation tree $\mathit{DT}(S)$. Since $B$ has an equivalent bag $B'$ in the full blocked tree $\mathfrak{T}_b$, there exists another homomorphism $\pi'$ from $Q$ to some $F'_k$ obtained by means of a derivation $S'$. Let us recompute an APT $\mathcal{Q}'$ (the same result might be obtained). Then (see proof of Property~\ref{prop-valid-adt}), there is a valid APT mapping $\Gamma = (\Pi, \pi_1, \ldots, \pi_k)$ from $\mathcal{Q}'$ to $DT(S')$, since $\mathit{DT}(S)$ is a prefix tree of some $\mathfrak{T}$ generated from $\mathfrak{T}_b$. $\Gamma$ is thus a valid APT mapping from $\mathcal{Q}$ to $\mathfrak{T}$.

        Now let us define the mapping $\Xi$ as follows: if $Q^\mathrm{at}$ is the root of $\mathcal{Q}$, then $\Xi(Q^\mathrm{at}) = \Pi(Q^\mathrm{at})$, otherwise $\Xi(Q^\mathrm{at}) = f(\Pi(Q^\mathrm{at}))$. For each term $t$ in the atom term bag $Q^\mathrm{at}_i$ of $\mathcal{Q}$ such that $\Xi(Q^\mathrm{at}_i) = B_j$, we define $\pi'_i(t) = \psi_{B_j\to f(B_j)} \circ \pi_i$. Let us consider the APT mapping $\Gamma' = (\Xi, \pi'_1, \ldots, \pi'_k)$ from $\mathcal{Q}$ to $\mathfrak{T}_b$. It is immediate to check that $\Gamma'$ is valid.
\end{itemize}
\end{proof}

\subsubsection{A bounded validation for APT-mappings}\label{sect-APT-BT-bounded}

Although Property~\ref{Prop-APT-blocked} brings us closer to our goal to obtain an algorithm for \RgRbRtRsR deduction, there is still the need to guess the generated tree $(\mathfrak{T}, f)$ used to validate an APT-mapping (Definition~\ref{DEF-APT-blocked}). We will now show that such a generated tree can be built in a backtrack-free manner by an exploration of the APT of the query. Then we establish an upper bound for each validation step (\emph{i.e.}, if we have validated an initial segment $\mathcal{Q}'$ of the APT $\mathcal{Q}$ of $Q$, and ${Q^\mathrm{at}}'$ is a child of some bag ${Q^\mathrm{at}}$ in $\mathcal{Q}'$, how do we validate $\mathcal \mathcal{Q}' \cup \{{Q^\mathrm{at}}'\}$?).

\begin{property}
\label{prop-apt-restriction} Let $\mathcal{Q}$ be an APT of $Q$, and $\Gamma$ be a valid APT-mapping from $\mathcal{Q}$ to a blocked tree $\mathfrak{T}_b$. Let $\mathcal{Q}'$ be a prefix tree of $\mathcal{Q}$, and $\Gamma'$ be the restriction of $\Gamma$ to $\mathcal{Q}'$. Note that $\Gamma'$ is a valid APT-mapping from $\mathcal{Q}'$ to $\mathfrak{T}_b$.

Consider now any proof $((\mathfrak{T}', f'), \Xi')$ of $\Gamma'$ (see Definition~\ref{DEF-APT-blocked}).\footnote{Note that this proof $((\mathfrak{T}', f'), \Xi')$ is not necessarily a subproof of the existing proof $((\mathfrak{T}, f), \Xi)$ of $\Gamma$ (\emph{i.e.}, $(\mathfrak{T}', f')$ is not necessarily an initial segment of $(\mathfrak{T}, f)$ and $\Xi'$ is not necessarily the restriction of $\Xi$).} Then, there exists a proof $((\mathfrak{T}'', f''), \Xi'')$ of $\Gamma$ such that $((\mathfrak{T}', f'), \Xi')$ is a subproof of $((\mathfrak{T}'', f''), \Xi'')$.
\end{property}

\begin{proof} Let us consider a proof $((\mathfrak{T}', f'), \Xi')$ of $\Gamma'$. As shown in the proof of Property~\ref{Prop-APT-blocked}, this proof corresponds to a homomorphism $\pi'$ of $Q'$ to $(\mathfrak{T}', f')$. Now consider any leaf bag ${Q^\mathrm{at}}'$ of $\mathcal{Q}'$, that is the root of a tree in $\mathcal{Q}$. $\Gamma$ and $\Gamma'$ can map ${Q^\mathrm{at}}'$ to different bags in $(\mathfrak{T}', f')$ and $(\mathfrak{T}, f)$. However, these bags are equivalent (according to Definition~\ref{def-pattern-inclusion}) to the same bag in $\mathfrak{T}_b$. So anything that can be mapped under one of these bags can be mapped in the same way under the other. In particular, the subtree rooted in ${Q^\mathrm{at}}'$ can be mapped in the same way under the bag of $(\mathfrak{T}', f')$. This construction leads to a homomorphism $\pi''$ from $Q$ to some $(\mathfrak{T}'', f'')$ that extends $\pi'$. Using again the proof of Property~\ref{Prop-APT-blocked}, this homomorphism can be used to build a proof $((\mathfrak{T}'', f''), \Xi'')$ of $\Gamma$, that is a superproof of $((\mathfrak{T}', f'), \Xi')$.
\end{proof}

This latter property provides us with a backtrack-free algorithm for checking the validity of an APT-mapping. Basically, Algorithm~\ref{ALGOvalidateAPT} performs a traversal of the APT $\mathcal{Q}$, while verifying whether $\Gamma$ ``correctly joins'' each bag of $\mathcal{Q}$ with its already ``correctly joined'' parent.

\begin{algorithm}[ht]
\KwData{A blocked tree $\mathfrak{T}_b$, an APT $\mathcal{Q}$, and an APT-mapping $\Gamma$ from $\mathcal{Q}$ to $\mathfrak{T}_b$.} \KwResult{{\sc yes} if $\Gamma$ is valid, {\sc no} otherwise.} Explored $:= \emptyset$\; \For{$i=1$ to $|\mathcal{Q}|$}
    {$Q^\mathrm{at} :=$ some bag of $\mathcal{Q}$ s.t. either $(${\sl parent}$(Q^\mathrm{at}), -) \in$ Explored, or $Q^\mathrm{at}$ is the root of $\mathcal{Q}$\;
    \eIf {{\sl joins}$(\Gamma, Q^\mathrm{at})\not=\emptyset$}
        {Explored $:=$ Explored $\cup \{(Q^\mathrm{at}, \mbox{\sl joins}(\Gamma, Q^\mathrm{at}))\}$\;}
        {\Return {\sc no}\;}
    }
\Return{\sc yes}\; \caption{{\sc ValidateAPT}} \label{ALGOvalidateAPT}
\end{algorithm}

It remains now to explain the procedure {\sl joins} that checks whether a valid APT-mapping of a subtree $\mathcal{Q}'$ of $\mathcal{Q}$ can be extended to  a child ${Q^\mathrm{at}}\in \mathcal{Q}\setminus \mathcal{Q}'$ of some bag of $\mathcal{Q}'$. Let us consider a proof $((\mathfrak{T}', f'), \Xi')$ of $\Gamma=(\Pi, \pi_1, \ldots, \pi_k)$ being a valid APT-mapping of $\mathcal{Q}'$. According to Def.~\ref{DEF-APT-blocked} and Prop.~\ref{Prop-APT-blocked}, it is sufficient to:
 \begin{itemize}
    \item find a bag $B_n$ that can be obtained by a \emph{bag copy sequence} $B_1, \ldots, B_n$ where $B_1 = \Xi'(\mbox{parent}({Q^\mathrm{at}}))$, $B_n$ is a bag equivalent to $\Pi(Q^\mathrm{at})$, and for $1 < i \leq n$, $B_i$ is obtained by a bag copy (see Definition ~\ref{def-bag-copy}) under $B_{i-1}$. Since $\Pi(Q^\mathrm{at})$ and $B_n$ are equivalent, there is a bijection from the terms of $\Pi(Q^\mathrm{at})$ to the terms of $B_n$, that we denote by $\psi$.
    \item it remains now to check that for every term $t$ appearing in an atom of ${Q^\mathrm{at}}$, $t$ is a term that belongs to a bag ${Q^\mathrm{at}}'$ in the branch from the root of $\mathcal{Q}$ to ${Q^\mathrm{at}}$, and $\Xi'(t) = \psi(\pi(t))$ (where $\pi$ is the mapping defined in the APT-mapping from the terms of ${Q^\mathrm{at}}$ to those of $\Pi({Q^\mathrm{at}})$). In that case, the call to {\sl joins} returns $\psi \circ \pi$, ensuring that we are able to evaluate the joins of the next bags.
\end{itemize}

We last prove that there exists a ``short'' proof of every valid APT-mapping.
\begin{property}
\label{prop-size-proof}
Let $\mathcal Q$ be an APT of $Q$, and $\Gamma$ be a valid APT-mapping from $\mathcal Q$ to a blocked tree $\mathfrak{T}_b$. There exists a proof $((\mathfrak T,f),\Xi)$ of $\Gamma$ such that for any two bags $Q^\mathrm{at}_i$ and $Q^\mathrm{at}_j$ from $\mathcal Q$ where $Q^\mathrm{at}_i$ is a child of $Q^\mathrm{at}_j$, the distance between $\Xi(Q\mathrm{at}_i)$ and $\Xi(Q\mathrm{at}_j)$ in $\mathfrak{T}$ is at most $p\times f^f$, where $p$ is the number of abstract patterns, and $f$ the maximum size of a rule frontier.
\end{property}

\begin{proof}
We prove the result by induction on the number of atom-term bags in $\mathcal Q$. If $\mathcal Q$ has only one atom-term bag, the property is trivially true. Let us assume the result to be true for any APT-mapping whose APT is of size $n$, and let $\mathcal Q$ be a  APT of size $n+1$, and $\Gamma$ be a valid APT-mapping of $\mathcal Q$.  Let $Q^\mathrm{at}_c$ be an arbitrary leaf of $\mathcal Q$, let $\mathcal Q'$ be equal to $\mathcal Q \setminus \{Q_c\}$, and let $\Gamma'$ be the restriction of $\Gamma$ to $\mathcal Q'$. By induction assumption, there exists a proof $((\mathfrak T',f'),\Xi')$ of $\Gamma'$ fulfilling the claimed property. By the proof of Property \ref{prop-apt-restriction}, this proof can be extended to a proof $((\mathfrak T,f),\Xi)$, such that $Q^\mathrm{at}_c$ is mapped to a descendant of $\Xi(Q^\mathrm{at}_p)$, where $Q^\mathrm{at}_p$ is the parent of $Q^\mathrm{at}_c$ in $\mathcal Q$. Moreover, the provided construction allows to ensure that there are no two distinct bags $B_i$ and $B_j$ on the path from $\Xi(Q^\mathrm{at}_p)$ to $\Xi(Q^\mathrm{at}_c)$  fulfilling the following two conditions:
 \begin{itemize}
\item $f(B_i) = f(B_j)$;
\item $\forall x \in X_s, \psi_{B_i\to f(B_i)}(x) = \psi_{B_j\to f(B_j)}(x)$, where $X_s$ is the image of the set of variables that appear both in $Q^\mathrm{at}_c$ and $\mathcal Q'$.
\end{itemize}
Then the distance between $\Xi(Q^\mathrm{at}_p)$ and $\Xi(Q^\mathrm{at}_c)$ has to be less than $p \times f^f$. Indeed, $X_s$ should belong to the frontier of $\Xi(Q^\mathrm{at}_c)$, by the running intersection property of $\mathfrak T$. There is at most $f^f$ ways of arranging these terms, thus providing the claimed upper-bound.
\end{proof}

%

%% file: subclasses-kr.tex

We provide a worst-case complexity analysis of the proposed algorithm, and present some small modifications that can be adopted in order to make the algorithm worst-case optimal for relevant subclasses of rules. Table \ref{tab-notation-complexity} provides a summary of the notation used for the complexity analysis.

\begin{table}[h]
\caption{Notations used in the complexity analysis}
\label{tab-notation-complexity}
\begin{center}
\begin{tabular}{|c|c|}
\hline
Notation & Signification \\ 
\hline
$b$ & upper-bound on the number of terms in any abstract pattern\\
\hline
$q$ & number of atoms plus number of terms in the query\\
\hline
$f$ & maximum size of a rule frontier\\
\hline
$a_B$ & maximum number of atoms in a rule body\\
\hline
$a_H$ & maximum number of atoms in a rule head\\
\hline
$t_B$ & maximum number of terms in a rule body\\
\hline
$t_H$ & maximum number of terms in a rule head \\
\hline
$|\mathcal R|$ & number of rules \\
\hline
$p$ & number of abstract patterns \\
\hline
$w$ & maximum arity of a predicate\\
\hline
$s$ & number of predicates appearing in the rule set\\
\hline
\end{tabular}
\end{center}
\end{table}

\subsubsection{Complexity of the algorithm}

The overall algorithm deciding whether $F, \mathcal R \models Q$ can now be sketched as follows:
 \begin{itemize}
    \item build the full blocked tree $\mathfrak T_b$ of $(F, \mathcal R)$ (note that this is done independently of the query)
    \item for every APT $\mathcal{Q}$ of $Q$, for every APT-mapping $\Gamma$ of $\mathcal{Q}$ to $\mathfrak T_b$, if {\sc ValidateAPT} returns {\sc yes}, return {\sc yes} (and return {\sc no} at the end otherwise).
 \end{itemize}


 The first step is done linearly in the number of evolution and creation rules. There are at most $p^2$ evolution rules, and $p^2\times b^f$ creation rules. We thus need to upper-bound the number of abstract patterns. There are at most $|\mathcal R| \times2^{a_B}$ subsets of rule bodies, and $b^{t_B}$ mappings from terms of a rule body to terms of a bag. An abstract pattern being a subset of the cartesian product of these two sets, the number of abstract patterns is upper-bounded by

 $$2^{|\mathcal R| \times 2^{a_B}\times b^{t_B}}.$$
 The first step is thus done in double exponential time, which drops to a single exponential when the set of rules is fixed. Note that only this first step is needed in the first version of the algorithm, where the query was considered as a rule, which yields the proof of Theorem \ref{thm-easy-upper-bound}.

 The second step can be done in $N_Q \times N_\Gamma \times N_V$ where:
 \begin{itemize}
    \item $N_Q$ is the number of APTs of a query $Q$ of size $q$, and $N_Q = \mathcal O(q^q)$ (the number of partitions on the atoms and terms of $Q$, times the number of trees that can be built on each of these partitions);
    \item $N_\Gamma$ is the number of APT mappings from one APT (of size $q$) to the full blocked tree. The size of the full blocked tree is $\mathcal O(pb^{f})$ and thus $N_\Gamma = \mathcal O(p^q \times b^{fq})$;
    \item $N_V$ is the cost of Algorithm~\ref{ALGOvalidateAPT} that evaluates the validity of the APT. It performs at most $q$ {\sl joins}, and each one generates at most $\mathcal O(p \times f^f)$ bags (see Property~\ref{prop-size-proof}).
 \end{itemize}

The second step of our algorithm thus operates in $\mathcal O(q^q\times p^q\times b^{fq}\times q \times p \times f^f)$.

The querying part is thus polynomial in $p$ (the number of patterns), and simply exponential in $q$ and in $f$. Since $p$ is in the worst-case double exponential w.r.t.\ $F$ and $\mathcal R$, the algorithm runs in \ExpExpTime{}. Last, given a (nondeterministically guessed) proof of $\Gamma$, we can check in polynomial time (if $\mathcal R$ and $F$ are fixed) that it is indeed a valid one, yielding:

\begin{theorem}
CQ entailment for \RgRbRtRsR is {\sc NP}-complete for query complexity.
\end{theorem}

Thereby, the lower bound comes from the well-known {\sc NP}-complete query complexity of plain (i.e., rule-free) CQ entailment.

\subsubsection{Adaptation to relevant subclasses}

We now show how to adapt the algorithm to the subclasses of \RgRbRtRsR that have smaller worst-case complexities. This is done by slightly modifying the construction of the full blocked tree, allowing its size to be simply exponential or even polynomial with respect to the relevant parameters. We consider three cases:
\begin{itemize}
\item (weakly) guarded rules, whose combined complexity in the bounded arity case drops to \ExpTime{},
\item guarded frontier-one rules, whose combined complexity in the unbounded arity case drops to \ExpTime{},
\item guarded, frontier-guarded and frontier-1 rules, whose data complexity drops down to \PTime{}.
\end{itemize}

For weakly guarded rules, we change the definition of pattern. Indeed, with this kind of rules, a rule application necessarily maps all the terms of a rule body to terms occurring in a single bag. This holds since every initial term belongs to every bag, and every variable of the rule body that could map to an existentially quantified variable is argument of the guard of the body of the rule. Thus, by storing all the possible mappings of a rule body atom (instead of all partial homomorphisms of a subset of a rule body), we are able to construct any homomorphism from a rule body to the current fact. The equivalence between patterns and the blocking procedure remains unchanged. Since there are at most $b^{w}$ such homomorphisms for an atom, the number of abstract patterns is bounded by $2^{b^{w}}$, which is a simple exponential since $w$ is bounded. The algorithm thus runs in exponential time for weakly guarded rules with bounded arity. 

If we consider only guarded frontier-one rules, the number of possible homomorphisms decreases. Indeed, the set of atoms whose terms are included in a given bag is upper-bounded by $a_H + s.t_H$: the atoms that have been created at the creation of the bag, plus atoms that may be created afterwards. However, these atoms must be of the from $r(x,\ldots,x)$, since the considered rules are frontier-one, hence at most $s.t_H$. This results in a simply exponential number of patterns, which provides the claimed upper-bound.

For frontier-guarded rules (and its subclasses), we slightly modify the construction of the decomposition tree. Indeed, in the original construction, every term of the initial fact is put in every bag of the decomposition tree. However, by putting only the constants appearing in a rule head as well as every instantiation of terms of the head of the rule creating the bag (that is, we do not put all initial terms in every bag), a correct tree decomposition would also be built, and the size of the bags (except for the root) would not be dependent of the initial fact any more. The number of patterns is then upper-bounded by $1+ 2^{|\mathcal R| \times 2^{a_B} \times t_H^{t_B}}$. When $\mathcal R$ is fixed, this number is polynomial in the data. Given that $Q$ is fixed, we get the \PTime{} upper-bound.

%% file: sec-fg.tex

We now provide the missing hardness results to fully substantiate
all complexity results displayed in Figure~\ref{complexities-min} and
Table~\ref{table:complexity}.

\subsection{Data Complexity of Guarded Frontier-One Rules is \PTime-hard}

\PTime{} hardness for \RgRfRrRoR rules is not hard to establish and
follows from known results (for instance, the description logic
$\mathcal{EL}$ is subsumed by \RgRfRrRoR rules and known to have
\PTime-hard data complexity). For the sake of self-containedness, we
will give a direct reduction from one of the prototypical \PTime{}
problems: entailment in propositional Horn logic.

\begin{theorem}
CQ entailment under constant-free \RgRfRrRoR rules is \PTime{}-hard
for data complexity.
\end{theorem}

\newcommand{\predstyle}[1]{\mathit{#1}}
\begin{proof}
Given a set $\mathcal{H}$ of propositional Horn clauses, we
introduce for every propositional atom $a$ occurring therein a
constant $c_a$. We also introduce one additional constant
$\mathit{nil}$. Moreover, for every Horn clause $C\in\mathcal{H}$
with $C = \mathit{a}_1\wedge\ldots\wedge\mathit{a}_n\to \mathit{a}$,
we introduce constants $\mathit{b}_{C,1}, \ldots, \mathit{b}_{C,n}$
and let $F$ consist of
$\predstyle{entails}(\mathit{b}_{C,1},c_\mathit{a})$,
$\predstyle{first}(\mathit{b}_{C,i},c_{\mathit{a}_i})$ for all $i\in
\{1,\ldots, n\}$, as well as
$\predstyle{rest}(\mathit{b}_{C,n},\mathit{nil})$, and
$\predstyle{rest}(\mathit{b}_{C,i},\mathit{b}_{C,i+1})$ for all
$i\in \{1,\ldots, n-1\}$, also let $F$ contain
$\predstyle{entailed}(\mathit{nil})$. Then the propositional atom
$a$ is entailed by $\mathcal{H}$ exactly if $F,\mathcal{R} \models
Q$ with $Q=entailed(c_a)$ and $\mathcal{R}$ containing the rules:
\[
\begin{array}{rrl}
\predstyle{first}(y,z) \wedge \predstyle{entailed}(z) \wedge \predstyle{rest}(y,z')\wedge \predstyle{entailed}(z') & \to & \predstyle{entailed}(y),\\
\predstyle{entailed}(y)\wedge \predstyle{entails}(y,z) & \to & \predstyle{entailed}(z).\\
\end{array}
\]
\end{proof}

\subsection{Combined Complexity of Guarded Frontier-One Rules is
\ExpTime-hard}\label{sec:HornALC}

We prove \ExpTime-hardness of CQ entailment under \RgRfRrRoR rules
by showing that any standard reasoning task in the description logic
 Role-Bounded Horn-$\mathcal{ALC}$, for which
\ExpTime-hardness is known \cite{KRH:hornAAAI,KRH:HornDLs2013}, can be polynomially reduced
to the considered problem.

We start by defining this problem. In order to avoid syntactic
overload, we will stick to first-order logic syntax and refrain from
using the traditional description-logic-style notation.

\begin{definition}[Role-Bounded Horn-$\mathcal{ALC}$]
Let $\mathsf{Pr}_1$ an infinite set of unary predicates and let
$\mathsf{Pr}_2$ be a finite set of binary predicates. A
\emph{reduced normalized Horn-$\mathcal{ALC}[\mathsf{Pr}_2]$
terminology} $\mathcal{T}$ is a set of rules having one of the
following shapes (with $p_1,p_2,p_3 \in \mathsf{Pr}_1$ and $r \in
\mathsf{Pr}_2)$:
\begin{enumerate}
\item[(A)] $p_1(x) \to p_2(x)$
\item[(B)] $p_1(x) \wedge p_2(x) \to p_3(x)$
\item[(C)] $r(x,y) \wedge p_1(y) \to p_2(x)$
\item[(D)] $p_1(x) \wedge r(x,y) \to p_2(y)$
\item[(E)] $p_1(x) \to \exists y.(r(x,y) \wedge p_2(y))$
\end{enumerate}
We refer to the problem of deciding if for some $\mathcal{T}$ and
$p_1,p_2\in \mathsf{Pr}_1$ holds $\mathcal{T}\models \forall
x(p_1(x) \to p_2(x))$ as \emph{unary subsumption checking}.
\end{definition}

\begin{theorem}[\cite{KRH:HornDLs2013}, Section
6.2.]\label{thm-HornALC} There is a finite set $\mathsf{Pr}_2$ such
that unary subsumption checking for reduced normalized
Horn-$\mathcal{ALC}[\mathsf{Pr}_2]$ terminologies is \ExpTime-hard.
\end{theorem}

The following corollary is now a straightforward consequence.

\begin{corollary}
CQ entailment under constant-free \RgRfRrRoR rules is \ExpTime-hard
for combined complexity.
\end{corollary}

\begin{proof}
Clearly, given a reduced normalized
Horn-$\mathcal{ALC}[\mathsf{Pr}_2]$ terminology $\mathcal{T}$ and
$p_1,p_2\in \mathsf{Pr}_1$, the unary subsumption entailment
$\mathcal{T}\models \forall x(p_1(x) \to p_2(x))$ is to be confirmed
if and only if $F,\mathcal{R}\models Q$ with $F=\{p_1(a)\}$,
$\mathcal{R}=\mathcal{T}$ and $Q=p_2(a)$. The latter can be
conceived as a CQ entailment problem of the desired type, since
$\mathcal{T}$ is a \RgRfRrRoR rule set.
\end{proof}

Note that our line of argumentation actually does not require the
set $\mathsf{Pr}_2$ to be fixed, however, this will be a necessary
precondition in the next section where we use the same logic, hence
we have introduced the logic in this form from the beginning.

\subsection{Data Complexity of Weakly Guarded Frontier-One Rules is \ExpTime-hard}

We will now show that the data complexity for deciding CQ entailment
under \RwRgRfRrRoR rules is \ExpTime-hard. Again, we obtain the
result by showing that a Horn-$\mathcal{ALC}[\mathsf{Pr}_2]$
reasoning problem can be reduced to CQ entailment under \RwRgRfRrRoR
rules but this time even with fixed rule set and query.

\begin{definition}
Given a set $\mathsf{Pr}_2$, let $\mathcal{R}_\mathrm{fix}$ be the
fixed $\RwRgRfRrRoR$ rule set containing the following rules (where
$r$ ranges over all elements of $\mathsf{Pr}_2$):
\begin{enumerate}
\item $\mathit{typeA}(z_1,z_2) \wedge \mathit{in}(x,z_1) \to \mathit{in}(x,z_2)$
\item $\mathit{typeB}(z_1,z_2,z_3) \wedge \mathit{in}(x,z_1) \wedge \mathit{in}(x,z_2) \to \mathit{in}(x,z_3)$
\item $\mathit{typeC}_r(z_1,z_2) \wedge r(x,y) \wedge \mathit{in}(y,z_1) \to \mathit{in}(x,z_2)$
\item $\mathit{typeD}_r(z_1,z_2) \wedge \mathit{in}(x,z_1) \wedge r(x,y) \to \mathit{in}(y,z_2)$
\item $\mathit{typeE}_r(z_1,z_2) \wedge \mathit{in}(x,z_1) \to r(x,y) \wedge \mathit{in}(y,z_2)$
\item $\mathit{test}(x) \wedge \mathit{sub}(z) \to \mathit{in}(x,z)$
\item $\mathit{test}(x) \wedge \mathit{in}(x,z) \wedge \mathit{super}(z) \to
\mathit{match}$
\end{enumerate}
Given a reduced normalized Horn-$\mathcal{ALC}[\mathsf{Pr}_2]$
terminology $\mathcal{T}$, we let $F_\mathcal{T}$ be the fact
containing
\begin{enumerate}
\item $\mathit{typeA}(c_{p_1},c_{p_2})$ for any $R = p_1(x) {\to} p_2(x)$ from
$\mathcal{T}$,
\item $\mathit{typeB}(c_{p_1},c_{p_2},c_{p_3})$  for any $R = p_1(x) {\wedge} p_2(x) {\to}
p_3(x)$ from $\mathcal{T}$,
\item $\mathit{typeC}_r(c_{p_1},c_{p_2})$ for any $R = r(x,y) {\wedge} p_1(y) {\to}
p_2(x)$ from $\mathcal{T}$,
\item $\mathit{typeD}_r(c_{p_1},c_{p_2})$  for any $R= p_1(x) {\wedge} r(x,y) {\to}
p_2(y)$ from $\mathcal{T}$, and
\item $\mathit{typeE}_r(c_{p_1},c_{p_2})$  for any $R = p_1(x) {\to} \exists
y.(r(x,y) {\wedge} p_2(y))$ from $\mathcal{T}$.
\end{enumerate}
Also, for $p_1,p_2 \in \mathsf{Pr}_1$, we let $F_{p_1p_2} =
\{\mathit{test}(a),\mathit{sub}(c_{p_1}),\mathit{super}(c_{p_2})\}$
\end{definition}

The following property now lists two straightforward observations.

\begin{property}\label{prop-HornALC-trafo}
$\mathcal{R}_\mathrm{fix}$ is a $\RwRgRfRrRoR$ rule set.
$F_\mathcal{T}$ can be computed in polynomial time and is of
polynomial size with respect to $\mathcal{T}$.
\end{property}

The next lemma establishes that subsumption in
Horn-$\mathcal{ALC}[\mathsf{Pr}_2]$ can be reduced to $CQ$
entailment w.r.t.~a fixed $\RwRgRfRrRoR$ rule set.

\begin{lemma}\label{lem-HornALC-to-wgfr1}
Let $\mathcal{T}$ be a reduced normalized
Horn-$\mathcal{ALC}[\mathsf{Pr}_2]$ terminology and let $p_1,p_2 \in
\mathsf{Pr}_1$. Then $\mathcal{T}\models \forall x(p_1(x) \to
p_2(x))$ if and only if $F_\mathcal{T}\cup
F_{p_1p_2},\mathcal{R}_\mathrm{fix}\models \mathit{match}$.
\end{lemma}

\begin{proof}
We successively prove both directions of the equivalence. 
\begin{itemize}
\item[$\Rightarrow$] Assume to the contrary that
$F_\mathcal{T}\cup F_{p_1p_2},\mathcal{R}_\mathrm{fix}\models
\mathit{match}$ does not hold, thus we find a model of
$F_\mathcal{T}\cup F_{p_1p_2},\mathcal{R}_\mathrm{fix}$ where
$\mathit{match}$ is false. We then can use this model to construct a
model of $\mathcal{T}$ not satisfying $\forall x(p_1(x) \to p_2(x))$
using the following defining equation for every $p_i\in
\mathsf{Pr}_1$: $$\forall x (p_i(x) \leftrightarrow
\mathit{in}(x,c_{p_i})).$$ It can be readily checked that this model
indeed satisfies $\mathcal{T}$. Moreover it satisfies $p_1(a)$ but
not $p_2(a)$, therefore $\forall x(p_1(x) \to p_2(x))$ does not
hold.

\item[$\Leftarrow$] Assume to the contrary that
$\mathcal{T}\models \forall x(p_1(x) \to p_2(x))$ does not hold,
i.e. we find a model of $\mathcal{T}$ not satisfying $\forall
x(p_1(x) \to p_2(x))$ (i.e., in this model there must be one element
$e$ in the extension of $p_1$ but not of $p_2$). We now use this
model to define a model of $F_\mathcal{T}\cup
F_{p_1p_2},\mathcal{R}_\mathrm{fix}$ by adding to the domain a new
element $c_{p_i}$ for every $p_i\in \mathsf{Pr}_1$ and let the
interpretation function on the eponymous constants be the identity.
The interpretation of the predicates
$\mathit{typeA},\ldots,\mathit{typeE}$ is as explicitly stated in
$F_\mathcal{T}$, the interpretation of the $\mathit{in}$ predicate
is defined as the minimal set such that $$\forall x (p_i(x)
\leftrightarrow \mathit{in}(x,c_{p_i}))$$ holds for all $p_i\in
\mathsf{Pr}_1$. Finally, we let $\mathit{sub}$ hold for $c_{p_1}$,
we let $\mathit{super}$ hold for $c_{p_2}$, and we let
$\mathit{test}$ hold for $a$, where the constant $a$ is mapped to
the domain element $e$ mentioned above. Then it can be easily
checked that the obtained model satisfies all of $F_\mathcal{T}\cup
F_{p_1p_2},\mathcal{R}_\mathrm{fix}$ but not $\mathit{match}$.
\end{itemize}
\end{proof}

\begin{theorem}[Data Complexity of $\RwRgRfRrRoR$ Rules]
\textsc{BCQ-Entailment} under constant-free $\RwRgRfRrRoR$
rules is \ExpTime-hard for data complexity.
\end{theorem}

\begin{proof}
By Theorem~\ref{thm-HornALC}, unary subsumption checking for reduced
normalized Horn-$\mathcal{ALC}[\mathsf{Pr}_2]$ terminologies is
\ExpTime-hard. By Lemma~\ref{lem-HornALC-to-wgfr1} and thanks to
Proposition~\ref{prop-HornALC-trafo}, any such problem can be
polynomially reduced to a CQ entailment problem
$F,\mathcal{R}\models q$ for uniform $\mathcal{R}$ and $q$.
Consequently, the data complexity of conjunctive query entailment
under $\RwRgRfRrRoR$ rules must be \ExpTime-hard as well.
\end{proof}

\subsection{Combined Complexity of Frontier-One Rules is
\ExpExpTime-hard}


In this section, we show that \RfRrRoR rules are \ExpExpTime{}-hard
for combined complexity no matter whether predicate arity is bounded
or not. Our proof reuses the general strategy and many technical
tricks from a construction used to show \ExpExpTime{}-hardness for
CQ entailment in the DL $\mathcal{ALCI}$ from \citep{Lutz-DL-07}.
Still, many adaptations were done in order to make the construction
fit our language fragment and to simplify unnecessarily complicated
parts.

We prove the desired result via a reduction of CQ entailment w.r.t.
\RfRrRoR rules to the word problem of an alternating Turing machine
with exponential space, which will be formally introduced next.

\begin{definition}
An \emph{alternating Turing machine} $\mathcal{M}$, short ATM, is a
quadruple $(\mathfrak{Q},\Gamma,q_0,\Delta)$ where:
\begin{itemize}
\item $\mathfrak{Q}$ is the set of \emph{states}\footnote{$\mathfrak{Q}$ (with possible subscripts) is used for state sets in order to avoid a notational clash with conjunctive queries denoted by $Q$.}, which can be partitioned into
existential states $\mathfrak{Q}_\exists$ and universal states
$\mathfrak{Q}_\forall$.
\item $\Gamma$ is a finite \emph{alphabet}, containing the \emph{blank
symbol} $\Box$,
\item $q_0\in \mathfrak{Q}$ is the \emph{initial state},
\item $\Delta \subseteq \mathfrak{Q}\times \Sigma \times \mathfrak{Q} \times \Sigma
\times \{L,R\}$ is the \emph{transition relation}.
\end{itemize}

A \emph{configuration} of an ATM is a word $wqw'$ where $w,w'\in
\Gamma^*$ and $q\in \mathfrak{Q}$. The \emph{successor
configurations} $\Delta(wqw')$ of a configuration $wqw'$ are defined as:

$$
\begin{array}{ll}

\Delta(wqw') = & \{ vq'\gamma''\gamma'v' \mid w=v\gamma'', w'=\gamma
v', (\gamma,q,\gamma',q',L) \in \Delta \}\\ & \ \ \cup\ \{
w\gamma'q'v' \mid w'=\gamma v', (\gamma,q,\gamma',q',R) \in \Delta \}\\
& \ \ \cup\ \{ w\gamma'q' \mid w'=\epsilon, (\Box,q,\gamma',q',R) \in
\Delta \}.
\end{array}
$$

The set of \emph{indirect successors} of a configuration $wqw'$ is
the smallest set of configurations that contains $wqw'$ and that is
closed under the successor relation.

A \emph{halting configuration} is of the form $wqw'$ with
$\Delta(wqw')=\emptyset$. The set of \emph{accepting configurations}
is the smallest set of configurations such that:
\begin{itemize}
\item
$wqw'$ is accepting if there exists $vq'v'\in \Delta(wqw')$ is accepting in
case of $q\in \mathfrak{Q}_\exists$,
\item
$wqw'$ is accepting if all $vq'v'\in \Delta(wqw')$ are accepting in
case of $q\in \mathfrak{Q}_\forall$.
\end{itemize}

An ATM is said to \emph{accept a word} $\mathfrak{w}\in \Gamma^*$,
if $q_0w$ is accepting.

An ATM is \emph{exponentially space bounded} if for any
$\mathfrak{w}\in \Gamma^*$, every indirect successor $vqv'$ of
$q_0\mathfrak{w}$ satisfies that $|vv'|<2^{|\mathfrak{w}|}$.
\end{definition}

According to \cite{ATM}, there is an exponentially space bounded ATM
$\mathcal{M}^*$, whose word problem is \ExpExpTime-hard. In order to
simplify our argument, we will, however, not directly simulate a run
of this Turing machine on a word $\mathfrak{w}$. Rather, given
$\mathcal{M}^*$ and a word $\mathfrak{w}$ it is straightforward to
construct an ATM $\mathcal{M}^*_\mathfrak{w}$ such that
$\mathcal{M}^*$ accepts $\mathfrak{w}$ exactly if
$\mathcal{M}^*_\mathfrak{w}$ accepts the empty word
$\epsilon$.\footnote{Such a machine can be easily obtained: add new
states and transitions that first write $\mathfrak{w}$ to the tape,
second go back to the starting position, and third switch to the
initial state of the original Turing machine.} Clearly, the size of
$\mathcal{M}^*_\mathfrak{w}$ is polynomially bounded by
$n=|\mathfrak{w}|$.

In the following, we will thus show how, given an exponentially
space bounded ATM $\mathcal{M}$ and a word $\mathfrak{w}$, we can
construct a fact $F$, rule set $\mathcal{R}$ and query ${Q}$ -- the
size of all being polynomially bounded by $n$ -- such that
$F,\mathcal{R}\models q$ iff
$\mathcal{M}^*_\mathfrak{w}=(\mathfrak{Q},\Gamma,q_0,\Delta)$
accepts the empty word. Thereby, the minimal model of $F$ and
$\mathcal{R}$ will contain elements representing the initial and all
its (indirect) successor configurations. These configurations will
themselves be endowed with a tree structure that stores the content
of the exponentially bounded memory. The most intricate task to be
solved will be to model memory preservation from one configuration
to its successors.

We start by introducing some predicates and their intuitive meaning:

\begin{itemize}
\item
$\mathit{conf}$: unary predicate to distinguish elements
representing configurations from other elements;
\item
$\mathit{firstconf}$: unary predicate to denote the initial
configuration;
\item
$\mathit{trans}_\delta$ with $\delta\in \Delta$: set of binary
predicates. $\mathit{trans}_\delta$ connects a configuration with
its successor configuration that was introduced due to $\delta$;
\item
$\mathit{state}_q$ with $q\in \mathfrak{Q}$: set of unary predicates
indicating the state of the configuration $wqw'$;
\item
$\mathit{symbol}_\gamma$ with $\gamma \in \Gamma$: set of unary
predicates indicating the symbol at the heads current position, i.e.
$\mathit{symbol}_\gamma$ holds for the configuration $wq\gamma w'$;
\item
$\mathit{accepting}$: unary predicate indicating if a configuration
is accepting;
\item
$\mathit{wire}$ binary predicate used later for memory operations;
\item
$\mathit{fw}$ unary predicate used later for memory operations.
\end{itemize}

We are now ready to provide first constituents of the fact $F$ and
of the rule set $\mathcal{R}$. We let $F$ contain the facts:

\begin{eqnarray}
\{\mathit{conf}(init),\mathit{firstconf}(init),\mathit{state}_{q_0}(init)\}.
\end{eqnarray}

For every $\delta=(q,\gamma,q',\gamma',D)\in \Delta$ (with $D\in
\{L,R\}$), let $\mathcal{R}$ contain:

\begin{eqnarray}
\mathit{state}_q(x) {\wedge} \mathit{symbol}_\gamma(x) &\!\!\!\! \to
\!\!\!\!& \mathit{trans}_\delta(x,y) {\wedge} \mathit{wire}(x,u)
{\wedge} \mathit{fw}(u) {\wedge} \mathit{wire}(u,v) {\wedge}
\mathit{fw}(v)
{\wedge} \mathit{wire}(v,y)\\
\mathit{trans}_\delta(x,y) &\!\!\!\! \to \!\!\!\!& \mathit{conf}(y)
{\wedge} \mathit{state}_{q'}(y)
\end{eqnarray}

Clearly, by means of these rules, we create the successor
configurations $y$ reached from a transition predicate from given
configuration $x$. The additionally introduced sequence
``$\stackrel{wire}{\longrightarrow} u
\stackrel{wire}{\longrightarrow} v
\stackrel{wire}{\longrightarrow}$'' between $x$ and $y$ will come
handy later for memory preservation purposes.
Figure~\ref{fig-conftree} displays the structure of the
configuration tree thus constructed.

\begin{figure}
\begin{center}
\includegraphics[scale=0.9]{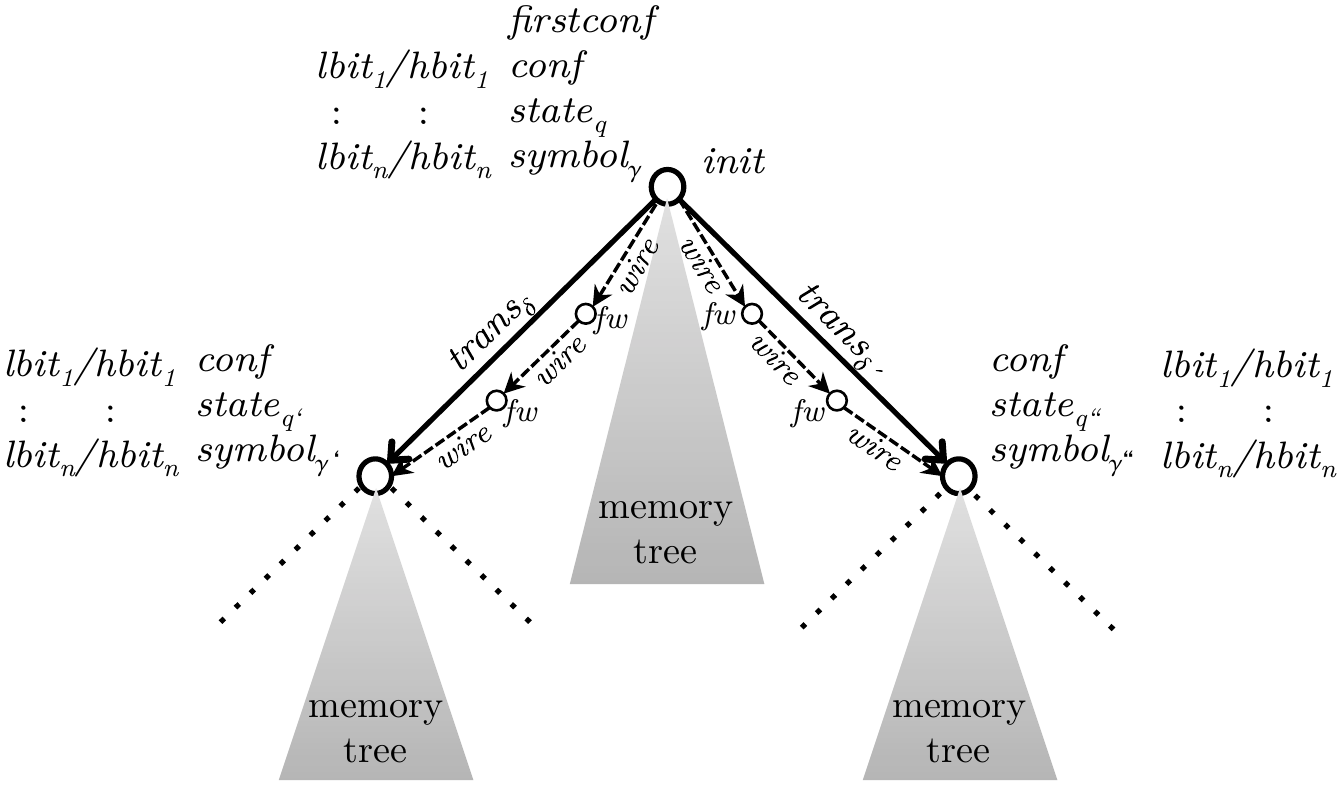}
\end{center}
\caption{Structure of the configuration tree in the constructed
model.\label{fig-conftree}}
\end{figure}

Next we take care of the implementation of the acceptance condition
for configurations. For every $q\in \mathfrak{Q}_\exists$ and every
$\delta=(q,\gamma,q',\gamma',D)\in \Delta$, we add the rule:

\begin{eqnarray}
\mathit{trans}_\delta(x,y) \wedge \mathit{accept}(y) \to
\mathit{accept}(x).
\end{eqnarray}

For every $q\in \mathfrak{Q}_\forall$ and $\gamma\in
\Gamma$, we add the rule:

\begin{eqnarray}
\bigwedge_{\delta=(q,\gamma,q',\gamma',D)\in \Delta}
\mathit{trans}_\delta(x,y_\delta) \wedge \mathit{accept}(y_\delta)
\to \mathit{accept}(x).
\end{eqnarray}

This way, as required, acceptance is propagated backward from
successors to predecessors.

Consequently, the query to be posed against the ``computation
structure'' described by our rule set should ask if the initial
configuration is accepting, i.e.:

\begin{eqnarray}
{Q} = \mathit{accept}(init).
\end{eqnarray}

Next, we prepare the implementation of the memory access. To this
end, we encode the position of the head (that is, the length of the
word $w$) in a configuration $wqw'$ as an $n$-digit binary number
(note that this allows us to address $2^n$ positions which is
sufficient for the required exponential memory). We will use unary
predicates $\mathit{hbit}_k$, $\mathit{lbit}_k$ with $1\leq k \leq
n$ for the following purpose: $\mathit{hbit}_k$ holds for an element
representing a configuration $wqw'$, if the $k$th bit of the
configuration's head position (i.e. the number $|w|$) expressed in
binary format is $1$. If the bit is $0$, then $\mathit{lbit}_k$
holds instead.

Clearly, the initial position of the head is $0$ (as we start from
configuration $q_0\mathfrak{w}$), thus for the initial configuration
(represented by $init$) all bits must be $0$. Hence we let
$F_{\mathcal{M},\mathfrak{w}}$ contain $\mathit{lbit}_k(init)$ for
all $1\leq k \leq n$.

In the course of a state transition
$\delta=(q,\gamma,q',\gamma',D)\in \Delta$, the head's position may
be increased by one (in case $D=R$) or decreased by 1 (in case
$D=L$). The next rules implement this behavior, hence given a
configuration's head position, they effectively compute the head
position of this configuration's direct successors. For every
$\delta=(q,\gamma,q',\gamma',D)\in \Delta$ with $D=R$ we let
$\mathcal{R}_{\mathcal{M},\mathfrak{w}}$ contain the rules (where
$k$ ranges from $1$ to $n$ and $m$ ranges from $1$ to $k$):

\begin{eqnarray}
\mathit{trans}_\delta(x,y)\wedge \bigwedge_{l\leq k}hbit_l(x)
& \to lbit_k(y)\\
\mathit{trans}_\delta(x,y)\wedge lbit_k(x) \wedge \bigwedge_{l <
k}hbit_l(x)
& \to hbit_k(y)\\
\mathit{trans}_\delta(x,y)\wedge lbit_k(x) \wedge lbit_m(x)
& \to lbit_k(y)\\
\mathit{trans}_\delta(x,y)\wedge hbit_k(x) \wedge lbit_m(x) & \to
hbit_k(y)
\end{eqnarray}

and for every $\delta=(q,\gamma,q',\gamma',D)\in \Delta$ with $D=L$
we let $\mathcal{R}_{\mathcal{M},\mathfrak{w}}$ contain the rules
(ranges of $k$ and $m$ as above):

\begin{eqnarray}
\mathit{trans}_\delta(x,y)\wedge \bigwedge_{l\leq k}lbit_l(x)
& \to hbit_k(y)\\
\mathit{trans}_\delta(x,y)\wedge hbit_k(x) \wedge \bigwedge_{l <
k}lbit_l(x)
& \to lbit_k(y)\\
\mathit{trans}_\delta(x,y)\wedge hbit_k(x) \wedge hbit_m(x)
& \to hbit_k(y)\\
\mathit{trans}_\delta(x,y)\wedge lbit_k(x) \wedge hbit_m(x) & \to
lbit_k(y)
\end{eqnarray}

In the next steps, we need to implement the exponential size memory
of our Turing machine. At the same time, the memory should be
``accessible'' by polynomial size rule bodies. Thus we organize the
memory as a binary tree of polynomial depth having exponentially
(that is $2^n$) many leaves. Thus, for every configuration element,
we create a tree of depth $n$ having the configuration element as
root and where the configuration's tape content is stored in the
leaves. We use the following vocabulary:

\begin{itemize}
\item
$\mathit{level}_k$ with $0\leq k \leq n$: set of unary predicates
stating for each node inside the memory tree its depth.
\item
$\mathit{leftchild}$, $\mathit{rightchild}$: the two (binary) child
predicates of the memory tree.
\item
$\mathit{child}$: a (binary) predicate subsuming the two above.
\item
$\mathit{entry}_\gamma$ with $\gamma \in \Gamma$: set of unary
predicates indicating for every leaf of the memory tree the symbol
stored there.
\end{itemize}

We now create the memory tree level by level (with $k$ ranging from
$1$ to $n$):

\begin{eqnarray}
\mathit{conf}(x)
& \to & \mathit{level}_0(x)\\
\mathit{level}_{k-1}(x) & \to & \mathit{leftchild}(x,y) \wedge \mathit{child}(x,y) \wedge \mathit{wired}(x,y) \wedge \mathit{wired}(y,x) \wedge \mathit{level}_{k}(y)\\
\mathit{level}_{k-1}(x) & \to & \mathit{rightchild}(x,y) \wedge \mathit{child}(x,y) \wedge \mathit{wired}(x,y) \wedge \mathit{wired}(y,x) \wedge \mathit{level}_{k}(y)\\
\end{eqnarray}

The leaf nodes of the memory tree thus created (i.e., the elements
satisfying $\mathit{level}_n$) will be made to carry two types of
information: (a) the current symbol stored in the corresponding tape
cell and (b) the tape cell's ``address'' in binary encoding. The
latter will be realized as follows: if the $k$th bit of the binary
representation of the address is clear, the leaf node $\nu$ will be
extended by a structure containing two newly introduced elements
$\nu_1$ and $\nu_2$ which will be connected with $\nu$ via the
following binary predicates: $\mathit{start}_k(\nu,\nu)$,
$\mathit{wired}(\nu,\nu_1)$, $\mathit{wired}(\nu_1,\nu_2)$, and
$\mathit{end}_k(\nu_2,\nu)$. In case the $k$th bit is set, we will
also introduce new elements $\nu_1$ and $\nu_2$ but they will be
connected with $\nu$ in a different way, namely:
$\mathit{start}_k(\nu,\nu_1)$, $\mathit{wired}(\nu_1,\nu)$,
$\mathit{wired}(\nu,\nu_2)$, and $\mathit{end}_k(\nu_2,\nu)$. The
reason for this peculiar way of encoding the address information
will become apparent in the sequel. Figure~\ref{fig-memorytree}
depicts the structure of the memory tree constructed under each
configuration tree.

\begin{figure}
\begin{center}
\includegraphics[scale=.8]{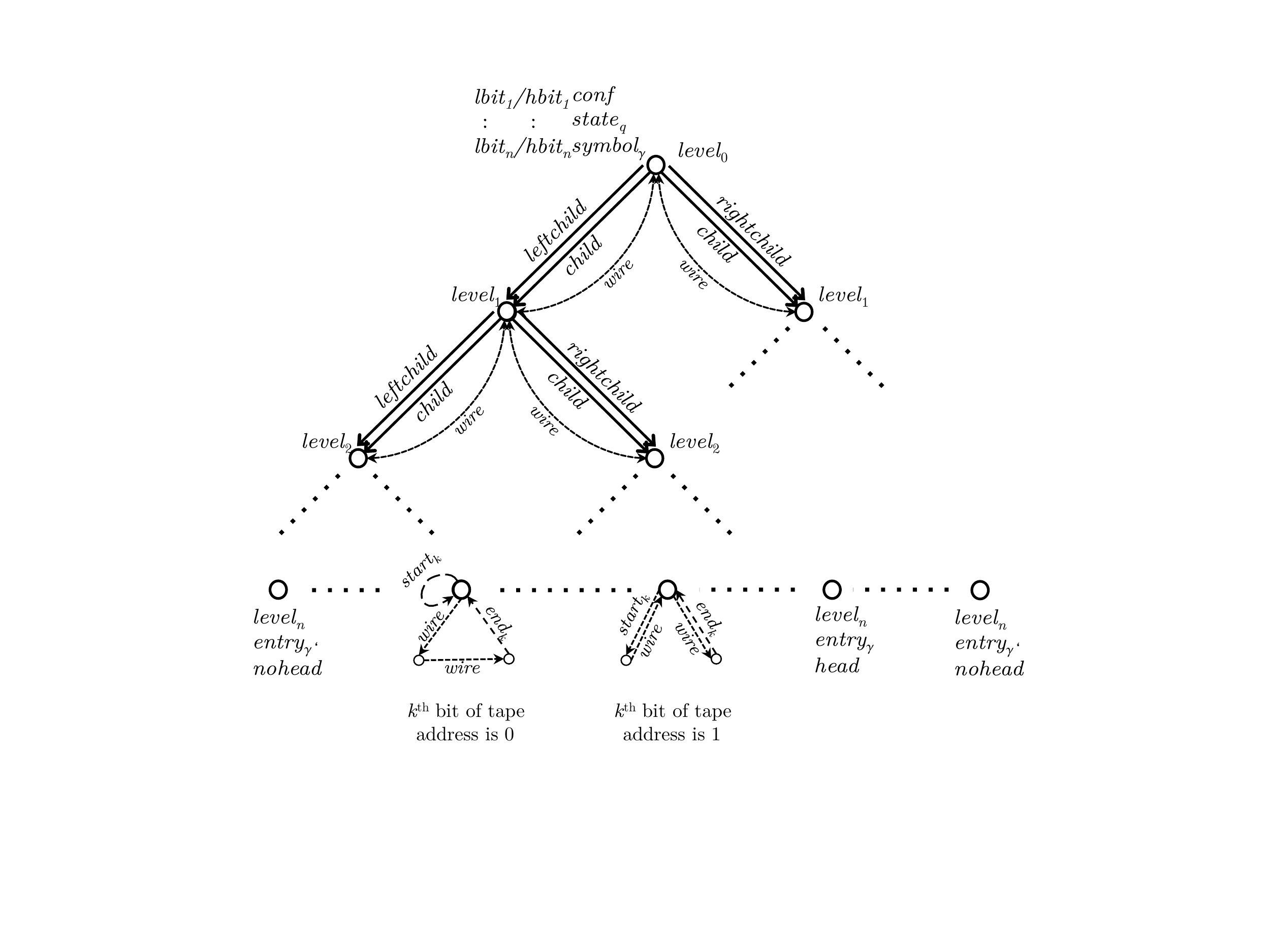}
\end{center}
\caption{Structure of the memory tree in the constructed
model.\label{fig-memorytree}}
\end{figure}

The following rules realize the aforementioned address
representation, exploiting the fact that the $k$th address bit will be $0$ if
the considered leaf node's ancestor on level $k-1$ is connected with
the ancestor on level $k$ via $\mathit{leftchild}$ and it will be
$0$ if the connection is via $\mathit{rightchild}$. Hence we let
$\mathcal{R}_{\mathcal{M},\mathfrak{w}}$ contain the rules (with $k$
ranging from $1$ to $n$ as above):

\begin{eqnarray}
\begin{array}{@{}rll@{}}
\mathit{leftchild}(x_{k-1},x_k) \wedge
\bigwedge_{i=k+1}^{n}(\mathit{child}(x_{i-1},x_i)) \wedge
\mathit{level}_n(x_n) & \to & \\
& & \hspace{-17em} \mathit{start}_k(x_n,x_n) \wedge
\mathit{wired}(x_n,x'_n) \wedge \mathit{wired}(x'_n,x''_n)\wedge \mathit{end}_k(x''_n,x_n)\\
\end{array}\\
\begin{array}{@{}rll@{}}
\mathit{rightchild}(x_{k-1},x_k) \wedge
\bigwedge_{i=k+1}^{n}(\mathit{child}(x_{i-1},x_i)) \wedge
\mathit{level}_n(x_n) & \to & \\
& & \hspace{-17em} \mathit{start}_k(x_n,x'_n) \wedge
\mathit{wired}(x'_n,x_n) \wedge \mathit{wired}(x_n,x''_n)\wedge \mathit{end}_k(x''_n,x_n)\\
\end{array}
\end{eqnarray}

One of the purposes of the previous construction is to mark in each
memory tree the leaf corresponding to the current head position by a
unary predicate $\mathit{head}$ and all other leaves by another unary
predicate $\mathit{nohead}$. To this end, we encode the head
position stored in the configuration elements via the predicates
$\mathit{lbit}_k$ and $\mathit{hbit}_k$ in a ``structural way'',
similar to our encoding in the leaves:

\begin{eqnarray}
\mathit{lbit}_k(x) & \to & \mathit{rootstart}_k(x,x)\\
\mathit{hbit}_k(x) & \to & \mathit{rootstart}_k(x,x') \wedge
\mathit{wired}(x',x)
\end{eqnarray}

For the assignment of $\mathit{head}$ and $\mathit{nohead}$ to leaf
nodes, we now exploit two facts. First, the $k$th bit of the head
address -- stored in a configuration element $\nu_c$ -- and the
$k$th bit of the address of a leaf node $\nu_l$ of the same
configuration element coincide, if there are nodes
$\nu_1,\ldots,\nu_{n-1}$ such that there is a path $$\nu_c
\stackrel{\mathit{rootstart}_k}{\longrightarrow}\nu_1\stackrel{\mathit{wired}}{\longrightarrow}\ldots
\stackrel{\mathit{wired}}{\longrightarrow}
\nu_{n-1}\stackrel{\mathit{end}_k}{\longrightarrow} \nu_{l};$$
moreover, no other two nodes are connected by such a path. Secondly,
the $k$th bit of the two nodes differ if there are nodes
$\nu_1,\ldots,\nu_{n}$ such that there is a path $$\nu_c
\stackrel{\mathit{rootstart}_k}{\longrightarrow}\nu_1\stackrel{\mathit{wired}}{\longrightarrow}\ldots
\stackrel{\mathit{wired}}{\longrightarrow}
\nu_{n}\stackrel{\mathit{end}_k}{\longrightarrow} \nu_{l};$$
moreover, no other two nodes are connected by such a path (to see
this, note that $\mathit{wired}$ goes both ways inside the tree
whence it is possible to make a back-and-forth step where
necessary).

This allows us to assign $\mathit{head}$ to all leaf nodes where a
path of the first kind exists \textbf{for every} $k$ as expressed by
the following rule:

\begin{eqnarray}
\bigwedge_{k=1}^{n}\left(
\mathit{rootstart}_k(x,x_{k,1})\wedge\left(\bigwedge_{i=1}^{n-2}\mathit{wired}(x_{k,i},x_{k,i+1})
\right)\wedge \mathit{end}_k(x_{k,n-1},y)\right)\to
\mathit{head}(y).
\end{eqnarray}

Likewise, we can assign $\mathit{nohead}$ to all leaf nodes where a
path of the second kind exists \textbf{for some} $k$, thus we add
for every $k$ ranging from $1$ to $n$ a separate rule of the
following kind:

\begin{eqnarray}
\mathit{rootstart}_k(x,x_{1})\wedge\left(\bigwedge_{i=1}^{n-1}\mathit{wired}(x_{i},x_{i+1})
\right)\wedge \mathit{end}_k(x_{n},y)\to \mathit{nohead}(y).
\end{eqnarray}

Now that we have an indicator of the head position in the memory
tree, we can enforce that every configuration element is indeed
assigned the $\mathit{symbol}_\gamma$ predicate whenever the symbol
$\gamma$ is found at the current head position:

\begin{eqnarray}
\left(\bigwedge_{i=0}^{n-1}\mathit{child}(x_{i},x_{i+1})
\right)\wedge \mathit{symbol}_\gamma(x_{n})\wedge head(x_{n})\to
\mathit{symbol}_\gamma(x_{0}).
\end{eqnarray}

The last bit of the alternating Turing machine functionality that
needs to be taken care of is memory evolution: a symbol stored in
memory changes according to the transition relation if and only if the head is
at the corresponding position. In our encoding
this means that for all $\mathit{nohead}$-assigned leaf nodes of a
configuration's memory tree, their stored symbol has to be
propagated to the corresponding leaf nodes of all direct successors'
memory trees. Again we exploit structural properties to connect the
corresponding leaf nodes of two subsequent configurations' memory
trees:

Let $\nu_c$ and $\nu'_c$ be two configuration elements such that
$\nu'_c$ represents a direct successor of $\nu_c$. Let $\nu_l$ be a
leaf node of $\nu_c$'s memory tree and let $\nu'_l$ be a leaf node
of $\nu'_c$'s memory tree. Let the $k$th bit of $\nu_l$'s address
and the $k$th bit of $\nu'_l$'s address coincide. Then -- and only
then -- there are nodes $\nu_1,\ldots,\nu_{2n+6}$ such that there is
a path $$\nu_l
\stackrel{\mathit{start}_k}{\longrightarrow}\nu_1\stackrel{\mathit{wired}}{\longrightarrow}\ldots
\stackrel{\mathit{wired}}{\longrightarrow}
\nu_{2n+6}\stackrel{\mathit{end}_k}{\longrightarrow} \nu'_{l}$$
where $\mathit{fw}$ holds for $\nu_{n+3}$.

This justifies to transfer the stored symbol from any non-head-leaf
to all leaf nodes to which it is simultaneously connected by such
paths \textbf{for every} $k$:

\begin{eqnarray}
\begin{array}{rll}
\displaystyle \bigwedge_{k=1}^{n}\!\!\left(\!\!
\mathit{start}_k(x,x_{k,1})\wedge\!\!\left(\bigwedge_{i=1}^{2n+5}\!\!\mathit{wired}(x_{k,i},x_{k,i+1})
\!\!\right)\wedge \mathit{end}_k(x_{k,2n+6},y)\!\!\right)\wedge & & \\
\displaystyle\left(\bigwedge_{k=1}^{n}
\mathit{fwd}(x_{k,n+3})\right) \wedge \mathit{symbol}_\gamma(x)
\wedge \mathit{nohead}(x) & \to & \mathit{symbol}_\gamma(y).
\end{array}
\end{eqnarray}

Of course, we also need to take care to assign the proper symbol
(which is determined by the transition by which the current
configuration has been reached) to the leaf node of the previous
configuration's head position. To this end, we add for every
$\delta=(q,\gamma,q',\gamma',D)\in \Delta$ the rule

\begin{eqnarray}
\begin{array}{r}
\displaystyle \mathit{head}(x)\wedge \!\!
\bigwedge_{k=1}^{n}\!\!\left(\!\!
\mathit{start}_k(x,x_{k,1})\wedge\!\!\left(\bigwedge_{i=1}^{2n+5}\mathit{wired}(x_{k,i},x_{k,i+1})
\!\!\right)\wedge \mathit{end}_k(x_{k,2n+6},z_n)\!\!\right) \wedge \hspace{1cm}\\
\displaystyle \left(\bigwedge_{k=1}^{n}
\mathit{fwd}(x_{k,n+3})\right) \wedge
\mathit{trans}_\delta(z,z_0)\wedge
\left(\bigwedge_{i=0}^{n-1}\mathit{child}(z_{i},z_{i+1}) \right)\to
\mathit{symbol}_\gamma(z_n).
\end{array}
\end{eqnarray}

Finally, we have to ensure that the initial configuration and its
memory tree carry all the necessary information. We have to
initialize the head position address to $0$ by adding to $F$ the facts

\begin{eqnarray}
\{lbit_1(init),\ldots,lbit_n(init)\}.
\end{eqnarray}

Moreover, all tape cells initially contain the blank symbol
$\Box$, which we achieve by extending $\mathcal{R}$ by the rule

\begin{eqnarray}
\mathit{firstconf}(x_0)\wedge
\left(\bigwedge_{k=0}^{n-1}\mathit{child}(x_{k},x_{k+1})\right) \to
\mathit{symbol}_{\Box}(x_n).
\end{eqnarray}

Concluding, we have just built $F$, $\mathcal{R}$ and ${Q}$ with the
desired properties. Moreover, $\mathcal{R}$ consists of only
\RfRrRoR rules and does not contain any constant. This concludes our
argument that the combined complexity of CQ entailment over \RfRrRoR
rules is \ExpExpTime{}-hard, even in the case where no constants
show up in the rules.

\begin{theorem}
Conjunctive query entailment for constant-free \RfRrRoR rules with
bounded predicate arity is \ExpExpTime{}-hard.
\end{theorem}

\subsection{Combined Complexity of Weakly Guarded Frontier-One Rules is \ExpExpTime-hard}

Our last hardness result will be established along the same lines as
the preceding one, namely by a reduction from the word problem of an
alternating Turing machine with exponential space. In fact, we will
also reuse part of the reduction and arguments presented in the
previous section. In particular, we assume everything up to formula
(14) as before except for the following modifications:
\begin{itemize}
\item Remove from Rule (2) all atoms built from the predicates
$\mathit{wire}$ and $\mathit{fw}$.
\item Replace the Rule (5) with the following rules:
\begin{itemize}
\item for every $\delta=(q,\gamma,q',\gamma',D)\in \Delta$ the rule
$$
\mathit{trans}_\delta(x,y) \wedge \mathit{accept}(y) \to
\mathit{accept}_\delta(x)
$$
\item the rule
$$
\bigwedge_{\delta=(q,\gamma,q',\gamma',D)\in \Delta}
\mathit{accept}_\delta(x) \to \mathit{accept}(x)
$$
\end{itemize}
Thereby, we introduce a fresh predicate $\mathit{accept}_\delta$ for
every $\delta=(q,\gamma,q',\gamma',D)\in \Delta$. Clearly this set
of rules has the same consequences as the previous Rule (5), however
it consists merely of \RgRfRrRoR rules.
\end{itemize}

This puts the ATM's ``state space'' and transition relations into
place. Now we turn to the task of encoding the exponential tape. As
opposed to the previous encoding, we now exploit that we can use
predicates of arbitrary arity. Thus, for every $\gamma \in \Gamma$
we introduce an $n{+}1$-ary predicate $\mathit{ontape}_\gamma$ where
the first $n$ positions are used for the binary encoding of a tape
address and the $n{+}1$st position contains the configuration
element that this tape information refers to. Our encoding will
ensure that the first $n$ positions of these predicates are
non-affected. Likewise we will use $n{+}1$-ary predicates
$\mathit{head}$ and $\mathit{nohead}$ to store for each tape
position of a configuration if the ATM's head is currently in that
position or not. For this purpose, we introduce auxiliary constants
to encode whether address bits are high, low, or unknown. Thus we add to $F$ the following atoms:

$$
high(1),\ \ bit(1),\ \ low(0), \ \ bit(0).
$$

Now, for every $k$ we introduce a binary predicate $bit_k$ (whose
second position is non-affected) with the intention to let
$bit_k(x,0)$ hold whenever $lbit_k(x)$ holds and to also have
$hbit_k(x)$ imply $bit_k(x,1)$, which is achieved by the following
two rules:

$$
lbit_k(x) \wedge low(z) \to bit_k(x,z),
$$

$$
hbit_k(x) \wedge high(z) \to bit_k(x,z).
$$

Moreover, we make sure that the binary encoding of the head
position's address that we find attached to the configuration
elements through the $bit_k$ predicates is transferred into the
$head$ predicate as stated above:

$$
\mathit{conf}(x) \wedge \bigwedge_{k=1}^{n} bit_k(x,z_k) \to
head(z_1,\ldots,z_n,x).
$$

Additionally, we make sure that for all \textbf{other} binary
addresses $z_1\ldots z_n$ the according $nohead$ atoms hold, by
adding for every $i$ in the range from $1$ to $n$ the rule

$$
\mathit{conf}(x) \wedge bit_i(x,w) \wedge other(w,z_i)
\bigwedge_{k=1}^{n} bit(z_k) \to nohead(z_1,\ldots,z_n,x).
$$

Furthermore, we have to ensure that the symbol $\gamma$ found at any
configuration's head position (expressed by the corresponding
$ontape_\gamma$ atom) is also directly attached to that
configuration by the according unary $symbol_\gamma$ atom:

$$
head(z_1,\ldots,z_n,x) \wedge ontape_\gamma(z_1,\ldots,z_n,x) \to
symbol_\gamma(x).
$$

Also, the symbol $\gamma'$ written to the tape at the previous
configuration's head position as a result of some transition
$\delta$ can be realized easily:

$$
head(z_1,\ldots,z_n,x)\wedge trans_\delta(x,y) \to
ontape_{\gamma'}(z_1,\ldots,z_n,x)
$$

On the other hand, all previous nohead-positions of the tape will
keep their symbol, as made sure by the rules (for all $\gamma\in
\Gamma$):

$$
nohead(z_1,\ldots,z_n,x)\wedge ontape_{\gamma}(z_1,\ldots,z_n,x)
\wedge trans_\delta(x,y) \to ontape_{\gamma}(z_1,\ldots,z_n,y).
$$

To ensure that the initial configuration and its tape carry all the
necessary information, we initialize the head position address to
$0$ by adding to $F$, as in the previous section, the facts:

\begin{eqnarray}
\{lbit_1(init),\ldots,lbit_n(init)\}.
\end{eqnarray}

To make sure that all tape cells initially contain the blank symbol $\Box$, we extend
$\mathcal{R}$ by the rule:

\begin{eqnarray}
\mathit{firstconf}(x)\wedge \bigwedge_{k=1}^{n} bit(z_k) \to
\mathit{ontape}_{\Box}(z_1,\ldots,z_n,x).
\end{eqnarray}

Concluding, we have just built $F$, $\mathcal{R}$ and ${Q}$ with the
desired properties. Moreover, $\mathcal{R}$ consists of only
\RwRgRfRrRoR rules. Thus we arrive at the desired theorem.

\begin{theorem}The combined complexity of CQ
entailment over constant-free \RwRgRfRrRoR rules of unbounded arity
is \ExpExpTime{}-hard.
\end{theorem}

%% file: sec-bodyacyclic.tex


In this section, we study the complexity of frontier-guarded rules with an acyclic body. The acyclicity notion considered here is a slight adaptation of \emph{hypergraph acyclicity} stemming from database theory. We will show that body-acyclic \RfRgR rules coincide with guarded rules: indeed, a body-acyclic \RfRgR rule can be linearly rewritten as a set of guarded rules, and a guarded rule is a special case of body-acyclic rule.

%

Let us consider the \emph{hypergraph} naturally associated with a set of atoms $S$: its set of nodes is in bijection with $\terms{S}$ and its multiset of hyperedges is in bijection with $S$, with each hyperedge being the subset of nodes assigned to the terms of the corresponding atom.

To simplify the next notions, we first proceed with some normalization of a set of \RfRgR rules, such that all rules have an empty frontier (so-called ``disconnected rules'' \citep{blm:10}) or a ``variable-connected'' body:
\begin{enumerate}
\item Let $R$ be a \RfRgR rule  with a non-empty frontier and let $\mathcal B$ be the hypergraph associated with $body(R)$. Split each node in $\mathcal B$ assigned to a constant into as many nodes as hyperedges it belongs to (thus each constant node obtained belongs to a single hyperedge); let $\mathcal B'$ be the hypergraph obtained; let $C_f$ be the connected component of $\mathcal B'$ that contains the frontier guard(s); if there are several frontier guards, they are all in $C_f$.
\item Let $R_0 = \mathcal B' \setminus C_f \rightarrow p_0$, where $p_0$ is a new nullary predicate.
\item Let $R_f = C_f \cup \{p_0\} \rightarrow head(R)$.
\end{enumerate}

Let $(F, \mathcal R, Q)$ be an instance of the entailment problem, where $\mathcal R$ is a set of \RfRgR rules. All non-disconnected rules from $\mathcal R$ are processed as described above, which yields an equivalent set of \RfRgR rules. Let us denote this set by $\mathcal R_{disc} \cup \mathcal R'$, where $R_{disc}$ is the set of disconnected rules, i.e., initial disconnected rules and obtained rules of form $R_0$. The rules in  $ \mathcal R_{disc}$ are integrated into $F$ as described in \citep{blm:10}, which can be performed with $|\mathcal R_{disc}|$ calls to an oracle solving the entailment problem for \RfRgR rules. Briefly, for each $R_d = (B_d, H_d) \in \mathcal R_{disc}$, it is checked whether $F, \mathcal R' \models B_d$: if yes, $H_d$ is added to $F$ and $R_d$ is removed from $\mathcal R_{disc}$; the process is repeated until stability of $\mathcal R_{disc}$. Let $F'$ the fact obtained: for any BCQ $Q$, $F, \mathcal R \models Q$ iff $F',  \mathcal R'\models Q$.
 From now on, we thus assume that all \RfRgR rules have a non-empty frontier and their body is ``variable-connected'', i.e., the associated hypergraph is connected and cannot be disconnected by the above step 1.


The acyclicity of a hypergraph $\mathcal H$ is usually defined with respect to its so-called dual graph, whose nodes are the hyperedges of $\mathcal H$ and edges encode non-empty intersections between  hyperedges of $\mathcal H$. We define below a notion close to the dual graph, that we call \emph{decomposition graph} of a set of atoms. In a decomposition graph, guarded atoms are grouped together with one of their guard into a single node, and constants are not taken into account in atom intersections (it follows that the associated acyclicity notion is slightly more general than hypergraph acyclicity).

\begin{definition}[Decomposition Graph]
Let $S$ be a set of atoms. A \emph{decomposition graph} of $S$ is an undirected labeled graph $D_S=(V,E,\fun{atoms},\fun{vars})$, where $V$ is the set of nodes, $E$ is the set of edges, $\fun{atoms}$ and $\fun{vars}$ are labeling mappings of nodes and of edges respectively, such that:
\begin{itemize}
\item Let $\{C_1, \ldots C_p\}$ be a partition of $S$ such that in each $C_i$ there is an atom that guards the other atoms of $C_i$, with $p$ being minimal for this property. Then $V = \{v_1, \ldots v_p\}$ and for $1 \leq i \leq p$,   $atoms(v_i) = C_i$.
\item For $1 \leq i,j \leq p$, $i \neq j$, there is an edge $v_iv_j$ if $C_i$ and $C_j$ share a variable; $vars(v_iv_j) = \vars{C_i} \cap \vars{C_j}$.
\end{itemize}
\end{definition}

Several decomposition graphs can be assigned to $S$, however they have all the same structure and the same labeling on edges.  The only difference between them comes from the choice of a guard when an atom is guarded by several guards with incomparable sets of variables. Now, considering the decomposition graph instead of the dual graph, the acyclicity of a set of atoms is then defined similarly to that of a hypergraph.

\begin{definition}[Acyclicity of an Atom Set, Body-Acyclic \RfRgR Rule]
Let $S$ be a set of atoms and $D_S$ be a decomposition graph of $S$. An edge $v_iv_j$ in $D_S$ is said to be \emph{removable} if there is another path $\lambda$ between $v_i$ and $v_j$ in $D_S$  such that for each each edge $v_kv_l$ in $\lambda$, $\vars{v_iv_j} \subseteq \vars{v_kv_l}$. An \emph{acyclic covering} of $S$ is a forest obtained from $D_S$ by removing removable edges only. $S$ is said to be \emph{acyclic} if has an acyclic covering. An \RfRgR rule $R$ is said to be \emph{body-acyclic} \emph{(ba)} if $\body{R}$ is acyclic.
\end{definition}

\begin{example} \label{ex-dec} Let $S = \{p_1(x), p_2(x,u), p_2(y,z), p_3(y, z, u), p_2(u,v), p_3(u,v,x)\}$. $D_S$ has set of nodes $\{v_1, v_2, v_3\}$, with $\atoms{v_1} =\{p_1(x), p_2(x,u)\}$, $\atoms{v_2} =\{p_2(y,z), p_3(y, z, u)\}$ and $\atoms{v_3} =\{p_2(u,v), p_3(u,v,x)\}$, and all edges between these nodes, with $\vars{v_1v_2} = \vars{v_2v_3} = \{u\}$ and $\vars{v_1v_3} = \{x,u\}$. An acyclic covering of $S$ is obtained by removing edge $v_1v_2$ or $v_2v_3$.
\end{example}


Let us point out that a set of atoms is acyclic according to the above definition if and only if the associated existentially closed conjunctive formula belongs to the guarded fragment of first-order logic (see \emph{acyclic guarded covering} in \citep{kerdiles:01} and \citep{chein-mugnier:09} for details about this equivalence). Note also that the decomposition graph associated with the body of a guarded rule is restricted to a single node. Thus, guarded rules are trivially ba-\RfRgR rules.

Given a set of atoms $S$, checking whether it is acyclic, and if so, outputing one of its acyclic coverings can be performed in linear time (from \cite{Maier} about the computation of a join tree in databases).

Let $R$ be a variable-connected ba-\RfRgR rule and let $T$ be an acyclic covering of \emph{body(R)}, which is thus a tree. Let $\{v_1, ..., v_p\}$ be the nodes in $T$ and let $v_r$ be a node such that $\atoms{v_r}$ contains a frontier guard. $T$ is considered as rooted in $v_r$, which yields a direction of its edges from children to parents: a directed edge $(v_i,v_j)$ is from a child to its parent.  $R$ is translated into a set of guarded rules $\{R_1, \ldots R_p\}$ as follows:
\begin{itemize}
\item  To each edge $(v_i,v_j)$ is assigned the atom $a_i=q_i(\vars{v_iv_j})$, where $q_i$ is a new predicate;
\item To each node $v_i$, $i \neq r$, is assigned the rule:\\
$R_i = \atoms{v_i} \cup \{ a_k | v_k$ child of $v_i \} \rightarrow a_i$
\item To the node  $v_r$ is assigned the rule: \\
$R_r = \atoms{v_r }\cup \{ a_k | v_k$ child of $v_r \}  \rightarrow head(R)$
\end{itemize}

Note that this translation is the identity on guarded rules.

\paragraph{Example \ref{ex-dec}}(Contd.) Let $R = p_1(x) \wedge p_2(x,u) \wedge p_2(y,z) \wedge p_3(y, z, u) \wedge p_2(u,v) \wedge p_3(u,v,x) \rightarrow \head{R}$, with $\fr{R} = \{u,v\}$. Consider the acyclic covering with set of nodes  $\{v_1, v_2, v_3\}$, with $\atoms{v_1} =\{p_1(x), p_2(x,u)\}$,  $\atoms{v_2} =\{p_2(y,z), p_3(y, z, u)\}$ and  $\atoms{v_3} =\{p_2(u,v), p_3(u,v,x)\}$, and set of edges $\{v_1v_3, v_2v_3\}$. On has $\vars{v_1v_3}= \{u\}$ and $\vars{v_1v_3} = \{x,u\}$. $v_3$ is the root. The obtained guarded rules are: \\
$R_1 = p_1(x) \wedge p_2(x,u) \rightarrow q_1(x,u)$\\
$R_2 = p_2(y,z) \wedge p_3(y, z, u)  \rightarrow q_2(u)$\\
$R_3 =  p_2(u,v) \wedge p_3(u,v,t)  \wedge q_1(x,u) \wedge q_2(u)  \rightarrow \head{R}$\\

 \begin{property} Guarded rules and ba-\RfRgR rules are semantically equivalent in the following sense: a guarded rule is a ba-\RfRgR rule and any set of ba-\RfRgR rules can be translated into a semantically equivalent set of guarded rules.
 \end{property}


The above translation is polynomial in the size of the rule and arity-preserving. Thus, complexity results on guarded rules apply to ba-\RfRgR rules. In particular ba-\RfRgR rules are \ExpTime-complete for bounded-arity combined complexity, while \RfRgR rules are \ExpExpTime-complete.

Actually the  \ExpTime{} lower bound already holds for combined complexity with arity bounded by 2. Indeed, standard reasoning in the much weaker description logic reduced normalized Horn-$\mathcal{ALC}$ (cf. Section~\ref{sec:HornALC}), which is a fragment of ba-\RgRfRrRoR rules with maximal arity of $2$, is already \ExpTime-Hard, as cited in Theorem~\ref{thm-HornALC}. It follows that ba-\RgRfRrRoR rules  are  \ExpTime-Hard for bounded-arity in combined complexity. One could have expected ba-frontier-1 rules to be simpler than ba-\RfRgR rules. In fact, they have the same data complexity and the same  bounded-arity combined complexity. The only remaining question, for which we have no answer yet, is whether they are simpler in the unbounded arity case.

%
%


Finally, let us point out that the acyclicity of rule bodies alone is not enough to guarantee a lower complexity, and even decidability: that the head of a rule shares variables with only one node of the decomposition graph (thus, that the frontier is guarded) is crucial. Without this assumption, the entailment problem remains undecidable. \footnote{See for instance \citep{baget-mugnier:02},  which provides a reduction from the word problem in a semi-Thue system, known to be undecidable, to the CQ entailment problem with existential rules (in a conceptual graph setting):  this reduction yields existential rules with predicate arity bounded by 2, a body restricted to a path and a frontier of size 2.}

%% file: sec-conclusion.tex

We have introduced the notion of greedy bounded-treewidth sets of
existential rules that subsumes guarded rules, as well as their known
generalizations, and gives rise to a generic and worse-case optimal
algorithm for deciding conjunctive query entailment. Moreover, we
have classified known \RgRbRtRsR subclasses with respect to their
combined and data complexities.

The existential rule framework is young and combines techniques from different research fields (such as databases, description logics, rule-based languages). 
A lot of work is still to be done to deepen its understanding, design efficient algorithms and develop its applications.  

It remains an open question whether
 \RgRbRtRsR is recognizable. We conjecture that the answer is yes. 
  However, even if this question is
interesting from a theoretical viewpoint, we recall that  \RgRbRtRsR is not more expressive than  \RwRfRgR,
and furthermore, any \RgRbRtRsR set of rules can be \emph{polynomially} translated into a  \RwRfRgR set of rules,
while preserving entailment (Section
\ref{sec:gbts}). Moreover, we built a reduction from the entailment problem, where
$\mathcal{R}$ is \RfRgR and constant-free, to the problem 
of checking if some rule set $\mathcal{R}'$ is \RgRbRtRsR (this reduction is not included in this paper, since it is only one step in the study of the recognizability issue).
Hence, we know that determining if some rule set is \RgRbRtRsR 
is significantly harder than for \RwRfRgR, 
 where this check can be done in polynomial time.

Future work will aim at adding rules expressing restricted forms of
equality and of useful properties such as transitivity into this
framework, while preserving decidability, and the desirable \PTime{}
data complexity of \RfRgR rules. 

We have shown in Section \ref{sec-subclasses} that the PatSat algorithm can be adapted to run with optimal worst-case complexity for fragments of the  \RgRbRtRsR family. An important research line is the optimization, implementation and practical evaluation of this algorithm and its variants. To the best of our knowledge, currently existing prototypes processing existential rules follow a backward chaining  approach, which involves rewriting the query into a union of CQs; hence, termination of the query rewriting process is ensured on  \RfRuRsR rules \cite{tods-14-gop,klmt:14}


